\documentclass[10pt,journal]{cls/IEEEtran}
\usepackage[table,usenames,dvipsnames]{xcolor}
\usepackage{enumitem}
\usepackage[noadjust]{cite}
\usepackage{amsmath,amssymb,amsfonts,amsthm,dsfont} 
\usepackage{algorithm,algorithmicx,listings}        
\usepackage[noend]{algpseudocode}			        
\ifCLASSINFOpdf
  \usepackage{graphicx} 
  \DeclareGraphicsExtensions{.pdf,.jpg,.png}
\else
  \usepackage[dvips]{graphicx}
  \DeclareGraphicsExtensions{.eps}
\fi
\usepackage{graphicx,tabularx}
\usepackage[export]{adjustbox}
\usepackage[font={small}]{caption}
\usepackage{subcaption}
\usepackage{tikz,tikz-3dplot,pgfplots}
\pgfplotsset{compat=1.14}
\usetikzlibrary{fadings}

\usepackage[titletoc,title]{appendix}

\usepackage[breaklinks=true, colorlinks, bookmarks=true, citecolor=Black, urlcolor=Violet,linkcolor=Black]{hyperref}

\newcommand{\calB}{{\cal B}}

\newcommand{\calD}{{\cal D}}
\newcommand{\calE}{{\cal E}}
\newcommand{\calF}{{\cal F}}
\newcommand{\calG}{{\cal G}}

\newcommand{\calN}{{\cal N}}
\newcommand{\calO}{{\cal O}}
\newcommand{\calP}{{\cal P}}
\newcommand{\calQ}{{\cal Q}}

\newcommand{\calS}{{\cal S}}
\newcommand{\calT}{{\cal T}}

\newcommand{\calV}{{\cal V}}
\newcommand{\calW}{{\cal W}}
\newcommand{\calX}{{\cal X}}

\newcommand{\setG}{\textsf{G}}

\newcommand{\bfe}{\mathbf{e}}
\newcommand{\bff}{\mathbf{f}}

\newcommand{\bfm}{\mathbf{m}}

\newcommand{\bfp}{\mathbf{p}}
\newcommand{\bfq}{\mathbf{q}}

\newcommand{\bfx}{\mathbf{x}}
\newcommand{\bfy}{\mathbf{y}}

\newcommand{\bfgamma}{\boldsymbol{\gamma}}

\newcommand{\bfzeta}{\boldsymbol{\zeta}}
\newcommand{\bfeta}{\boldsymbol{\eta}}

\newcommand{\bfmu}{\boldsymbol{\mu}}

\newcommand{\bfpi}{\boldsymbol{\pi}}

\newcommand{\bfomega}{\boldsymbol{\omega}}

\newcommand{\bfE}{\mathbf{E}}

\newcommand{\bfR}{\mathbf{R}}

\newcommand{\bbR}{\mathbb{R}}

\newcommand{\bbX}{\mathbb{X}}

\usepackage{cls/mysymbol}

\newcommand{\env}{\calW} 
\newcommand{\allPs}{\calP_{\#}} 
\newcommand{\cP}[2][]{\calP_{#2}^{#1}} 
\renewcommand{\P}[2][]{\tilde{\calP}_{#2}^{#1}} 

\newcommand{\D}[2][]{\tilde\calD_{#2}^{#1}} 

\newcommand{\X}[2][]{\tilde\calX_{#2}^{#1}} 

\newcommand{\Y}[2][]{\tilde\bby_{#2}^{#1}} 

\newcommand{\Z}[2][]{Z_{#2}^{#1}} 

\def\weight{W} 
\def\rb{i} 
\def\rbsec{j} 
\def\cls{l}
\def\msr{k}

\def\classN{{\ensuremath{\mathcal C}}} 

\newtheorem{proposition}{Proposition}

\newtheorem{lemma}{Lemma}
\theoremstyle{definition}
\newtheorem{definition}{Definition}

\theoremstyle{remark}
\newtheorem{remark}{Remark}
\usepackage{lipsum}

\newcommand{\scaleMathLine}[2][1]{\resizebox{#1\linewidth}{!}{$\displaystyle{#2}$}}
\newcommand{\prl}[1]{\left(#1\right)}
\newcommand{\brl}[1]{\left[#1\right]}
\newcommand{\crl}[1]{\left\{#1\right\}}

\DeclareMathOperator*{\diag}{diag}
\def \sign    {\text{\normalfont sign} }

\title{\LARGE \bf
Dense Incremental Metric-Semantic Mapping for Multi-Agent Systems via Sparse Gaussian Process Regression}

\author{Ehsan Zobeidi$^1$ \and Alec Koppel$^2$ \and Nikolay Atanasov$^1$
\thanks{We gratefully acknowledge support from ARL DCIST CRA W911NF-17-2-0181, NSF NRI CNS-1830399, and ONR SAI N00014-18-1-2828.}%
\thanks{$^{1}$Department of Electrical and Computer Engineering, University of California, San Diego, La Jolla, CA 92093, USA
        {\tt\small \{ezobeidi,natanasov\}@ucsd.edu}}%
\thanks{$^{2}$Computational and Information Sciences Directorate, U.S. Army Research Laboratory, Adelphi, MD 20783, USA
        {\tt\small alec.e.koppel.civ@mail.mil}}%
}

\begin{document}

\maketitle
\thispagestyle{empty}
\pagestyle{empty}




\begin{abstract}
We develop an online probabilistic metric-semantic mapping approach for mobile robot teams relying on streaming RGB-D observations. The generated maps contain full continuous distributional information about the geometric surfaces and semantic labels (e.g., chair, table, wall). Our approach is based on online Gaussian Process (GP) training and inference, and avoids the complexity of GP classification by regressing a truncated signed distance function (TSDF) of the regions occupied by different semantic classes. Online regression is enabled through a sparse pseudo-point approximation of the GP posterior. To scale to large environments, we further consider spatial domain partitioning via an octree data structure with overlapping leaves. An extension to the multi-robot setting is developed by having each robot execute its own online measurement update and then combine its posterior parameters via local weighted geometric averaging with those of its neighbors. This yields a distributed information processing architecture in which the GP map estimates of all robots converge to a common map of the environment while relying only on local one-hop communication. Our experiments demonstrate the effectiveness of the probabilistic metric-semantic mapping technique in 2-D and 3-D environments in both single and multi-robot settings.
\end{abstract}

\section{Introduction}
\label{sec:introduction}
Autonomous systems navigating and executing complex tasks in real-world environments require an understanding of the 3-D geometry and semantic context of the environment. This paper develops a probabilistic metric-semantic mapping algorithm, using streaming distance and semantic category observations onboard a robot, to reconstruct geometric surfaces and their semantic identity (e.g., chairs, tables, doors) via sparse online GP regression. In addition to a multi-modal environment abstraction, probabilistic metric-semantic mapping provides uncertainty estimates that can aid safe navigation and active mapping algorithms. To support collaboration among multiple robots operating in the same environment, we also consider a distributed setting in which each robot observes the environment locally, with its onboard sensors, and communicates its local map with one-hop neighbor robots to arrive at a common map of the environment observed across the whole robot network.

%


\begin{figure}[t]
\centering
\includegraphics[width=0.33\linewidth]{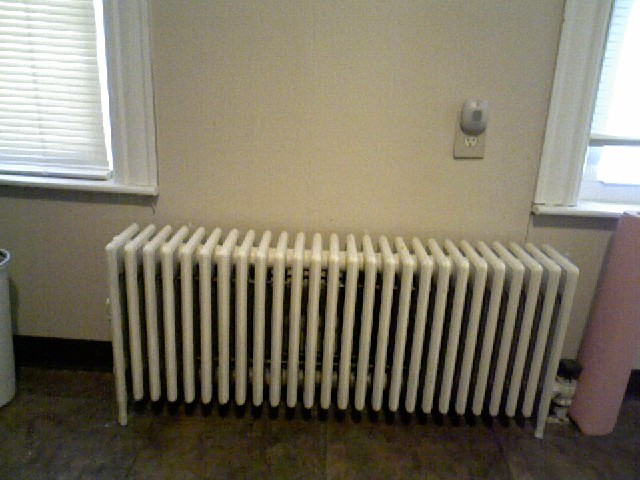}%
\includegraphics[width=0.33\linewidth]{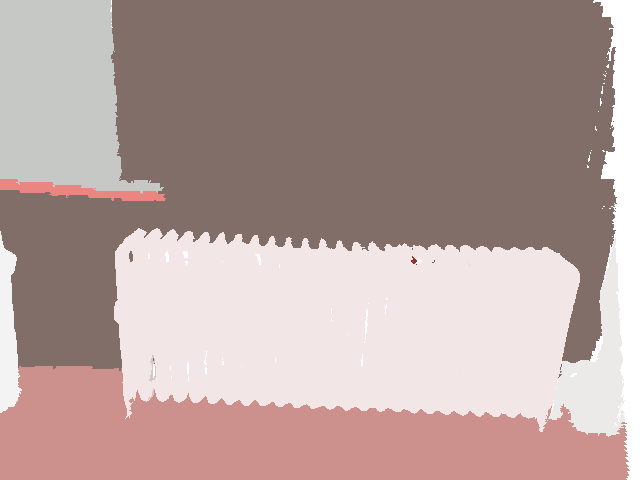}%
\includegraphics[width=0.33\linewidth]{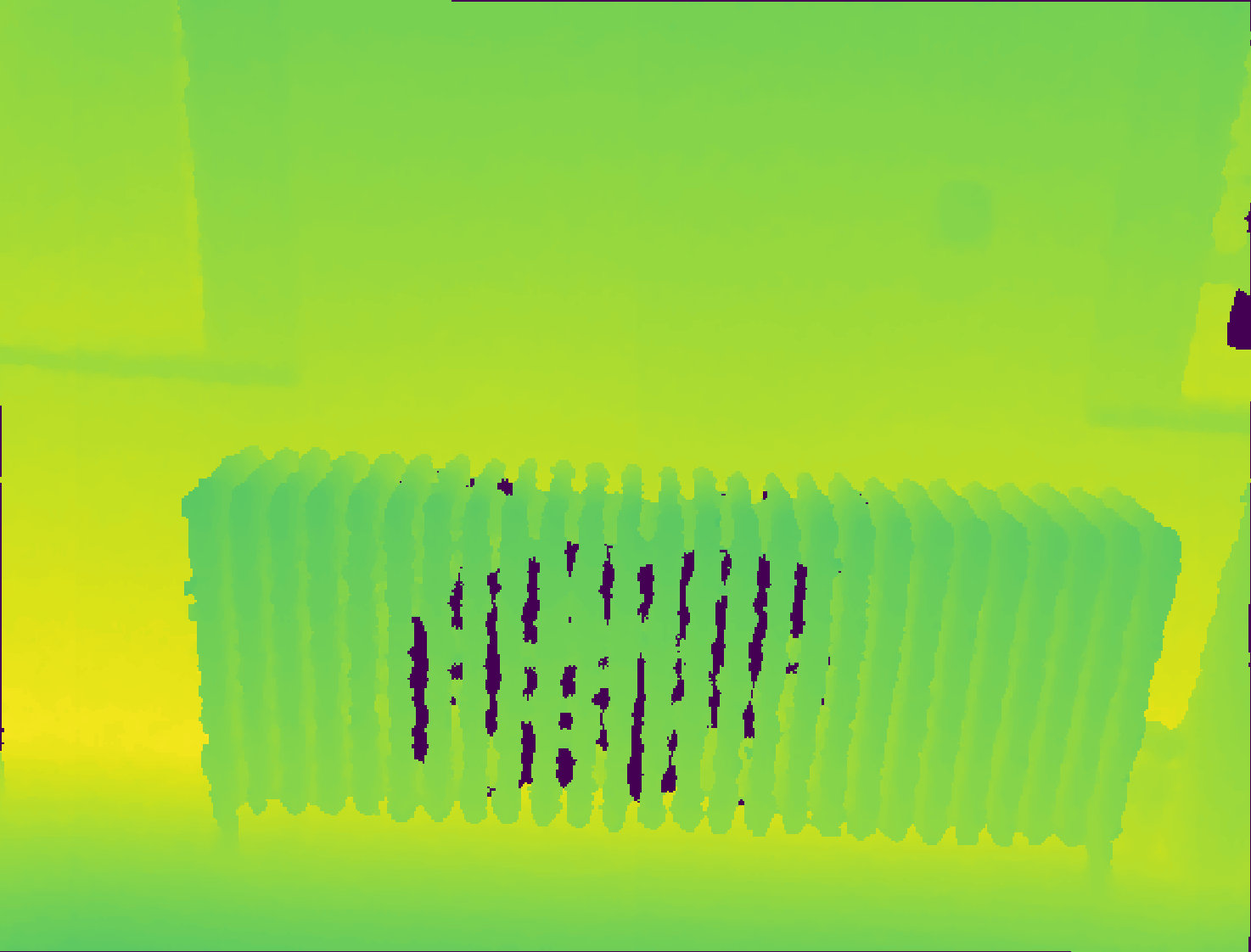}\\
\includegraphics[width=0.33\linewidth]{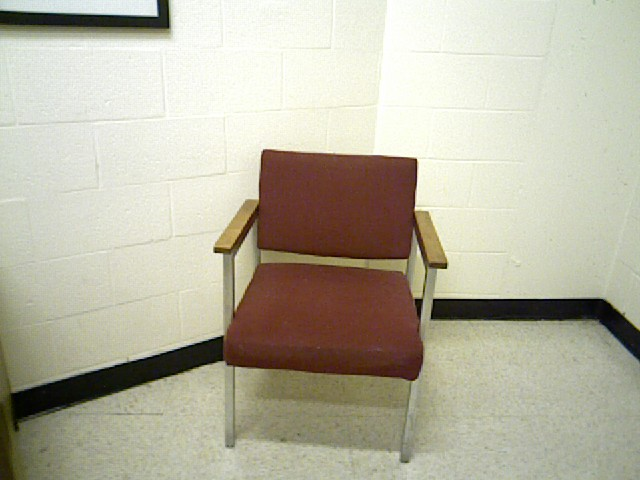}%
\includegraphics[width=0.33\linewidth]{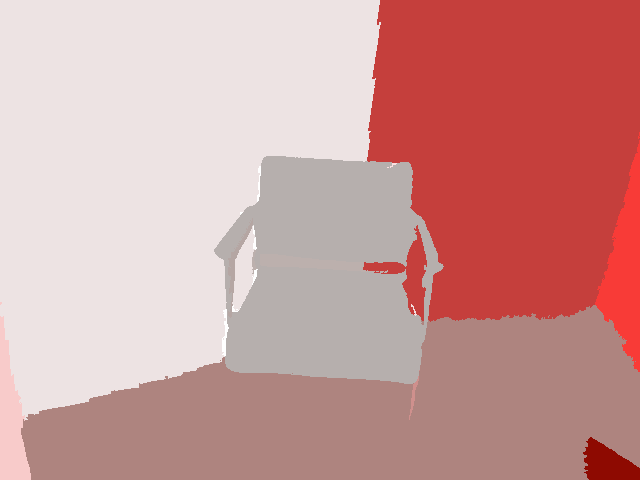}%
\includegraphics[width=0.33\linewidth]{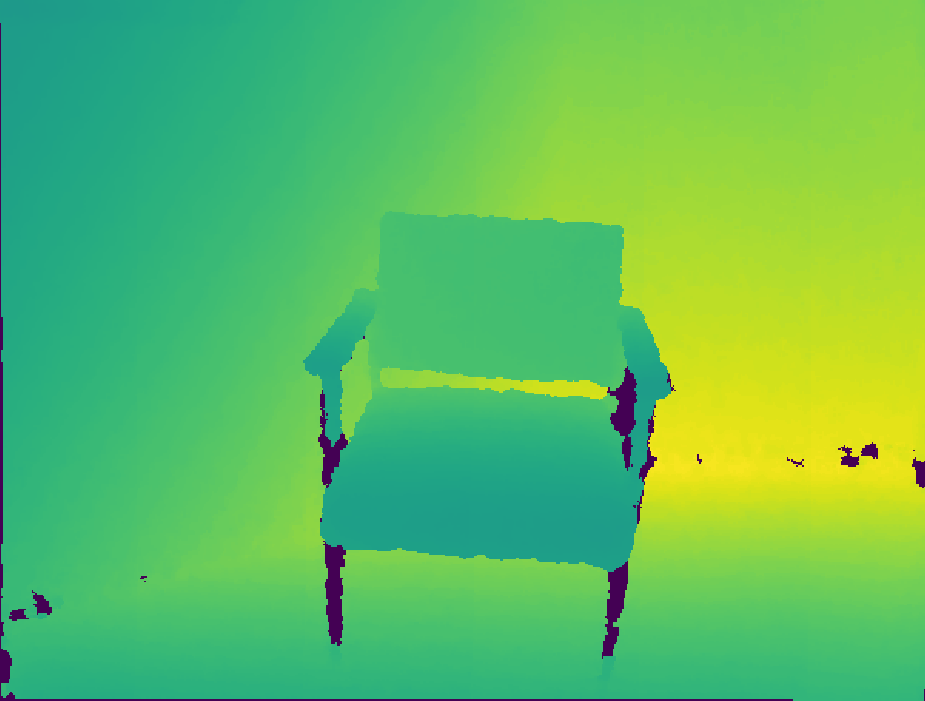}
\caption{RGB images (first column), segmented images (second column), and depth images (third column) used by the proposed approach for online construction of dense metric-semantic maps.}
\label{RGBDex}
\end{figure}

We focus on a TSDF representation \cite{curless1996volumetric,kazhdan2006poisson} which defines geometric surfaces implicitly, as the zero level-set of a TSDF function. TSDF surface representations have gained popularity due to their high accuracy (compared to regular, adaptive, or sparse grid representations \cite{hornung2013octomap, niessner2013real}) and ability to directly provide distance and gradient information (compared to explicit mesh representations \cite{kimera}) useful to specification of safety and visibility constraints. Classification of the geometric surfaces into semantic categories is crucial for context understanding and specification of complex robot tasks \cite{mccormac2017semanticfusion, hermans2014dense,kundu2014joint}. Many classification techniques, however, provide maximum likelihood, instead of Bayesian, estimates because efficient probabilistic classification remains an open problem in machine learning \cite{GP-classification-1,GP-classification-2}. The challenge is that discrete data likelihoods are not conjugate with a continuous map prior. While one may employ Laplace approximations to partially mitigate this challenge \cite{jadidi2017gaussian}), we propose a multi-class TSDF inference approach based on Bayesian \emph{regression}.




We employ GP regression \cite{rasmussen2003gaussian} to incorporate spatial correlation into a probabilistic resolution-free TSDF map of the 3-D environment. GP inference has been successfully used to obtain continuous map representations \cite{o2012gaussian, kim2013occupancy, jadidi2014exploration} but existing formulations are binary (instead of multi-class) and model occupancy (instead of a distance field). Range sensors, such as Lidars and depth cameras, do not provide direct TSDF observations because they measure distance in a specific viewing direction rather than to the nearest obstacle surface. To obtain TSDF training examples, we triangulate each depth image into a local mesh surface and measure the distance to it from a set of 3-D locations. 

Onboard sensors provide repeated observations of the same scene. While this redundancy is important for mitigating measurement noise, the amount of training data keeps growing over time. Hence, an important consideration for metric-semantic mapping is to build maps whose memory and computation requirements are designated by the underlying structure of the environment, rather than the number of distance and category observations. Unfortunately, GP training scales cubically with the number training examples but there are various ways to address this bottleneck \cite{snelson2006sparse, vff, koppel2019consistent, koppel2019optimally}. We observe that, in our setting, the data can be compressed significantly through averaging before GP training and, notably, this does not affect the posterior TSDF distribution. The remaining training pairs are used as \emph{pseudo points} \cite{snelson2006sparse} to support the continuous GP representation with a finite set of parameters. To reduce the complexity in large maps further, one might consider local kriging, decomposing the spatial domain into subdomains and making predictions at a test location using only the pseudo points contained within the subdomain. Choosing independent subdomains, however, leads to discontinuities of the predicted TSDF function at the subdomain boundaries. Ensemble methods that construct multiple local estimators and use a weighted combination of their predictions include Bayesian committee machines \cite{tresp2000bayesian,kim2014recursive}, sparse probabilistic regression \cite{bauer2016understanding}, or infinite mixtures Gaussian process experts \cite{rasmussen2002infinite}. These techniques avoid the discontinuities of local kriging but their computation cost is still significant for online training. Inspired by the adaptive occupancy representation of Octomap \cite{hornung2013octomap}, we propose an efficient approach that decomposes the environment into an Octree of overlapping subdomains, while preventing discontinuities in the GP posterior. Combining these ideas yields a hierarchical pseudo-point parameterization of the GP, which may be updated online to achieve incremental probabilistic mapping. Our method generates dense metric-semantic surfaces and, yet, remains efficient even in large environments.

Finally, we provide a distributed formulation of our TSDF GP regression, enabling multiple robots to collaboratively build a common metric-semantic map of the environment. Each robot updates a local GP Octree pseudo-point approximation but synchronizes its pseudo-point statistics by averaging with its one-hop communication neighbors. Our distributed inference approach is inspired by probabilistic consensus techniques \cite{nedic2017distributed, jadbabaie2012non}, but we generalize those from using a fixed parameter dimension to a changing number of pseudo-point parameters, resulting from robots observing new environment regions online. We prove that the local GP estimates of each individual robot converge in \emph{finite time} to the same GP posterior that would have been obtained by a central server using all observations obtained from all robots.

A preliminary version of this work was presented in \cite{Zobeidi_GPMapping_IROS20}. This version improves the theoretical development for the centralized single-robot setting and extends the approach to a decentralized multi-robot setting by introducing a novel approach for distributed incremental sparse GP regression with theoretical guarantees for consistent estimation. Additionally, this paper demonstrates the effectiveness of our decentralized approach via evaluations in 2-D simulation and 3-D real data sets. The main \textbf{contributions} of this work are to:
\begin{itemize}

\item develop an online GP training and inference algorithm for TSDF regression that enables 3-D semantic segmentation of the environment from streaming sensor data,

\item ensure controllable computation and memory complexity while providing a continuous-space probabilistic representations of the environment,

\item provide a distributed formulation of the TSDF GP regression, which enables a robot team to collaboratively build a common metric-semantic map from local observations and one-hop communication with provably equivalent quality to batch centralized estimation.
%
%
\end{itemize}
Our metric-semantic mapping approach is demonstrated in simulated and real-world datasets and may be used either offline, with all sensory data provided in advance, or online, processing distance and semantic category observations incrementally as they arrive.

%

\section{Related Work}
\label{sec:relatedworks}

Various representations have been proposed for occupancy or geometric surface estimation from range or depth measurements. Occupancy grid mapping \cite{elfes1989using} discretizes the environment into a regular voxel grid and estimates the occupancy probability of each voxel independently. A dense voxel representation quickly becomes infeasible for large domains and adaptive resolution data structures, such as an octree, are necessary \cite{hornung2013octomap,supereight}. While accurate maps may also be constructed using point cloud \cite{engel2014lsd,segmap} or surfel \cite{henry2012rgb,surfel-mapping} representations, such sparse maps do not easily support collision and visibility checking for motion and manipulation planning. Recent work is considering explicit polygonal mesh \cite{teixeira2016real,piazza2018real,kimera} and implicit signed distance function \cite{newcombe2011kinectfusion, whelan2016elasticfusion,oleynikova2017voxblox,han2019fiesta} models. We focus our review on TSDF techniques as they are most closely related to our work.

The seminal work of Curless et al. \cite{curless1996volumetric} emphasized the representation power of TSDF and showed that dense surface modeling can be done incrementally using range images. KinectFusion \cite{newcombe2011kinectfusion} achieved online TSDF mapping and RGB-D camera pose estimation by storing weighted TSDF values in a voxel grid and performing multi-scale iterative closest point (ICP) alignment between the predicted surface and the depth images. Niessner et al.~\cite{niessner2013real} demonstrated that TSDF mapping can be achieved without regular or hierarchical grid data structures by hashing TSDF values only at voxels near the surfaces. These three works inspired a lot of subsequent research, allowing mapping of large environments \cite{kintinuous}, real-time operation without GPU acceleration \cite{klingensmith2015chisel,flashfusion}, map correction upon loop closure \cite{InfiniTAM,voxgraph}, and semantic category inference \cite{grinvald2019volumetric}. Bylow et al.~\cite{bylow2013real} propose a direct minimization of TSDF projective depth error instead of relaying on explicit data association or downsampling as in ICP. TSDF maps are accurate and collision checking in them is essentially a look-up operation, prompting their use as an alternative to occupancy grids for robot motion planning and collision checking \cite{Oleynikova2016ContinuoustimeTO,han2019fiesta}. Voxblox \cite{oleynikova2017voxblox} incrementally builds a (non-truncated) Euclidean signed distance field (ESDF), applying a wavefront algorithm to the hashed TSDF values. Fiesta \cite{han2019fiesta} improves the ESDF construction by introducing two independent queues for inserting and deleting obstacles. Saulnier et al.~\cite{Saulnier_ActiveMapping_ICRA20} show that weights of the TSDF values arise as the variance of a Kalman filter and may be used as an uncertainty measure for autonomous exploration and active TSDF mapping.

Most TSDF mapping techniques, however, forgo probabilistic representations in the interest of scalability. Gaussian process (GP) inference has been used to capture correlation in binary occupancy mapping. O'Callaghan et al. \cite{o2012gaussian} is among the first works to apply GP regression to infer a latent occupancy function using data from a range sensor. The GP posterior is squashed to a binary observation model a posteriori to recover occupancy likelihood. The resulting probabilistic least-squares method is more efficient than GP classification but still scales cubically with the amount of training data. To address this, several works \cite{kim2014recursive,kim2013occupancy,Kim2015,Wang_GPRegressionMapping} rely on sparse kernels to perform separate GP regressions with small subsets of the training data and Bayesian Committee Machines (BCM) to fuse the separate estimates into a full probabilistic occupancy map. Ramos et al. \cite{ramos2016hilbert} proposed fast kernel approximations to project the occupancy data into a Hilbert space where a logistic regression classifier can distinguish occupied and free space. This idea has been extended to dynamic maps \cite{bayesian_hilbert} as well as into a variational autoencoder formulation \cite{Guizilini-RSS-17} that compresses the local spatial information into a latent low-dimensional feature representation and then decodes it to infer the occupancy of a scene. Guo and Atanasov \cite{guo2019information} showed that using a regular grid discretization of the latent function and a decomposable radial kernel leads to special structure of the kernel matrix (kronecker product of Toeplitz matrices) that allows linear time and memory representation of the occupancy distribution.


Augmenting occupancy representations with object and surface category information is an important extension, allowing improved situational awareness and complex mission specification for robots. Several works \cite{hermans2014dense, vineet2015incremental, sengupta2015semantic, yang2017semantic,zhao2016building} employ conditional random fields (CRFs) to capture semantic information. Vineet et al. \cite{vineet2015incremental} provide incremental reconstruction and semantic segmentation of outdoor environments using a hash-based voxel map and a mean-field inference algorithm for densely-connected CRFs. These techniques are accurate but also computationally expensive because they operate over each map element. Zheng et al. \cite{zheng2018pixels} incorporate spatial information across multiple levels of abstraction and form a probability distribution over semantic attributes and geometric representations of places using TopoNet, a deep sum-product neural network. Grinvald et al. \cite{grinvald2019volumetric} reconstruct individual object shapes from multi-view segmented images and assemble the estimates in a voxelized TSDF map. Gan et al. \cite{gan2020bayesian} propose a continuous-space multi-class mapping approach, which relies on a Dirichlet class prior, a Categorical observation likelihood, and Bayesian kernel inference to extrapolate the class likelihoods to continuous space. Rosinol et al. \cite{kimera}, provides a modern perception library by combining the state of the art in geometric and semantic understanding.

In many applications, metric-semantic mapping may be performed by a team of collaborating robots. Relying on centralized estimation has numerous limitations related to the communication, computation, and storage requirements of collecting all robot measurements and map estimates at a central server. It is important to develop distributed techniques that allow local inference and storage at each robot, communication over few-hop neighborhoods, and consensus among the robot estimates. Techniques extending network consensus \cite{consensus} to distributed probabilistic estimation \cite{rad2010distributed,nbsl,atanasov2014joint,nedic2016distributed,nedic2017distributed} are closely related. These works show that distributed estimation of a finite-dimensional parameter is consistent when the probability density functions maintained by different nodes are averaged over one-hop neighborhoods in a strongly connected, potentially time-varying graph. Our work extends these techniques to distributed probabilistic estimation functions relying on local averaging of sparse (pseudo-point) GP distributions. Specific to cooperative semantic mapping, Choudhary et al. \cite{choudhary2017distributed} develop distributed pose-graph optimization algorithms based on successive and Jacobi over-relaxation to split the computation among the robots. Koch et al. \cite{koch2016multi} develop a parallel multi-threaded implementation for cooperative 2-D SDF mapping. Lajoie et al. \cite{door-slam} propose a distributed SLAM approach with peer-to-peer communication that rejects spurious inter-robot loop closures using pairwise consistent measurement sets.

\section{Problem Formulation}
\label{sec:problem_formulation}

%


Consider a team of $n$ robots, communicating over a network represented as an undirected graph $G = (\calV, \calE)$ with vertices $\calV := \{1,...,n\}$ and edges $\calE \subset \calV \times \calV$. An edge $(\rb,\rbsec) \in \calE$ from robot $\rb$ to robot $\rbsec$ exists if the two robots can communicate. The robots directly connected to robot $\rb$ are called \emph{neighbors} and will be denoted by $\mathcal{N}_\rb := \crl{j \in \calV \mid (i,j) \in \calE}$.

The robots operate in an unknown workspace, represented as a subset of Euclidean space, $\env\subset \reals^3$. The workspace consists of two disjoint subsets $\calO$ and $\calF$, comprising obstacles and free space, respectively, i.e., $\env=\calO\cup\calF$. The obstacle region is a closed set that is a pairwise disjoint union, $\calO=\cup_{\cls=1}^\classN\calO_\cls$, of $\classN$ closed sets, each denoting the region occupied by object instances from the same semantic class. For example, $\calO_1$ may be the space occupied by all chairs, while $\calO_2$ may be the space occupied by all tables.


Each robot is equipped with a sensor, such as a lidar scanner or an RGB-D camera, that provides distance and class observations of the objects in its vicinity. We assume that the position $\bfp^\rb_t \in \reals^3$ and orientation $\bfR^\rb_t \in SO(3)$ of each sensor $i \in \calV$ at time step $t$ are known, e.g., from a localization algorithm running onboard the robots. We model a sensor observation as a set of rays (unit vectors), e.g., corresponding to lidar scan rays or RGB-D image pixels.

\begin{definition}
A \emph{sensor frame} $\bfE^\rb = \{\bfeta_\msr^\rb\}_{\msr}$ is a set of vectors $\bfeta_\msr^\rb \in \mathbb{R}^3$ such that $\|\bfeta_\msr^\rb\| = 1$, $\forall i, k$.
\end{definition}


At time $t$, the $\msr$-th sensor ray of robot $\rb$, starts at position $\bfp^\rb_t$ and has direction $\bbR^\rb_t\bfeta^\rb_\msr$. Each ray measures the distance to and semantic class of the object that it intersects with first. In practice, the class measurements are obtained from a semantic segmentation algorithm (e.g., \cite{milioto2019icra}), applied to the RGB image or lidar scan (see Fig.~\ref{RGBDex}), while the distance measurements are provided either as a transformation of the depth image or directly from the lidar scan.

\begin{definition}
	A \emph{sensor observation} of robot $\rb$ at time $t$ is a collection of distance $\lambda^\rb_{t,\msr} \in \reals_{\geq 0}$ and object class $c^\rb_{t,\msr} \in \{1,...,\classN\}$ measurements acquired along the rays $\bfeta_\msr^\rb \in \bfE^\rb$. 
\end{definition}

We define the relationship among the object sets $\calO_\cls$ and the sensor observations $\lambda^\rb_{t,\msr}$, $c^\rb_{t,\msr}$ next.

\begin{definition}\label{direcSDF}
	The \emph{truncated signed directional distance function} (TSDDF) $h_\cls(\bfx,\bfeta)$ of object class $\calO_\cls$, is the signed distance from $\bfx\in\env$ to the boundary $\partial\calO_\cls$ in direction $\bfeta\in\mathbb{R}^3$, truncated to a maximum of $\bar{d} \geq 0$, i.e.,
	\begin{align} \label{eq:H_signed_distance}
		h_\cls(\bbx,\bfeta) &:= 
		\begin{cases}
			-\min \left( d_{\bfeta}(\bbx,\partial\calO_\cls), \bar{d}\right)& \text{if } \bbx\in \calO_\cls\\
			\;\;\;\min \left(d_{\bfeta}(\bbx,\partial\calO_\cls), \bar{d}\right)& \text{if } \bbx\in \env\setminus\calO_\cls,
		\end{cases}\notag\\
		d_{\bfeta}(\bbx,\partial\calO_\cls) &:= \min \left\{ d \geq 0 \;\big\vert\; \bfx + d \bfeta \in  \partial\calO_\cls \right\}. 
	\end{align}
	%
\end{definition}


According to Def.~\ref{direcSDF}, $h_\cls(\bfp^\rb_t, \bbR^\rb_t\bfeta_\msr^\rb)$ is the (truncated) distance from sensor position $\bfp^\rb_t$ to object class $\calO_\cls$ along the direction $\bbR^\rb_t\bfeta_\msr^\rb$ of the $\msr$-th ray at time $t$. The class observation $c^\rb_{t,\msr}$ is determined by the object set $\calO_\cls$ with minimum absolute TSDDF to $\bfp^\rb_t$ along $\bbR^\rb_t\bfeta_\msr^\rb$:
\begin{equation}\label{eq:class_measurement}
	c^\rb_{t,\msr} = \argmin_{\cls \in \{1,\ldots,\classN\}} |h_\cls(\bfp^\rb_t,\bbR^\rb_t\bfeta_\msr^\rb)|.
\end{equation}
The distance observation $\lambda^\rb_{t,\msr}$ is a noisy measurement of the distance to the nearest object class: 
\begin{equation}\label{eq:data}
	\lambda^\rb_{t,\msr} = h_{c^\rb_{t,\msr}}(\bfp^\rb_t, \bbR^\rb_t \bfeta_\msr^\rb) + \epsilon, \qquad \epsilon \sim \ccalN(0,\sigma^2),
\end{equation}
where $\sigma^2$ is the variance of the distance measurement noise. These definitions are illustrated in Fig.~\ref{fig:observations}.


Given sensor poses $\bfp^\rb_t$, $\bfR^\rb_t$ and streaming onboard observations $\lambda^\rb_{t,\msr}$, $c^\rb_{t,\msr}$ for $t = 1, 2, \ldots$, the main objective of this work is to construct a metric-semantic map of the observed environment online by estimating the object class sets $\calO_l = \left\{ \bfx \in \env \mid \min_{\bfeta} h_\cls(\bfx,\bfeta) \leq 0 \right\}$, implicitly represented by the TSDDFs $h_\cls(\bbx,\bfeta)$. Note that each object class is associated with a posterior distribution over sensor frames $\bfeta$. To reduce the complexity of estimating TSDDFs, which are defined for arbitrary directions $\bfeta$, we consider the more usual TSDF model, defined as the minimum of a TSDDF over $\bfeta$.

\begin{definition}
	The \emph{truncated signed distance function} (TSDF) $f_\cls(\bbx)$ of object class $\calO_\cls$ is the truncated signed distance from $\bbx\in\ccalX$ to the boundary $\partial\calO_\cls$, i.e.,
	\begin{align}\label{eq:H_signed_distance}
		f_\cls(\bbx):= h_\cls(\bbx,\bfeta^*) \ \text{where} \
		\bfeta^*=\argmin_{\bfeta }|h_\cls(\bbx,\bfeta)|.
	\end{align}
\end{definition}

We develop incremental sparse Gaussian Process regression to maintain distributions $\mathcal{GP}(\mu^\rb_{t,\cls}(\bfx), k^\rb_{t,\cls}(\bfx,\bfx'))$ over the TSDF functions $f_\cls(\bfx)$ in \eqref{eq:H_signed_distance} at each robot $\rb$, conditioned on the sensor observations $\left\{\lambda^\rb_{\tau,\msr},c^\rb_{\tau,\msr}\right\}$ up to time $t$. We propose a new data compression technique in Sec.~\ref{sec:GP-regression} and apply it in the design of the GP training algorithm for probabilistic TSDF inference in Sec.~\ref{sec:algorithm}. Our approach generates a continuous-space probabilistic model of the distance to and semantic classes of the environment surfaces. To achieve scalable online mapping of large domains, we train independent sparse GP models over an octree cover of the 3-D space. 

Next, we extend our approach from a centralized single-robot to a distributed multi-robot formulation. We develop new techniques for distributed incremental sparse GP regression in Sec.~\ref{sec:DistributedGP} and apply them to the collaborative semantic TSDF mapping problem in Sec.~\ref{sec:DistributedMSM}. Our method allows each robot to update its own sparse TSDF GP model, relying on local sensor observations and one-hop information exchange with its neighborhoods, yet guarantees theoretically that the model parameters of different robots converge in finite-time to the same parameters that would be obtained by centralized GP regression. The effectiveness of our approach is demonstrated in single- and multi-robot experiments using simulated 2-D data in Sec.~\ref{sec:evaluation2d} and real 3-D data in Sec.~\ref{sec:evaluation3d}.

\begin{figure}[t]
	\input{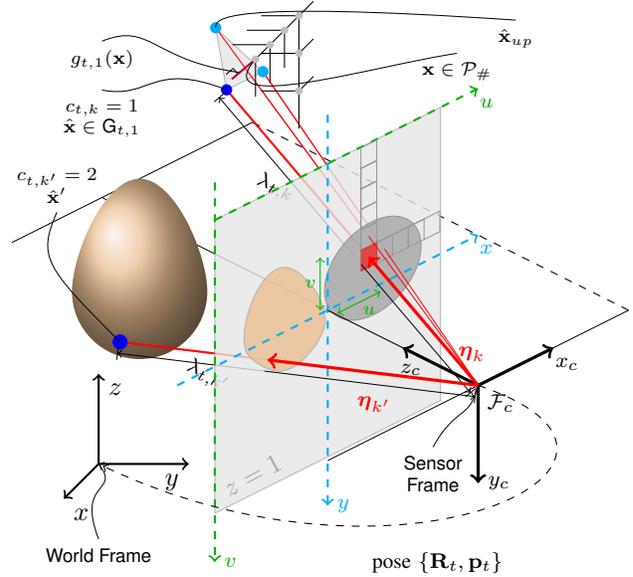}
		\caption{Sensor observation at time $t$ showing the distance $\lambda_{t,k}$, $\lambda_{t,k'}$ and class $c_{t,k}$, $c_{t,k'}$ measurements obtained along sensors rays $\bfeta_k$, $\bfeta_k' \in \bfE$ when a camera sensor is at position $\bfp_t$ with orientation $\bfR_t$. The pseudo points $\allPs$ (see Sec.~\ref{sec:training_set}) close to the observed surface are shown in gray.}
	\label{fig:observations}
\end{figure}

\section{Data Compression for Incremental Sparse Gaussian Process Regression}
\label{sec:GP-regression}

This section reviews sparse Gaussian Process regression and introduces a new approach for compressing training data acquired by repeated observation of the same locations, which is typical when an onboard robot sensor observes the same environment multiple times. Our data compression allows training a GP model with much fewer samples, yet provably generates the same GP posterior that would have been computed using the full uncompressed training set. Finally, the sparse GP model and the data compression allow us to design an efficient incremental GP algorithm that updates the GP posterior with sequential data instead of recomputing it from scratch.

\subsection{Background on Sparse GP Regression}
\label{sec:sparseGPreg}

A Gaussian Process is a set of random variables such that the joint distribution of any finite subset of them is Gaussian. A GP-distributed function $f(\bfx) \sim \mathcal{GP}(\mu_0(\bfx), k_0(\bfx, \bfx'))$ is defined by a mean function $\mu_0(\bfx)$ and a covariance (kernel) function $k_0(\bfx, \bfx')$. The mean and covariance are such that for any finite set $\calX = \crl{\bfx_1,\ldots,\bfx_M}$, the random vector $f(\calX) := \brl{f(\bfx_1),\ldots,f(\bfx_M)}^\top \in \mathbb{R}^{M}$ has mean with $j$-th element $\mu_0(\bfx_j)$ and covariance matrix with $(j,l)$-th element $k_0(\bfx_j, \bfx_l)$ for $j,l = 1,\ldots,M$. Given a training set $\calD=\{(\bfx_j, y_j)\}_{j=1}^M$, generated according to $y_j=f(\bfx_j)+\eta_j$ with independent Gaussian noise $\eta_j \sim \calN(0, \sigma^2)$, the posterior distribution of the random function $f(\bfx)$ can be obtained from the joint distribution of the value $f(\bfx)$ at an arbitrary location $\bfx$ and the random vector $\bfy := \brl{y_1,\ldots,y_M}^\top$ of measurements. In detail, the joint distribution is:
\begin{equation*}
\begin{aligned}
\begin{bmatrix} f(\bfx) \\ \bfy \end{bmatrix} \sim \mathcal{N}\prl{ \begin{bmatrix} \mu_0(\bfx) \\ \mu_0(\calX) \end{bmatrix}, \begin{bmatrix} k_0(\bfx,\bfx) & k_0(\bfx, \calX)\\
k_0(\calX,\bfx) & k_0(\calX,\calX) + \sigma^2 I\end{bmatrix}},
\end{aligned}
\end{equation*}
while the corresponding conditional distribution $f(\bfx)|\calX, \bfy \sim \mathcal{GP}(\mu(\bfx), k(\bfx, \bfx'))$ has mean and covariance functions:
\begin{equation}
\label{eq:gp_posterior}
\scaleMathLine{\begin{aligned}
\mu(\bfx) &:= \mu_0(\bfx) + k_0(\bfx, \calX) (k_0(\calX, \calX)+ \sigma^2 I)^{-1}(\bfy - \mu_0(\calX)),\\
k(\bfx,\bfx') &:=  k_0(\bfx, \bfx') - k_0(\bfx, \calX) (k_0(\calX, \calX) + \sigma^2 I)^{-1} k_0(\calX, \bfx').
\end{aligned}}
\end{equation}
Computing the GP posterior has cubic complexity in the number of observations $M$ due to the matrix inversion in~\eqref{eq:gp_posterior}.

Inspired by Snelson and Ghahramani~\cite{snelson2006sparse}, we introduce a sparse approximation to the GP posterior in~\eqref{eq:gp_posterior} using a set of \emph{pseudo points} $\calP\subset \calD$ whose number $|\calP| \ll M$. The key idea is to first determine the distribution $\calN\prl{\bfmu,\Sigma}$ of $\bff := f(\calP)$ conditioned on $\calX$, $\bfy$ according to~\eqref{eq:gp_posterior}:
\begin{align}\label{eq:pseudopoint_distribution}
\bfmu &:= \mu_0(\calP) + k_0(\calP, \calX) (k_0(\calX, \calX)+ \sigma^2 I)^{-1}(\bfy - \mu_0(\calX))\notag\\
 &\phantom{:}= \mu_0(\calP)+k_0(\calP,\calP)\prl{k_0(\calP,\calP) + \Gamma}^{-1}\bfgamma\\ 
\Sigma &:=  k_0(\calP,\calP) - k_0(\calP,\calX)\prl{ k_0(\calX,\calX) + \sigma^2 I }^{-1} k_0(\calX,\calP),\notag\\
 &\phantom{:}= k_0(\calP,\calP)\prl{k_0(\calP,\calP) + \Gamma}^{-1}k_0(\calP,\calP)\notag
\end{align}
where $\Gamma := k_0(\calP,\calX)\prl{\Lambda + \sigma^2 I}^{-1}k_0(\calX,\calP)$, $\Lambda := k_0(\calX,\calX) - k_0(\calX,\calP) k_0(\calP,\calP)^{-1} k_0(\calP,\calX)$, and $\bfgamma := k_0(\calP,\calX)\prl{\Lambda + \sigma^2 I}^{-1}(\bfy - \mu_0(\calX))$. Using the definitions of information matrix $\Omega := \Sigma^{-1}$ and information mean $\bfomega := \Omega\bfmu$, we can equivalently write:
\begin{equation}
\label{eq:pseudopoint_distribution_information}
\begin{aligned}
\bfomega &= \Omega\mu_0(\calP)+k_0(\calP,\calP)^{-1}\bfgamma,\\ 
\Omega &= k_0(\calP,\calP)^{-1}\prl{k_0(\calP,\calP) + \Gamma}k_0(\calP,\calP)^{-1}.
\end{aligned}
\end{equation}
%
Then, the posterior density of $f(\bfx)$ conditioned on $\calX, \bfy$ is:
\begin{equation}
p( f(\bfx) | \calX, \bfy) = \int p(f(\bfx) | \bff) p(\bff | \calX, \bfy) d\bff
\end{equation}
which is a GP with mean and covariance functions:
\begin{align}\label{eq:pseudopoint_gp_posterior_information}
\mu(\bfx) &=\mu_0(\bfx) + k_0(\bfx, \calP)k_0(\calP,\calP)^{-1}\prl{\Omega^{-1}\bbomega-\mu_0(\calP)}\notag\\
k(\bfx,\bfx') &=k_0(\bfx,\calP) k_0(\calP,\calP)^{-1}\Omega^{-1}k_0(\calP,\calP)^{-1}k_0(\calP,\bfx')\notag\\ 
&\quad + k_0(\bfx,\bfx') - k_0(\bfx,\calP)k_0(\calP,\calP)^{-1}k_0(\calP,\bfx').
\end{align}
If we assume that conditioned on $\calP$, the measurements $y_j$ are generated independently, i.e., $\Lambda$ is approximated by a diagonal matrix with elements $\lambda(\bfx_j) := k_0(\bfx_j,\bfx_j) - k_0(\bfx_j,\calP)k_0(\calP,\calP)^{-1}k_0(\calP,\bfx_j)$, then the complexity of computing $\bfmu$, $\Sigma$ in~\eqref{eq:pseudopoint_distribution} (training) and $\mu(\bfx)$, $k(\bfx,\bfx')$ in~\eqref{eq:pseudopoint_gp_posterior_information} (testing) are $O(|\calP|^2 |\calX| + |\calP|^3)$ and $O(|\calP|^2)$, respectively, instead of $O(|\calX|^3)$ and $O(|\calX|^2)$ without pseudo points in~\eqref{eq:gp_posterior}. The use of pseudo points leads to significant computational savings when $|\calP| \ll |\calX|$. We assume that the kernel parameters are optimized offline and focus on online computation of the terms in~\eqref{eq:pseudopoint_gp_posterior_information}, needed for prediction.


\subsection{Repeated Input Data Compression}


Next, we detail a way to obtain additional savings in terms of data storage requirements. Specifically, if the training data $\calD = (\calX,\bfy)$ contains repeated observations from the same locations, i.e., the points in $\calX$ are not unique, then the GP training complexity can be reduced from cubic in $|\calX|$ to cubic in the number of distinct points in $\calX$. We formalize this in the following proposition, which establishes that the GP posterior is unchanged if we compress the observations in $\bfy$ obtained from the same locations in $\calX$.

\begin{proposition}
\label{prop:gp_compression}
Consider $f(\bfx) \sim \mathcal{GP}(\mu_0(\bfx), k_0(\bfx,\bfx')$. Let:
\begin{equation*}
\begin{aligned}
\calX&=\{\bfx_1\;\;, \ldots,\bfx_1\;\;\;\;, \bfx_2\;\;,\ldots,\bfx_2\;\;\;\;,\ldots,\bfx_n\;\;,\ldots,\bfx_n\;\;\;\;\}\\
\bfy&=\brl{y_{1, 1}, \ldots, y_{1, m_1}, y_{2, 1}, \ldots, y_{2, m_2}, \ldots, y_{n, 1}, \ldots, y_{n, m_n}}^\top
\end{aligned}
\end{equation*}
be data generated from the model $y_{i, j} = f(\bfx_i) + \eta_{i,j}$ with $\eta_{i,j} \sim \mathcal{N}(0, \sigma^2)$ for $i=1,\ldots, n$ and $j = 1,\ldots, m_i$. Let:
\begin{equation}
\label{eq:compressed_data}
\begin{aligned}
\calP = \{\bfx_1, \ldots, \bfx_n\}, \;  
\bfzeta = \brl{ \frac{1}{m_1} \sum_{j=1}^{m_1} y_{1, j}, \ldots, \frac{1}{m_n} \sum_{j=1}^{m_n} y_{n, j} }^\top
\end{aligned}
\end{equation}
be a compressed version of the data generated from $f(\bfx_i)$ with noise $\hat{\eta}_{i} \sim \mathcal{N}(0, \frac{\sigma^2}{m_i})$. Then, $f(\bfx) | \calX,\bfy$ and $f(\bfx) | \calP, \bfzeta$ have the same Gaussian Process distribution $\mathcal{GP}(\mu(\bfx),k(\bfx,\bfx'))$ with:
\begin{equation}
\label{eq:compressed_gp_posterior}
\begin{aligned}
\mu(\bfx) &= \mu_0(\bfx) + k_0(\bfx, \calP) Z(\bfzeta - \mu_0(\calP)),\\
k(\bfx,\bfx') &=  k_0(\bfx, \bfx') - k_0(\bfx, \calP) Z k_0(\calP, \bfx'),
\end{aligned}
\end{equation}
where $Z^{-1} := k_0(\calP, \calP) + \sigma^2\diag(\bfm)^{-1}$ and $\bfm$ is a vector with elements $m_i$.
\end{proposition}

\begin{proof}
The distribution of $f(\bfx) | \calX,\bfy$ is provided in~\eqref{eq:gp_posterior}. Using the data $\calP$, $\bfzeta$, instead of $\calX$, $\bfy$, to compute the posterior GP distribution of $f(\bfx)$, according to~\eqref{eq:gp_posterior}, leads to the expression in~\eqref{eq:compressed_gp_posterior}. We need to show that~\eqref{eq:gp_posterior} and~\eqref{eq:compressed_gp_posterior} are equal given the relationship between $\calX$, $\bfy$ and $\calP$, $\bfzeta$ in \eqref{eq:compressed_data}. Let $E$ be a binary matrix defined such that $k_0(\calX,\bfx) = E k_0(\calP,\bfx)$. Note that $k_0(\calX,\calX) = E k_0(\calP,\calP)E^\top$, $k_0(\bfx,\calX) = k_0(\bfx,\calP)E^\top$, $E^\top E = \diag(\bfm)$, and $\bfzeta = (E^\top E)^{-1} E^\top \bfy$. Using these expressions in~\eqref{eq:gp_posterior} leads to:
\begin{align}\label{eq:E_gp_posterior}
\mu(\bfx) &= \mu_0(\bfx) +\notag\\ 
&k_0(\bfx, \calP)E^\top (Ek_0(\calP, \calP)E^\top+ \sigma^2 I)^{-1}(\bfy - E\mu_0(\calP)),\notag\\
k(\bfx,\bfx') &=  k_0(\bfx, \bfx') - \\
&k_0(\bfx, \calP)E^\top (Ek_0(\calP, \calP)E^\top + \sigma^2 I)^{-1} Ek_0(\calP, \bfx').\notag
\end{align}
An application of the matrix inversion lemma followed by algebraic manipulation shows that $E^\top (Ek_0(\calP, \calP)E^\top + \sigma^2 I)^{-1} = \prl{k_0(\calP, \calP) + \sigma^2 (E^\top E)^{-1}}^{-1} (E^\top E)^{-1}E^\top = Z (E^\top E)^{-1}E^\top$. Replacing this and $\bfzeta = (E^\top E)^{-1} E^\top \bfy$ in~\eqref{eq:E_gp_posterior} shows that the GP distributions of $f(\bfx) | \calX,\bfy$ and $f(\bfx) | \calP, \bfzeta$ are equal.
\end{proof}

Prop.~\ref{prop:gp_compression} allows us to summarize a training set $\calX$, $\bfy$ by keeping the distinct points $\calP \subset \calX$ as well as the average observation value $\zeta(\bfp)$ and number of times $m(\bfp)$ that each point $\bfp \in \calP$ has been observed. Given these statistics, the mean function $\mu(\bfx)$ and covariance function $k(\bfx,\bfx')$ of the posterior GP can be obtained according to~\eqref{eq:compressed_gp_posterior} with $\bfzeta := \zeta(\calP)$ and $\bfm := m(\calP)$. When the training points $\calX$ contain many repetitions, the subset $\calP$ of distinct points is a natural choice of pseudo points (Sec.~\ref{sec:sparseGPreg}) and, in this case, the posterior obtained from training with $\calP$ is \emph{exact} (Prop.~\ref{prop:gp_compression}) instead of an approximation of the posterior obtained from training with $\calX$. We exploit this compression technique for efficient incremental GP training when the same observations are observed multiple times.



\subsection{Incremental Compressed Sparse GP Regression}
\label{sec:incGPreg}

Suppose now that, instead of a single training set $\calD$, the data are provided sequentially, i.e., an additional dataset $\tilde{\calD}_t$ of points $\tilde{\calX}_t$ with labels $\tilde{\bfy}_t$ is provided at each time step $t$. The cumulative data up to time $t$ are $\calD_t := \cup_{\tau = 1}^t \tilde{\calD}_\tau$. Based on Prop.~\ref{prop:gp_compression}, we can define an incrementally growing set of pseudo points $\calP_t$ with associated number of observations $m_t(\bfp)$ and average observation $\zeta_t(\bfp)$ for $\bfp \in \calP_t$ and observation precision $Z_t$. We show how to update these statistics when a new dataset $\tilde{\calD}_{t+1} = (\tilde{\calX}_{t+1}, \tilde{\bfy}_{t+1})$ arrives at time $t+1$. Let $\tilde{\calP}_{t+1}$ be the set of unique points in $\tilde{\calX}_{t+1}$ with number of observations $\tilde{m}_{t+1}(\bfp)$ and average observation $\tilde{\zeta}_{t+1}(\bfp)$ for $\bfp \in \tilde{\calP}_{t+1}$. The update of $\calP_t$, $m_t(\bfp)$ and $\zeta_t(\bfp)$ is:
\begin{equation}
\label{eq:update}
\begin{aligned}
\calP_{t+1} &= \calP_t \cup \tilde{\calP}_{t+1}\\
m_{t+1}(\bfp) &= \begin{cases}
      m_t(\bfp) + \tilde{m}_{t+1}(\bfp), & \text{if } \bfp \in \calP_t,\\
      \tilde{m}_{t+1}(\bfp), & \text{else},
      \end{cases}\\
\zeta_{t+1}(\bfp) &= \begin{cases}
      \frac{m_t(\bfp)\zeta_t(\bfp) + \tilde{m}_{t+1}(\bfp)\tilde{\zeta}_{t+1}(\bfp)}{m_{t+1}(\bfp)}, & \text{if } \bfp \in \calP_t,\\
      \tilde{\zeta}_{t+1}(\bfp), & \text{else}.
      \end{cases}
\end{aligned}
\end{equation}
To update the observation precision $Z_t$, first consider the existing pseudo points $\calP_t$. Let $l$ be the index of $\bfp \in \calP_t$ in $Z_t$. Define $\epsilon_l := \sigma^2 \prl{\frac{1}{m_{t+1}(\bfp)} - \frac{1}{m_{t}(\bfp)}}$, $B_0 := Z_t$, and for $l = 1,\ldots,|\calP_t|$:
\begin{equation}
\label{eq:precisionUpdate1}
B_{l+1} = \prl{B_l^{-1} +  \epsilon_l \bfe_l \bfe_l^\top}^{-1} = B_l - \frac{B_l \bfe_l \bfe_l^\top B_l}{\frac{1}{\epsilon_l} + \bfe_l^\top B_l \bfe_l}.
\end{equation}
With some abuse of notation, let $B := B_{|\calP_t|}$ be the observation precision after all $\bfp \in \calP_t$ have been updated. Finally, we update $B$ by introducing the pseudo points $\tilde{\calP}_{t+1}\setminus \calP_t$ that have been observed for the first time:
\begin{equation}
\label{eq:precisionUpdate2}
Z_{t+1} = \begin{bmatrix} B^{-1} & C\\ C^\top & D \end{bmatrix}^{-1} = \begin{bmatrix} B + BCSC^\top B & -BCS\\ -S C^\top B & S \end{bmatrix},
\end{equation}
where $C := k_0(\calP_t,\tilde{\calP}_{t+1}\setminus \calP_t)$, $D := k_0(\tilde{\calP}_{t+1}\setminus \calP_t,\tilde{\calP}_{t+1}\setminus \calP_t) + \sigma^2 \diag(\tilde{m}_{t+1}(\tilde{\calP}_{t+1}\setminus \calP_t))^{-1}$, and $S := (D  - C^\top B C)^{-1}$. By recursively tracking these matrix inverses, the posterior update can be executed efficiently every time a new observation arrives with complexity that is cubic in the number of new distinct points. This is a significant improvement over na\"ive GP training. 

Unfortunately, this complexity still exhibits computational bottlenecks over large domains, where the number of pseudo points $\calP_t$ continues to grow with $t$. Returning to the TSDF mapping problem, this situation happens when a robot continuously explores a large 3-D environment. We introduce an octree spatial decomposition with overlapping subregions, allowing us to train independent GPs with a fixed maximum number of pseudo points in each subregion. This aspect, as well as how the training sets are constructed from the robot observations, discussed in Sec. \ref{sec:problem_formulation}, and utilized for probabilistic semantic TSDF mapping are the focus of the following section.






\section{Probabilistic Metric-Semantic Mapping}
\label{sec:algorithm}

In this section, we consider the single-robot mapping problem. For simplicity of notation, we suppress the superscript $\rb$ that denotes the robot index. The sensor measurements $\crl{\lambda_{t,\msr},c_{t,\msr}}$ are generated according to the models in \eqref{eq:class_measurement} and \eqref{eq:data} that depend on the TSDDFs $\crl{h_\cls(\bfx,\bfeta)}$ of the different semantic classes in the environment. As mentioned in Sec.~\ref{sec:problem_formulation}, instead of $\crl{h_\cls(\bfx,\bfeta)}$, we focus on estimating the TSDFs $\crl{f_\cls(\bfx)}$, whose domains are lower-dimensional. We apply the incremental GP regression technique developed in Sec.~\ref{sec:GP-regression}. Since the sensor data $\crl{\lambda_{t,\msr},c_{t,\msr}}$ are not direct samples from the TSDFs, they need to be transformed into training sets $\D{t,\cls}$, suitable for updating the GP distributions of $\crl{f_\cls(\bfx)}$.


%

\subsection{Training Set Construction}
\label{sec:training_set}

The class measurements allow us to associate the sensor data with particular semantic classes, while the distance measurements allow us to estimate the points where the sensor rays hit the object sets $\calO_\cls$. We define the following point sets for each detected semantic class at time $t$:
\begin{align}\label{Gtc}
 \setG_{t,\cls}=\{ \hat{\bfx} \in \mathbb{R}^3 \given \hat{\bbx} = \lambda_{t,\msr}\bbR_t\bbeta_\msr+\bbp_t \text{ and } c_{t,\msr}=\cls\}.
\end{align}
The values $f_\cls(\hat{\bbx})$ of the TSDFs are close to zero at points $\hat{\bbx} \in \setG_{t,\cls}$ because the sensor rays hit an object surface close to these locations.

As shown in Prop.~\ref{prop:gp_compression}, the complexity of online GP training can be improved by forcing the training data to repeatedly come from a finite set of points. We choose a grid discretization $\allPs$ of the workspace $\calW$ and construct a training set by selecting points $\bbx \in \allPs$, that are at most $\epsilon > 0$ away from the points $\hat{\bfx} \in \setG_{t,\cls}$, and approximating their TSDF values $f_\cls(\bfx) \approx g_{t,\cls}(\bfx)$ (see Fig.~\ref{fig:observations}). Precisely, the training data sets are constructed at time $t$ as:
\begin{equation}\label{defD}
\D{t,\cls} = \{ (\bbx, g_{t,\cls}(\bbx)) |\bbx \in \allPs, \exists \hat{\bbx} \in \setG_{t,\cls} \text{ s.t. } ||\bbx - \hat{\bbx}||_2\leq\epsilon\}.
\end{equation}
In the case of a camera sensor, the TSDF value $g_{t,\cls}(\bfx)$ of a pseudo point $\bfx$ is obtained by projecting $\bfx$ to the image plane and approximating its distance from the distance values of nearby pixels. In detail, suppose $\bfeta_\msr$ is the unit vector corresponding to the pixel closest to the projection of $\bfx$ (red pixel in Fig.~\ref{fig:observations}) and let $\hat{\bfx} \in \setG_{t,\cls}$ be the coordinates of its ray endpoint (blue point in Fig.~\ref{fig:observations}). Let $\hat{\bbx}_{right}$ and $\hat{\bbx}_{up}$ (two cyan points in Fig.~\ref{fig:observations}) be the ray endpoints of two adjacent pixels. Then, $g_{t,\cls}(\bbx)$ is the signed distance from $\bfx$ to the plane defined by $\hat{\bfx}$, $\hat{\bbx}_{right}$, and $\hat{\bbx}_{up}$:
\begin{equation}
\begin{gathered}
g_{t,\cls}(\bbx) :=  \bbn^{\top} (\bbx-\hat{\bbx}), \quad \bbn := \sign(\bbq^{\top} (\bbp_t-\hat{\bbx}))\bbq,\\
\bbq = \frac{(\hat{\bbx}_{right}-\hat{\bbx}) \times (\hat{\bbx}_{up}-\hat{\bbx})}{\|(\hat{\bbx}_{right}-\hat{\bbx}) \times (\hat{\bbx}_{up}-\hat{\bbx})\|},
\end{gathered}
\end{equation}
where $\bbq$ is the normal of the plane and the signed distance from $\bfp_t$ to the plane is positive because the sensor is known to be outside of the object set $\calO_\cls$. With the input variables as distance observations and target variables as truncated signed distance field specified, we shift to how actually compute the posterior inference.

\subsection{Incremental TSDF Inference}
\label{sec:GP_update}
Recall that we are using streaming measurements to update the GP distributions of the TSDFs $\crl{f_\cls(\bfx)}$. We derived an incremental sparse GP update in Sec.~\ref{sec:incGPreg}. Here, we use the transformed TSDF training data $\D{t,\cls}$ to update the GP distribution for each class $\cls$. At time $t$, the new data are $\D{t,\cls} = (\X{t,\cls}, \Y{t,\cls})$ and the new pseudo points are $\P{t,\cls} =  \X{t,\cls} \setminus \calP_{t-1,\cls}$. Given $\X{t, \cls}$, $\Y{t, \cls}$, $\P{t, \cls}$ for each class $\cls$, we can update $\cP{t, \cls}$, $\zeta_{t,\cls}$, $m_{t,\cls}$ via \eqref{eq:update}. If online prediction is required, we can also update the precision matrix $\Z{t,\cls}$ using \eqref{eq:precisionUpdate1} and \eqref{eq:precisionUpdate2}. Then, we have the GPs of all classes updated, and can predict the TSDF at any query point according to \eqref{eq:compressed_gp_posterior}. Next we discuss how the inferred posterior may be employed to construct a semantic category prediction.


\subsection{Semantic Category Prediction}
\label{sec:class}

Next, we discuss how to predict the semantic class labels on the surfaces of the implicitly estimated object sets $\calO_\cls$. While we did not explicitly model noise in the class observations in~\eqref{eq:class_measurement}, in practice, semantic segmentation algorithms may produce incorrect pixel-level classification. This leads to some sensor observations $\lambda_{t,\msr}$, $c_{t,\msr}$ being incorrectly included into the training set $\D{t,\cls}$ of a different semantic class. This happens, for example, if objects from two different classes, say $\cls_1$ and $\cls_2$, are spatially close to each other and, in an RGB image, parts of the boundary of one are classified as belonging to the other class. Over time, with multiple sensor observations, the TSDF approximations for both classes $\cls_1$ and $\cls_2$ may contain pseudo points $\bfx \in \allPs$ with small TSDF values, indicating an object surface at the same location. To predict the correct semantic class, we compare the likelihoods of the different classes at surface points using the posterior GP distributions of the TSDFs $f_\cls(\bfx)$. 


\begin{proposition}
\label{prop:class_prediction}
Let $\mathcal{GP}(\mu_{t,\cls}(\bfx), k_{t,\cls}(\bfx,\bfx'))$ be the distributions of the truncated signed distance functions $f_\cls(\bfx)$ at time $t$, determined according to \eqref{eq:compressed_gp_posterior}. Consider an arbitrary point $\bfx \in \partial \calO$ on the surface of the obstacle set, i.e., $\bfx$ is such that $f_\cls(\bfx) = 0$ for some class $\cls \in \{1, \ldots, \classN\}$. Then, the probability that the true class label of $\bfx$ is $c \in \{1,\ldots,\classN\}$ is:
\begin{equation*}\label{eq:class_probability}
\scaleMathLine{\mathbb{P}\prl{\argmin_\cls|f_\cls(\bfx)| = c \;\bigg\vert\; \min_\cls |f_\cls(\bbx)| = 0} = \frac{\frac{1}{\sigma_{t,c}(\bfx)}\phi(\frac{\mu_{t,c}(\bfx)}{\sigma_{t,c}(\bfx)})}{\sum_\cls \frac{1}{\sigma_{t,\cls}(\bfx)}\phi(\frac{\mu_{t,\cls}(\bfx)}{\sigma_{t,\cls}(\bfx)})},}
\end{equation*}
where $\phi(\cdot)$ is the probability density function of the standard normal distribution and $\sigma_{t,\cls}(\bfx) := \sqrt{ k_{t,\cls}(\bfx,\bfx)}$.
\end{proposition}

\begin{proof}
~\\
Let $l_c(z):=\mathbb{P}\prl{\argmin_\cls|f_\cls(\bfx)| = c \text{ and } \min_\cls |f_\cls(\bbx)| \leq |z|}$. Since $\mathbb{P}\prl{\min_\cls |f_\cls(\bbx)| \leq |z|}=\sum_\cls l_\cls(z)$:
\begin{equation*}
\scaleMathLine{\mathbb{P}\prl{\argmin_\cls|f_\cls(\bfx)| = c \;\bigg\vert\; \min_\cls |f_\cls(\bbx)| \leq |z|} =\frac{l_c(z)}{\sum_\cls l_\cls(z)}}
\end{equation*}
The term we are interested in computing is $\lim_{z \to 0} \frac{l_c(z)}{\sum_\cls l_\cls(z)}$. Let $\bfx$ be an arbitrary (test) point and define $\mu_\cls := \mu_{t,\cls}(\bfx)$ and $\sigma_\cls := \sigma_{t,\cls}(\bfx)$ for $\cls = 1,\ldots,\classN$. The GP distribution of $f_\cls$ stipulates that its value at $\bfx$ has a density function $p(z) = \frac{1}{\sigma_\cls} \phi\big(\frac{z-\mu_\cls}{\sigma_\cls}\big)$. Hence, $\mathbb{P}(|f_\cls(\bbx)|\geq z) = 1-\Phi(\frac{|z|-\mu_\cls}{\sigma_\cls})+\Phi(\frac{-|z|-\mu_\cls}{\sigma_\cls})$. Note that $l_c(z)$ corresponds to the probability that $|f_c(\bbx)| \leq |f_\cls(\bbx)|$ for all $\cls$. Since all $f_\cls$ are independent of each other:
\begin{equation*}
\scaleMathLine{l_c(z) = \frac{1}{\sigma_c}\int_{-z}^z \phi\big(\frac{\zeta-\mu_c}{\sigma_c}\big)\prod_{\cls\neq c}\biggl(1-\Phi\bigl(\frac{|\zeta|-\mu_\cls}{\sigma_\cls}\bigr)+\Phi\bigl(\frac{-|\zeta|-\mu_\cls}{\sigma_\cls}\bigr)\biggr)d\zeta}
\end{equation*}
The claim is concluded by $\lim\limits_{z \to 0}\frac{l_c(z)}{2z}=\frac{1}{\sigma_c}\phi\big(\frac{-\mu_c}{\sigma_c}\big)$. 
\end{proof}

The class distribution for an arbitrary point $\bfx \in \calW$, not lying on an object surface, may also be obtained, as shown in the proof of Prop.~\ref{prop:class_prediction} but is both less efficient to compute and rarely needed in practice.




\subsection{Octree of Gaussian Processes}
\label{sec:octree}

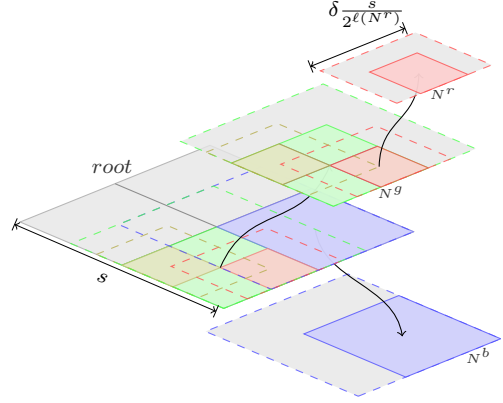
\begin{figure}[t]
\centering
\vspace{-20mm}

\def\rota{-23}
\rotatebox{\rota}{
\begin{tikzpicture}
\usetikzlibrary{fadings}
\tikzfading[name=fade inside,
inner color=transparent!-10,
outer color=transparent!3]
\def\x{5/1.7}
\def\y{-1.5}
\def\z{5}
\def\xShiftLlow{\x/3}
\def\yShiftLlow{\y}
\def\zShiftLlow{-\z/3}
\begin{scope}[shift={(\xShiftLlow,\yShiftLlow,\zShiftLlow)}]
\filldraw[dashed, fill=gray!20, draw=blue!70, opacity=0.8] (\x,0,\z*3/4) -- (\x,0,\z*3/4-\z*3/4)-- (\x-\x*3/4,0,\z*3/4-\z*3/4) --(\x-\x*3/4,0,\z*3/4)--(\x,0,\z*3/4);
\filldraw[fill=blue!20, draw=blue!70, opacity=0.8] (\x,0,\z*3/4-\z*3/4) -- (\x,0,\z*3/4-\z*3/4+\z/2) -- (\x-\x/2,0,\z*3/4-\z*3/4+\z/2) --(\x-\x/2,0,\z*3/4-\z*3/4)--(\x,0,\z*3/4-\z*3/4) node[anchor=north east,align=center,font=\sffamily\tiny, rotate=-\rota] {$N^{b}$};
\shade[ball color=cyan, path fading=fade inside] (\x-\x/16,0,\z/2-\z/16) circle (0.08);
\end{scope}
\draw[thin,->] (\x*3/4,0,\z/4) .. controls (\x*3/4+\x/8,0+\yShiftLlow*2/5,\z/4) and (\x*3/4-\x/8+\xShiftLlow,0+\yShiftLlow-\yShiftLlow*2/5,\z/4+\zShiftLlow) ..  (\x*3/4+\xShiftLlow,0+\yShiftLlow,\z/4+\zShiftLlow) ;
%
\filldraw[fill=gray!20, draw=gray!70, opacity=0.8] (0,0,0) -- (0,0,\z)-- (\x,0,\z) --(\x,0,0)--(0,0,0) node[anchor= east , yshift = -0.35*\z 0pt, xshift = -3\x 0pt,rotate=-\rota,font=\footnotesize]{$root$};
\draw[thin, solid, gray!90] (0,0,\z/2) -- (\x,0,\z/2);
\draw[thin, solid, gray!90] (\x/2,0,\z) -- (\x/2,0,0);
\draw[thin,solid,gray!90] (0+\x/2,0,\z/4+\z/2) -- (\x/2+\x/2,0,\z/4+\z/2);
\draw[thin,solid,gray!90] (\x/4+\x/2,0,\z/2+\z/2) -- (\x/4+\x/2,0,0+\z/2);
\filldraw[ fill=green!20, draw=green!70, opacity=0.8] (\x/2,0,\z/2) -- (\x,0,\z/2)-- (\x,0,\z)--(\x/2,0,\z)--(\x/2,0,\z/2);
\filldraw[fill=blue!20, draw=blue!50, opacity=0.8] (\x,0,\z*3/4-\z*3/4) -- (\x,0,\z*3/4-\z*3/4+\z/2) -- (\x-\x/2,0,\z*3/4-\z*3/4+\z/2) --(\x-\x/2,0,\z*3/4-\z*3/4)--(\x,0,\z*3/4-\z*3/4);
\filldraw[fill=red!20, draw=red!50, opacity=0.8] (\x,0,\z*3/4) -- (\x,0,\z*3/4-\z/4) -- (\x-\x/4,0,\z*3/4-\z/4) -- (\x-\x/4,0,\z*3/4)--(\x,0,\z*3/4);
\def\xShiftLlow{-\x/4}
\def\yShiftLlow{0}
\def\zShiftLlow{\z/4}
\begin{scope}[shift={(\xShiftLlow,\yShiftLlow,\zShiftLlow)}]
\filldraw[fill=olive!20, draw=olive!50, opacity=0.8] (\x,0,\z*3/4) -- (\x,0,\z*3/4-\z/4) -- (\x-\x/4,0,\z*3/4-\z/4) -- (\x-\x/4,0,\z*3/4)--(\x,0,\z*3/4);
\end{scope}
\draw[dashed, draw=blue!70, opacity=0.8] (\x,0,\z*3/4) -- (\x,0,\z*3/4-\z*3/4)-- (\x-\x*3/4,0,\z*3/4-\z*3/4) --(\x-\x*3/4,0,\z*3/4)--(\x,0,\z*3/4);
\draw[dashed, draw=red!70, opacity=0.8] (\x,0,\z*7/8) -- (\x-\x*3/8,0,\z*7/8)-- (\x-\x*3/8,0,\z*7/8-\z*4/8) --(\x,0,\z*7/8-\z*4/8)--(\x,0,\z*7/8);
\draw[dashed, draw=olive!70, opacity=0.8] (\x*3/8,0,\z) -- (\x*3/8,0,\z-\z*3/8)-- (\x*7/8,0,\z-\z*3/8) --(\x*7/8,0,\z)--(\x*3/8,0,\z);
%
\draw[dashed, draw=green!70, opacity=0.8] (\x,0,\z) -- (\x-\x*3/4,0,\z)-- (\x-\x*3/4,0,\z-\z*3/4) --(\x,0,\z-\z*3/4)--(\x,0,\z);
\draw[very thin, solid,|<->|] (0,0,\z*1.03) -- (\x,0,\z*1.04) node[anchor= north west, yshift = 0 0pt, xshift = -17*\x 0pt,rotate=0,font=\footnotesize]{$s$};
\shade[ball color=cyan, path fading=fade inside] (\x-\x/16,0,\z/2-\z/16) circle (0.08);
\shade[ball color=blue, path fading=fade inside] (\x/2-\x/16,0,\z-\z/16) circle (0.08);
%
\def\xShiftLone{\x/6}
\def\yShiftLone{-\relFirY*\y}
\def\zShiftLone{-\z/6}
\def\relFirY{1}
\draw[thin,->] (\x*3/4,0,\z*3/4) .. controls (\x*3/4,\yShiftLone*5/8,\z*3/4) and (\x*3/4+\xShiftLone,\yShiftLone*3/8,\z*3/4+\zShiftLone) ..  (\x*3/4+\xShiftLone,\yShiftLone,\z*3/4+\zShiftLone) ;
\begin{scope}[shift={(\xShiftLone,\yShiftLone,\zShiftLone)}]
%
\filldraw[dashed, fill=gray!20, draw=green!70, opacity=0.8] (\x,0,\z) -- (\x-\x*3/4,0,\z)-- (\x-\x*3/4,0,\z-\z*3/4) --(\x,0,\z-\z*3/4)--(\x,0,\z) node[anchor=west, yshift=2*\z, xshift=2*\z,align=center,font=\sffamily\tiny, rotate=-\rota] {$N^g$};
\filldraw[ fill=green!20, draw=green!70, opacity=0.8] (\x/2,0,\z/2) -- (\x,0,\z/2)-- (\x,0,\z)--(\x/2,0,\z)--(\x/2,0,\z/2);
\draw[thin,solid,gray!90] (\x/2,0,\z*3/4) -- (\x,0,\z*3/4);
\draw[thin,solid,gray!90] (\x*3/4,0,\z) -- (\x*3/4,0,\z/2);
\filldraw[fill=red!20, draw=red!70, opacity=0.8] (\x,0,\z*3/4) -- (\x,0,\z*3/4-\z/4) -- (\x-\x/4,0,\z*3/4-\z/4) -- (\x-\x/4,0,\z*3/4)--(\x,0,\z*3/4);
\def\xShiftLlow{-\x/4}
\def\yShiftLlow{0}
\def\zShiftLlow{\z/4}
\begin{scope}[shift={(\xShiftLlow,\yShiftLlow,\zShiftLlow)}]
\filldraw[fill=olive!20, draw=olive!70, opacity=0.8] (\x,0,\z*3/4) -- (\x,0,\z*3/4-\z/4) -- (\x-\x/4,0,\z*3/4-\z/4) -- (\x-\x/4,0,\z*3/4)--(\x,0,\z*3/4);
\end{scope}

\draw[dashed, draw=red!70, opacity=0.8] (\x,0,\z*7/8) -- (\x-\x*3/8,0,\z*7/8)-- (\x-\x*3/8,0,\z*7/8-\z*4/8) --(\x,0,\z*7/8-\z*4/8)--(\x,0,\z*7/8);

\draw[dashed, draw=olive!70, opacity=0.8] (\x*3/8,0,\z) -- (\x*3/8,0,\z-\z*3/8)-- (\x*7/8,0,\z-\z*3/8) --(\x*7/8,0,\z)--(\x*3/8,0,\z);

%
\shade[ball color=cyan, path fading=fade inside] (\x-\x/16,0,\z/2-\z/16) circle (0.08);
\shade[ball color=blue, path fading=fade inside] (\x/2-\x/16,0,\z-\z/16) circle (0.08);
\end{scope}
\def\xShiftLtwo{\x/6}
\def\yShiftLtwo{-\rely*\y}
\def\zShiftLtwo{-\z/6}
\def\rely{1.9}
\draw[thin,->] (\x*7/8+\xShiftLone,\yShiftLone,\z*5/8 +\zShiftLone) .. controls (\x*7/8+\xShiftLone,\yShiftLone+\yShiftLtwo*5/8-\yShiftLone*5/8,\z*6/8 +\zShiftLone) and (\x*7/8+\xShiftLtwo,\yShiftLtwo-\yShiftLtwo*5/8+\yShiftLone*5/8,\z*4/8 +\zShiftLtwo) ..  (\x*7/8+\xShiftLtwo,\yShiftLtwo,\z*5/8 +\zShiftLtwo) ;
\begin{scope}[shift={(\xShiftLtwo, \yShiftLtwo, \zShiftLtwo)}]
\filldraw[dashed,fill=gray!20, draw=red!70, opacity=0.8] (\x,0,\z*7/8) -- (\x-\x*3/8,0,\z*7/8)-- (\x-\x*3/8,0,\z*7/8-\z*4/8) --(\x,0,\z*7/8-\z*4/8)--(\x,0,\z*7/8);
\draw[very thin, solid,|<->|] (\x-\x*3/8-0.05*\x,0,\z*7/8) -- (\x-\x*3/8-0.05*\x,0,\z*7/8-\z*4/8) node[anchor= north west, yshift = 0 0pt, xshift = -12*\x 0pt,rotate=0,font=\footnotesize, rotate=-\rota]{$\delta\frac{s}{2^{\ell(N^r)}}$};
\filldraw[fill=red!20, draw=red!70, opacity=0.8] (\x,0,\z*3/4) -- (\x,0,\z*3/4-\z/4) -- (\x-\x/4,0,\z*3/4-\z/4) -- (\x-\x/4,0,\z*3/4)--(\x,0,\z*3/4)node[anchor=west,align=center,font=\sffamily\tiny, rotate=-\rota] {$N^r$};
\shade[ball color=cyan, path fading=fade inside] (\x-\x/16,0,\z/2-\z/16) circle (0.08);
%
\end{scope}
%

\end{tikzpicture}
}
\vspace{-7mm}
\caption{Illustration of an octree data structure, containing two pseudo points (blue and cyan) in two dimensions. The support regions $\calS(\cdot)$ and test regions $\calT(\cdot)$ of three nodes $N^r$, $N^g$, $N^b$ are shown as dashed and filled areas with red, green, and blue color, respectively. No pseudo points are contained in the test region $\calT(N^g)$ (filled green) of node $N^g$ but two pseudo points are in its support region $\calS(N^g)$ (dashed green). In this example, the maximum number of allowable pseudo points for each region is $max(N)=1$, so node $N^g$ is split into the red ($N^r$) and yellow (not labeled) regions. The cyan pseudo point belongs to both $\calP_{t,\cls}(N^b)$ and $\calP_{t,\cls}(N^r)$.}
\label{fig:quadtree}
\end{figure}

Even after compressing the TSDF training data using Prop.~\ref{prop:gp_compression} to a small set of distinct pseudo points, the GP training complexity still scales cubically with the number of pseudo points. To ensure that online training is possible for large environments, we develop an octree data structure with overlapping octant regions to store the pseudo points. We train independent GPs in each of these regions, which is efficient since the maximum number of pseudo points per region is fixed. The region overlap serves to eliminate discontinuities in the resulting TSDF estimate. At test time, the TSDF value of a query point is inferred using only the parameters of the corresponding region according to \eqref{eq:compressed_gp_posterior}. The overlapping octant regions are illustrated in Fig.~\ref{fig:quadtree}.

Formally, an octree of pseudo points is a tree data structure such that each internal node has at most eight children. Each node is associated with a spatial region in the 3-D workspace $\env$. The root is associated with a cube with side length $s > 0$, which is recursively subdivided into up to eight overlapping octant regions by the eight child nodes. Each node $N$ maintains the following information:


\begin{enumerate}

\item $\ell(N) \geq 0$: level of $N$ in the tree, starting from $0$ at the root node.

\item $ctr(N) \in \reals^3$: center of the region associated with $N$.

\item $\calS(N) := \{\bbx \in \env |\; \|\bbx-ctr(N)\|_\infty \leq \delta \frac{s}{2^{\ell(N)+1}}\}$: support region of $N$ with $\delta > 1$.

\item $\calT(N) := \{\bbx \in \env |\; \|\bbx-ctr(N)\|_\infty \leq \frac{s}{2^{\ell(N)+1}}\}$: test region of $N$.

\item $\calP(N) \subseteq \calS(N) \cap \allPs$: set of pseudo points assigned to this node

\item $max(N)$: node $N$ splits into eight children if the number of observed pseudo points $\calP(N)$ exceeds $max(N)$

\item $children(N)$: empty set if $N$ is a leaf and, otherwise, a set of eight nodes at level $\ell(N)+1$ with centers in $\{ctr(N)+s_x\bbe_1+s_y\bbe_2+s_z\bbe_3|s_x,s_y,s_z\in\{-\frac{s}{2^{\ell(N)+1}},+\frac{s}{2^{\ell(N)+1}}\}\}$.

\end{enumerate}

The pseudo points $\calP_{t,\cls}$ observed up to time $t$ (see Sec.~\ref{sec:GP_update}) are stored in octree data structures for each class $\cls$. The points assigned to node $N$ for class $\cls$ at time $t$ are $\calP_{t,\cls}(N) := \calP_{t,\cls}\cap \calS(N)$. The pseudo points $\calP_{t,\cls}(N)$ of each leaf node $N$ are used to train an independent GP. At time step $t$, prediction associated with each class $l$ for test points in the region $\calT(N)$ is performed by the GP associated with node $N$. The class distribution of test points with small predicted TSDF values (surface points) is determined according to Prop.~\ref{prop:class_prediction}. With the data structure developed for efficient representations of possibly large domains, we next shift to how the proposed incremental posterior inference scheme may be decentralized across a collection of interconnected robots.

\section{Distributed Incremental Sparse\\ GP Regression}
\label{sec:DistributedGP}
In this section, we develop a distributed version of the incremental sparse GP regression in Sec.~\ref{sec:GP-regression}. We consider $n$ robots, communicating over a network $G = (\calV, \calE)$. Each robot $\rb \in \calV$ receives its own local observations $\tilde{\calD}_{t}^\rb = (\tilde{\calX}_t^\rb, \tilde{\bfy}_t^\rb)$ at time $t$ and extracts newly observed pseudo points $\tilde{\calP}_t^\rb$, with associated number of observations $\tilde{\bfm}_t^\rb$ and average values $\tilde{\bfzeta}_t^\rb$, as detailed in Sec.~\ref{sec:incGPreg}. This information is used to update the complete set of pseudo points $\calP_t^\rb$ observed up to time $t$, along with the number of observations $\bfm_t^\rb$ and average values $\bfzeta_t^\rb$, according to \eqref{eq:update}. These parameters $\Theta_t^\rb := \crl{\calP_t^\rb, \bfm_t^\rb, \bfzeta_t^\rb}$, maintained by robot $\rb$, define a complete GP distribution for the function $f(\bfx)$, with mean and covariance functions in \eqref{eq:compressed_gp_posterior}.

While each robot may estimate $f(\bfx)$ individually, we consider how the robots may exchange information to estimate $f(\bfx)$ collaboratively. Our approach is inspired by techniques extending network consensus \cite{consensus} to distributed probabilistic estimation \cite{rad2010distributed,jadbabaie2012non,atanasov2014joint,nedic2016distributed,nedic2017distributed}. We observe that the continuous-space GP distribution of $f(\bfx)$ is induced by the statistics $\bfm_t^\rb$, $\bfzeta_t^\rb$ associated with the finite number of pseudo points $\calP_t^\rb$ and, hence, if the robots exchange information about and agree on these finite-dimensional parameters, then the corresponding GP distributions of $f(\bfx)$ at each robot will agree. Our main innovation is a distributed algorithm for updating the sparse GP parameters of one robot using the parameters of its one-hop neighbors' distributions. While existing results apply to fixed finite-dimensional parameter estimation, our approach applies to function estimation with an infinite-dimensional GP distribution, updated via consensus on an incrementally growing set of pseudo-point parameters. 

In Sec.~\ref{sec:sparseGPreg}, we demonstrated a duality between the joint Gaussian distribution over the pseudo points and the posterior GP induced by these pseudo points. Specifically, if the joint Gaussian distribution of the pseudo points in \eqref{eq:pseudopoint_distribution} or \eqref{eq:pseudopoint_distribution_information} is available, then we can calculate the mean and covariance functions the GP in \eqref{eq:pseudopoint_gp_posterior_information}. This observation suggests that it is sufficient to keep track of the information mean and information matrix of the joint Gaussian distribution of the pseudo points.


Before continuing, we define a few key quantities related to the graph $G$. Specifically, denote as $A \in \reals^{n \times n}$  its adjacency matrix, whose elements $A_{\rb\rbsec}$ may be non-binary. Let $D := \diag(D_{11}, \ldots, D_{nn})$ be the diagonal degree matrix of the graph with elements $D_{\rb\rb} = \sum_{\rbsec\neq \rb} A_{\rb\rbsec}$ and $L:=D-A$ be the graph Laplacian. Define a weight matrix $\weight := I - \nu L$ for $0 \leq \nu \leq \frac{1}{\Delta}$, where $\Delta = \max(D_{11},\ldots, D_{nn})$ is the maximum node degree. The vector of ones, $\mathbf{1} \in \reals^n$, is an eigenvector of $\weight$ since $L \mathbf{1} = \bb0$. Also, $\weight$ is a row-stochastic nonnegative and primitive matrix \cite{consensus} and, hence, has a stationary distribution, specified by its left eigenvector $\bfpi$ with $\sum_{\rb=1}^n \pi_{\rb} = 1$. This Perron weight matrix construction is common in consensus and distributed gradient descent algorithms \cite{consensus,random-network-consensus,distributed-subgradient}.

To gain intuition about the construction of consensus schemes over GP posteriors, we first review distributed Kalman filtering for fixed-dimensional parameter estimation.

\begin{remark}[Directed time-varying graphs]
\label{remark:network}
For simplicity, we consider an undirected static graph $G$ with a fixed weight matrix $\weight$. Relying on consensus results for switching networks \cite{consensus,moreau-condition,time-varying-consensus}, our results may be generalized to directed and time-varying graphs assuming that the graph sequence is uniformly strongly connected, i.e., there exists an integer $T > 0$ such that the union of the edges over any time interval of length $T$ is strongly connected.
\end{remark}

\subsection{Distributed Kalman Filtering}\label{NormalUpdate}

Suppose that the robots aim to estimate a fixed (finite-dimensional) vector $\bff$ cooperatively using local observations $\bfy_t^\rb$, generated according to a linear Gaussian model:
\begin{equation}
\label{eq:linear-Gaussian-om}
\bfy^\rb_t = H^\rb \bff + \bfeta^\rb_t, \qquad \bfeta^\rb_t \sim \mathcal{N}(0, V^\rb).
\end{equation}
Assume that the observations $\bfy^\rb_t$ received by robot $\rb$ are independent over time and from the observations of all other robots. Assume also that the graph $G$ is connected and that $\bff$ is observable if one has access to the observations received by all robots, i.e., the matrix $\begin{bmatrix} H^1 & \cdots & H^n \end{bmatrix}$ has rank equal to the dimension of $\bff$. Since individual observations $\bfy^\rb_t$ alone may be insufficient to estimate $\bff$, the robots need to exchange information. We suppose that each robot starts with a prior probability density function $p^\rb_0(\bff)$ over the unknown vector $\bff$ and updates it over time, relying on its local observations $\bfy^\rb_t$ as well as communication with one-hop neighbors in $G$.

Rahnama Rad and Tahbaz-Saleh~\cite{rad2010distributed} developed a consistent distributed estimation algorithm, in which each agent $\rb$ uses standard Bayesian updates with its local observations $\bfy^\rb_{t+1}$ but, instead of its own prior $p^\rb_t$, each agent uses a weighted geometric average of its neighbors' priors:
\begin{equation}
\label{eq:distributed-bayesian-filter}
p^\rb_{t+1}(\bff) \propto p^\rb(\bfy^\rb_{t+1} | \bff) \prod_{\rb=1}^n (p^\rb_t(\bff))^{\weight_{\rb\rbsec}},
\end{equation}
where $p^\rb(\bfy^\rb_{t+1} | \bff)$ is an observation model, such as \eqref{eq:linear-Gaussian-om}, that should satisfy certain regularity conditions \cite{rad2010distributed}. Atanasov et al.~\cite{atanasov2014joint} showed that if the prior distributions $p_0^\rb$ are Gaussian and the observation models are linear Gaussian as in \eqref{eq:linear-Gaussian-om}, the resulting distributed Kalman filter is mean-square consistent (the estimates $\arg\max_{\bff} p^\rb_{t}(\bff)$ of all agents $\rb$ converge in mean square to the true $\bff$). Specifically, if the priors are $\bff \sim \mathcal{N}(\bfmu^\rb_0,\Sigma^\rb_0)$ with information matrix $\Omega^\rb_0 := (\Sigma^\rb_0)^{-1}$ and information mean $\bfomega^\rb_0 := \Omega^\rb_0 \bfmu^\rb_0$, the Gaussian version of the distributed estimator in \eqref{eq:distributed-bayesian-filter} is:
\begin{equation}
\label{eq:normalUpdate}
\begin{aligned}
\bfomega_{t+1}^\rb &= \sum_{\rb=1}^n\weight_{\rb\rbsec} \bfomega_{t}^\rbsec + H^{\rb\top} {V^\rb}^{-1}\bfy^\rb_{t+1} \\
\Omega_{t+1}^\rb &= \sum_{\rbsec=1}^n\weight_{\rb\rbsec} \Omega_{t}^\rbsec + H^{\rb\top} {V^\rb}^{-1} {H^\rb}
\end{aligned}
\end{equation}
because geometric averaging and Bayesian updates with Gaussian densities lead to a Gaussian posterior density \cite{atanasov2014joint}. The relationship between geometric means being used for belief propagation in \eqref{eq:distributed-bayesian-filter} and weighted averaging via mixing matrix $W$ forms the conceptual basis for message passing in the more general GP posterior inference setting which we detail next.

\subsection{Distributed Incremental Sparse GP Regression}

The distributed estimation algorithm in~\eqref{eq:normalUpdate} does not directly apply to GP regression because the estimation target $f(\bfx)$ is infinite-dimensional. However, the sparse GP regression, described in Sec.~\ref{sec:GP-regression}, relies on a finite (albeit incrementally growing) set of pseudo points $\calP_t$, and we show that it is possible to obtain distributed incremental sparse GP regression based on \eqref{eq:normalUpdate}. As discussed in the beginning of this section, each robot $\rb$ maintains parameters $\Theta_t^\rb := \crl{\calP_t^\rb, \bfm_t^\rb, \bfzeta_t^\rb}$ based on its local observations $\tilde{\calD}_{t}^\rb = (\tilde{\calX}_t^\rb, \tilde{\bfy}_t^\rb)$. Our key idea is to perform weighted geometric averaging over local posteriors, which translates to simple weighted averaging of the means and covariances in \eqref{eq:pseudopoint_distribution} of $f$ at a finite set of pseudo points $\calQ \supseteq \calP_t^\rb$, which will be specified precisely below. The parameters $\Theta_t^\rb$ maintained by robot $\rb$ induce a GP distribution over $f$ in \eqref{eq:compressed_gp_posterior}, which in turn provides a Gaussian probability density function $p_t^\rb(\bff) := p(\bff | \Theta_t^\rb)$ over the (finite-dimensional) vector $\bff := f(\calQ)$ with mean and covariance, obtained from \eqref{eq:compressed_gp_posterior}:
\begin{equation}
\label{eq:meanForRobotsPs}
\begin{aligned}
\mu_t^\rb(\calQ) &:= \mu_0^\rb(\calQ) + k_0^\rb(\calQ,\calP_t^\rb)Z_t^\rb\prl{\bfzeta_t^\rb-\mu_0^\rb(\calP_t^\rb)}, \\
\Sigma_t^\rb(\calQ) &:=  k_0^\rb(\calQ,\calQ) - k_0^\rb(\calQ,\calP_t^\rb)Z_t^\rb k_0^\rb(\calP_t^\rb,\calQ),
\end{aligned}
\end{equation}
where $Z_t^\rb = (k_0^\rb(\calP_t^\rb,\calP_t^\rb) + \sigma^2\diag(\bfm_t^\rb)^{-1})^{-1}$. In order to derive decentralized updates for GPs akin to \eqref{eq:normalUpdate}, we first present the iterative updates associated with robots' local posteriors in terms of their information mean and information matrix corresponding to the mean and covariance of $p_t^\rb(\bff)$.

\begin{lemma}
The information mean $\omega^\rb_t(\calQ):= \Omega^\rb_t(\calQ) \bfmu_t^\rb(\calQ)$ and information matrix $\Omega^\rb_t(\calQ) := (\Sigma^\rb_t(\calQ))^{-1}$ of the Gaussian probability density function $p_t^\rb(\bff) := p(\bff | \Theta_t^\rb)$ of $\bff := f(\calQ)$ with parameters $\Theta_t^\rb := \crl{\calP_t^\rb, \bfm_t^\rb, \bfzeta_t^\rb}$ and mean and covariance in \eqref{eq:meanForRobotsPs} are:
\begin{equation}
\label{eq:infoMeanCovPs}
\begin{aligned}
\omega^\rb_t(\calQ) &= k_0^\rb(\calQ,\calQ)^{-1}\mu_0^\rb(\calQ) + \sigma^{-2}\diag(m_t^\rb(\calQ))\zeta_t^\rb(\calQ)\\
\Omega^\rb_t(\calQ) &= k_0^\rb(\calQ,\calQ)^{-1} + \sigma^{-2}\diag(m_t^\rb(\calQ)),
\end{aligned}
\end{equation}
where, similar to Sec.~\ref{sec:incGPreg}, $m_t^\rb(\bfp)$ and $\zeta_t^\rb(\bfp)$ denote the number of observations and average observation, respectively, for $\bfp \in \calP_t^\rb$ and their domains have been extended to $\calQ \supseteq \calP_t^\rb$ by defining $m_t^\rb(\bfq) = \zeta_t^\rb(\bfq) = 0$ for $\bfq \in \calQ \setminus \calP_t^\rb$.
\end{lemma}
\begin{proof}
Similar to the proof of Prop.~\ref{prop:gp_compression}, let $E$ be a binary matrix such that $k_0^\rb(\calP_t^\rb,\bfx) = Ek_0^\rb(\calQ,\bfx)$, i.e., $E$ selects the points from the superset $\calQ$ which correspond to $\calP_t^\rb$. Note that $k_0^\rb(\calQ,\calP_t^\rb) = k_0^\rb(\calQ,\calQ)E^\top$, $k_0^\rb(\calP_t^\rb,\calQ) = Ek_0^\rb(\calQ,\calQ)$, and $k_0^\rb(\calP_t^\rb,\calP_t^\rb) = Ek_0^\rb(\calQ,\calQ)E^\top$. The expression for $\Omega^\rb_t(\calQ)$ follows from the matrix inversion lemma applied to the covariance matrix in \eqref{eq:meanForRobotsPs} and noting that $E^\top \diag(\bfm_t^\rb) E = \diag(m_t^\rb(\calQ))$. Then, note that:
\begin{align}
\Omega^\rb_t&(\calQ)k_0^\rb(\calQ,\calP_t^\rb)Z_t^\rb \notag\\
&= \prl{I + \sigma^{-2}E^\top \diag(\bfm_t^\rb) Ek_0^\rb(\calQ,\calQ)}E^\top Z_t^\rb\\
&= \sigma^{-2}E^\top\diag(\bfm_t^\rb) (Z_t^\rb)^{-1}Z_t^\rb = \sigma^{-2}E^\top\diag(\bfm_t^\rb).\notag
\end{align}
Thus, the information mean is:
\begin{align}
\omega^\rb_t(\calQ) &= \Omega^\rb_t(\calQ) \prl{\mu_0^\rb(\calQ) + k_0^\rb(\calQ,\calP_t^\rb)Z_t^\rb\prl{\bfzeta_t^\rb-\mu_0^\rb(\calP_t^\rb)}} \notag\\
&= \Omega^\rb_t(\calQ)\mu_0^\rb(\calQ) + \sigma^{-2}E^\top\diag(\bfm_t^\rb)\prl{\bfzeta_t^\rb-\mu_0^\rb(\calP_t^\rb)} \notag\\
&= k_0^\rb(\calQ,\calQ)^{-1}\mu_0^\rb(\calQ) + \sigma^{-2}E^\top\diag(\bfm_t^\rb)\bfzeta_t^\rb\\
&= k_0^\rb(\calQ,\calQ)^{-1}\mu_0^\rb(\calQ) + \sigma^{-2}\diag(m_t^\rb(\calQ))\zeta_t^\rb(\calQ) \tag*{\qedhere}
\end{align}
\end{proof}

With the expression for the parametric updates associated with the posterior inference defined by observations acquired locally at robot $i$ only, we next detail how to augment this update with neighboring robots' information.

\subsubsection{Distributed updates with a fixed pseudo-point set}

To begin, suppose that the pseudo-point sets are fixed across all robots, i.e., $\calP \equiv \calP_t^\rb$, and the local observations $\tilde{\calD}_{t+1}^\rb = (\tilde{\calX}_{t+1}^\rb, \tilde{\bfy}_{t+1}^\rb)$ satisfy $\tilde{\calX}_{t+1}^\rb \subseteq \calP$ for all $t$, $i$. Then, the information means and matrices in \eqref{eq:infoMeanCovPs} have equal dimensions across the robots, and we can apply the update in \eqref{eq:normalUpdate} directly:
\begin{equation}
\label{eq:fixed-pseudopoint-update}
\begin{aligned}
\bfomega^\rb_{t+1} &= \sum_{\rb=1}^n\weight_{\rb\rbsec} \bfomega_{t}^\rbsec + H_{t+1}^{\rb\top}(\sigma^2I)^{-1} \tilde{\bfy}_{t+1}^\rb,\\
\Omega^\rb_{t+1} &= \sum_{\rbsec=1}^n\weight_{\rb\rbsec} \Omega_{t}^\rbsec + H_{t+1}^{\rb\top}(\sigma^2I)^{-1} H_{t+1}^\rb,\\
\end{aligned}
\end{equation}
where $H_{t+1}^\rb := k_0^\rb(\tilde{\calX}^{\rb}_{t+1},\calP)k_0^\rb(\calP,\calP)^{-1}$, $\bfomega_t^\rb := \omega_t^\rb(\calP)$, and $\Omega_t^\rb := \Omega_t^\rb(\calP)$. The information means and matrices have a simple structure, and, similar to \eqref{eq:update}, it is sufficient to track only the number of observations $\bfm_t^\rb$ and the average observations $\bfzeta_t^\rb$ over time:
\begin{gather}
\label{eq:fixed-pseudopoint-update2}
\scaleMathLine{\begin{aligned} 
\bfomega^\rb_{t+1} &= \sum_{\rb=1}^n\weight_{\rb\rbsec}\bfomega_0^\rbsec + \frac{1}{\sigma^{2}} \sum_{\rb=1}^n\weight_{\rb\rbsec}\diag(\bfm_t^\rbsec)\bfzeta_t^\rbsec + \frac{1}{\sigma^{2}}\diag(\tilde{\bfm}_{t+1}^\rb)\tilde{\bfzeta}_{t+1}^\rb\\
\Omega^\rb_{t+1} &= \sum_{\rbsec=1}^n\weight_{\rb\rbsec}\Omega^\rbsec_{0} + \frac{1}{\sigma^{2}}\sum_{\rb=1}^n\weight_{\rb\rbsec} \diag(\bfm_t^\rbsec) + \frac{1}{\sigma^{2}}\diag(\tilde{\bfm}_{t+1}^\rb),
\end{aligned}}
\raisetag{4.5ex}
\end{gather}
where $\tilde{\bfm}_{t+1}^\rb$ and $\tilde{\bfzeta}_{t+1}^\rb$ are the number of new observations and new observation averages received by robot $\rb$ of the pseudo points $\calP$ at time $t+1$. We consider the case with incrementally growing pseudo-point sets that are potentially different across the robots before presenting the final distributed update equations for $\bfm_t^\rb$ and $\bfzeta_t^\rb$. This is the focus of the following subsection.

\subsubsection{Distributed updates with dynamic pseudo-point sets}
Consider the general case where each robot maintains its own pseudo-point set $\calP_t^\rb$ and the observations $\tilde{\calD}_{t+1}^\rb = (\tilde{\calX}_{t+1}^\rb, \tilde{\bfy}_{t+1}^\rb)$ may introduce new pseudo-points $\tilde{\calP}_{t+1}^\rb \not\subseteq \calP_t^\rb$. Our key observation is that the parameters $\Theta_t^\rb$ induce a GP distribution over the whole function $f$ and, hence, can be used to obtain a Gaussian distribution over a pseudo-point set that is larger than $\calP_t^\rb$ according to \eqref{eq:meanForRobotsPs} and \eqref{eq:infoMeanCovPs}. Note that the structure of the information mean and information matrix in \eqref{eq:infoMeanCovPs} remains unchanged when the set of pseudo-points $\calQ$ changes. To increment the pseudo-point set of robot $\rb$ at time $t$, we aggregate the pseudo points $\calP_t^\rbsec$ of its neighbors and the newly observed pseudo points $\tilde{\calP}_{t+1}^\rb$ as follows:
\begin{equation}
\label{eq:dynamic-pseudopoint-update1}
\calP_{t+1}^\rb = \bigcup\limits_{\rbsec \in \mathcal{N}_\rb\cup\{\rb\}} \cP[\rbsec]{t}\cup\P[\rb]{t+1}
\end{equation}
Then, the distributed averaging in \eqref{eq:normalUpdate} can be performed over the information means and information matrices in \eqref{eq:infoMeanCovPs} with $\calQ = \calP_{t+1}^\rb$ and $H_{t+1}^\rb := k_0^\rb(\tilde{\calX}^{\rb}_{t+1},\calP_{t+1}^\rb)k_0^\rb(\calP_{t+1}^\rb,\calP_{t+1}^\rb)^{-1}$:
\begin{equation}
\label{eq:dynamic-pseudopoint-update2}
\scaleMathLine[0.89]{\begin{aligned}
\omega^\rb_{t+1}(\calP_{t+1}^\rb) &= \sum_{\rb=1}^n\weight_{\rb\rbsec} \omega^\rbsec_t(\calP_{t+1}^\rb) + H_{t+1}^{\rb\top}(\sigma^2I)^{-1} \tilde{\bfy}_{t+1}^\rb,\\
\Omega^\rb_{t+1}(\calP_{t+1}^\rb) &= \sum_{\rb=1}^n\weight_{\rb\rbsec} \Omega^\rbsec_t(\calP_{t+1}^\rb) + H_{t+1}^{\rb\top}(\sigma^2I)^{-1} H_{t+1}^\rb.
\end{aligned}}
\end{equation}
We may rewrite the preceding expressions in terms of the number of observations $m^\rb_{t+1}(\bfp)$ and average observations $\zeta^\rb_{t+1}(\bfp)$ for any $\bfp \in \calP_{t+1}^\rb$, akin to \eqref{eq:update}, by following the steps in \eqref{eq:fixed-pseudopoint-update2} for the dynamic pseudo-point case, leading to:
\begin{align}\label{eq:updaterule}
m^\rb_{t+1}(\bfp) &= \sum_{ \rbsec \in \mathcal{N}_\rb\cup\{\rb\}} \weight_{\rb\rbsec} m^\rbsec_t(\bfp) + \tilde{m}^\rb_{t+1}(\bfp),\\
\zeta^\rb_{t+1}(\bfp) &= \frac{\sum_{ \rbsec \in \mathcal{N}_\rb\cup\{\rb\}} \weight_{\rb\rbsec} m^\rbsec_t(\bfp)\zeta^\rbsec_t(\bfp) + \tilde{m}^\rb_{t+1}(\bfp)\tilde{\zeta}^\rb_{t+1}(\bfp)}{m_{t+1}^\rb(\bfp)}. \notag
\end{align}

With the updates for robot $i$ in terms of its local observations and message passing with its neighbors $\mathcal{N}_i$ specified, we shift in the following subsection to establishing its statistical properties.

\subsection{Theoretical Guarantee for Consistent Estimation}
\label{sec:centralized-estimator}

We show that the proposed distributed incremental sparse GP regression defined by \eqref{eq:dynamic-pseudopoint-update1}, \eqref{eq:updaterule}, and \eqref{eq:compressed_gp_posterior} converges to a centralized sparse GP regression, which uses the observation data $\cup_t \cup_\rb \tilde{\calD}_{t}^\rb$ from all robots. At each time step $t$, the centralized estimator receives data $\cup_\rb \tilde{\calD}_{t}^\rb$, and, as discussed in Sec.~\ref{sec:incGPreg}, updates a global set of pseudo points $\calP_t^{ctr}$, the number of times $m_t^{ctr}(\bfp)$ each pseudo point $\bfp \in \calP_t^{ctr}$ has been observed, and the average observation $\zeta_t^{ctr}(\bfp)$ of $\bfp \in \calP_t^{ctr}$. In order to show that the GP maintained by each robot $\rb$ eventually agrees with the centralized GP, the centralized estimator should also be affected by the Perron weight matrix $\weight$. If $\weight = \frac{1}{n}\boldsymbol{11}^\top$, the information provided by different robots is equally credible and the centralized estimator can use the combined set of observations $\cup_\rb \tilde{\calD}_{t}^\rb$ directly. If, however, the left eigenvector $\bfpi$ of $\weight$ is not $\boldsymbol{1}$, then its elements $\pi_\rb$ specify different credibility for the different robots. More precisely, the centralized estimator should treat the measurements $\tilde{\calD}_{t}^\rb$ of robot $\rb$ as if they were generated with noise variance $\sigma^2/\pi_i$, instead of the true noise variance $\sigma^2$. This is equivalent to scaling the number of observations $\tilde{m}_t^\rb$ provided by robot $\rb$ by its ``credibility'' $\pi_i$, leading to the following update for the centralized sparse GP regression parameters:
\begin{align} \label{eq:central}
  \calP_{t+1}^{ctr} &= \cup_{\rb=1}^n \tilde{\calP}_{t+1}^{\rb} \cup \calP_t^{ctr},\notag\\ 
	m_{t+1}^{ctr}(\bfp) &= m_{t}^{ctr}(\bfp) + \sum_{\rb=1}^n \pi_\rb \tilde{m}_{t+1}^{\rb}(\bfp),\\
  \zeta_{t+1}^{ctr}(\bfp) &= \frac{m_{t}^{ctr}(\bfp)\zeta_t^{ctr}(\bfp)+\sum_{\rb=1}^n \pi_\rb \tilde{m}_{t+1}^{\rb}(\bfp)\tilde{\zeta}_{t+1}^\rb(\bfp)}{m_{t+1}^{ctr}(\bfp)},\notag
\end{align}
for all $\bfp \in \calP_{t+1}^{ctr}$. The next result shows that the individual GP distributions maintained by each robot using the distributed updates in \eqref{eq:updaterule} converge to the centralized GP distribution determined by the parameters above.

\begin{proposition}
\label{th:Net1}
Let $\tilde{\calD}_{t}^\rb = (\tilde{\calX}_t^\rb, \tilde{\bfy}_t^\rb)$ be the data received by robot $\rb$ at time $t$, associated with pseudo points $\tilde{\calP}_t^\rb \subset \calP_{\#}$ and number of observations $\tilde{m}_t^\rb(\bfp)$ and average observation $\tilde{\zeta}_t^\rb(\bfp)$ for $\bfp \in \calP_{\#}$. If the data streaming stops at some time $T < \infty$, then as $t \to \infty$, the distributions $\calG\calP(\mu_t^\rb(\bfx), k_t^\rb(\bfx,\bfx'))$ maintained by each robot $\rb$, specified according to \eqref{eq:compressed_gp_posterior} with parameters $\calP_t^\rb$, $m_t^\rb(\bfp)$, $\zeta_t^\rb(\bfp)$ in \eqref{eq:dynamic-pseudopoint-update1} and \eqref{eq:updaterule} converge to the distribution $\calG\calP(\mu_t^{ctr}(\bfx), k_t^{ctr}(\bfx,\bfx'))$ of the centralized estimator with parameters $\calP_t^{ctr}$, $m_t^{ctr}(\bfp)$, $\zeta_t^{ctr}(\bfp)$ in \eqref{eq:central}, i.e., $|\mu_t^\rb(\bfx) - \mu^{ctr}_t(\bfx)| \to 0$ and $|k_t^\rb(\bfx,\bfx') - k^{ctr}_t(\bfx,\bfx')| \to 0$ almost surely for all $i \in \calV$, $\bfx,\bfx'$.
\end{proposition}

\begin{proof}
Since the distributions $\calG\calP(\mu_t^\rb(\bfx), k_t^\rb(\bfx,\bfx'))$ and $\calG\calP(\mu_t^{ctr}(\bfx), k_t^{ctr}(\bfx,\bfx'))$ are completely determined by the parameters $\calP_t^\rb$, $m_t^\rb(\bfp)$, $\zeta_t^\rb(\bfp)$ and $\calP_t^{ctr}$, $m_t^{ctr}(\bfp)$, $\zeta_t^{ctr}(\bfp)$, respectively, it is sufficient to show that $|m_t^\rb(\bfp) - m^{ctr}_t(\bfp)| \to 0$ and $|\zeta_t^\rb(\bfp) - \zeta^{ctr}_t(\bfp)| \to 0$ for all $i \in \calV$, $\bfp \in \calP_{\#}$. Let $\bfp \in \calP_{\#}$ be arbitrary and note that $m_0^\rb(\bfp) = m^{ctr}_0(\bfp) = 0$ and $\zeta_0^\rb(\bfp) =  \zeta^{ctr}_0(\bfp) = 0$ since no pseudo points have been observed initially. Expand \eqref{eq:central} recursively to obtain $m^{ctr}_t(\bfp)$ and $\zeta^{ctr}_t(\bfp)$ in terms of the observation statistics:
\begin{equation}
\label{eq:central-expansion}
\begin{aligned}
m_{t}^{ctr}(\bfp) &= \sum_{\tau=0}^{t}\sum_{\rb=1}^n \pi_\rb \tilde{m}_\tau^\rb(\bfp),\\
\zeta_{t}^{ctr}(\bfp) &= \frac{1}{m_t^{ctr}(\bfp)}\sum_{\tau=0}^{t}\sum_{\rb=1}^n \pi_\rb \tilde{m}_{\tau}^\rb(\bfp) \tilde{\zeta}_\tau^\rb(\bfp).
\end{aligned}
\end{equation}
Similarly, expand \eqref{eq:updaterule} to obtain $m^{\rb}_t(\bfp)$ and $\zeta^{\rb}_t(\bfp)$ in terms of the observation statistics:
\begin{equation}
\label{eq:echo-k-expansion}
\begin{aligned}
m_{t}^{\rb}(\bfp) &= \sum_{\tau=0}^{t}\sum_{\rbsec=1}^n \brl{\weight^{t-\tau}}_{\rb\rbsec} \tilde{m}_\tau^\rbsec(\bfp),\\
\zeta_{t}^{\rb}(\bfp) &= \frac{1}{m_t^{\rb}(\bfp)}\sum_{\tau=0}^{t}\sum_{\rbsec=1}^n \brl{\weight^{t-\tau}}_{\rb\rbsec}\tilde{m}_{\tau}^\rbsec(\bfp) \tilde{\zeta}_\tau^\rbsec(\bfp),
\end{aligned}
\end{equation}
where the weights $\brl{\weight^{t-\tau}}_{\rb\rbsec}$ appear since the data $\tilde{m}_\tau^\rbsec(\bfp)$ and $\tilde{\zeta}_\tau^\rbsec(\bfp)$ propagate through the network with weight matrix $\weight$ and reach robot $\rb$ via all paths of length $t-\tau$. Alternatively, \eqref{eq:echo-k-expansion} can be viewed as the solution of the discrete-time linear time-invariant system in \eqref{eq:updaterule} with transition matrix $\Phi(t,\tau) = \weight^{t-\tau}$, $t \geq \tau$. Since the data collection stops at some finite time $T$, $\tilde{m}_t^\rb(\bfp) = \tilde{\zeta}_t^\rb(\bfp) = 0$ for all $t > T$, $\rb \in \calV$. The convergence of \eqref{eq:echo-k-expansion} to \eqref{eq:central-expansion} is concluded from the fact that $\brl{\weight^{t}}_{\rb\rbsec} \to \pi_\rbsec > 0$ since $\weight$ is a row-stochastic nonnegative and primitive matrix.
\end{proof}


Prop.~\ref{th:Net1} is a similar result to \cite[Thm.~3]{rad2010distributed}, where it is shown that, if the weight matrix $\weight$ is doubly stochastic, a distributed parameter estimator is as efficient as any centralized parameter estimator. However, Prop.~\ref{th:Net1} applies to distributed function estimation using an incrementally growing set of parameters and re-weights the observations used by the centralized estimator via the stationary distribution $\bfpi$ of $\weight$ to ensure convergence even when $\weight$ is not doubly stochastic.

\subsection{Echoless Distributed GP Regression}

The distributed pseudo point update we derived in \eqref{eq:updaterule} is not efficient for two reasons. First, convergence to the central GP estimate is guaranteed only in the limit, as $t \to \infty$ (Prop.~\ref{th:Net1}). Second, every time robots exchange messages, all information they have must be sent. This is inefficient as may be seen in the proof of Prop.~\ref{th:Net1}, the observations are exchanged an infinite number of times (echos in the network). To address these limitations, we label the communication messages with the list of robots that have already received them and show that convergence to the centralized estimate can, in fact, be achieved in finite time.

\begin{figure}[t]
\includegraphics[width=\columnwidth]{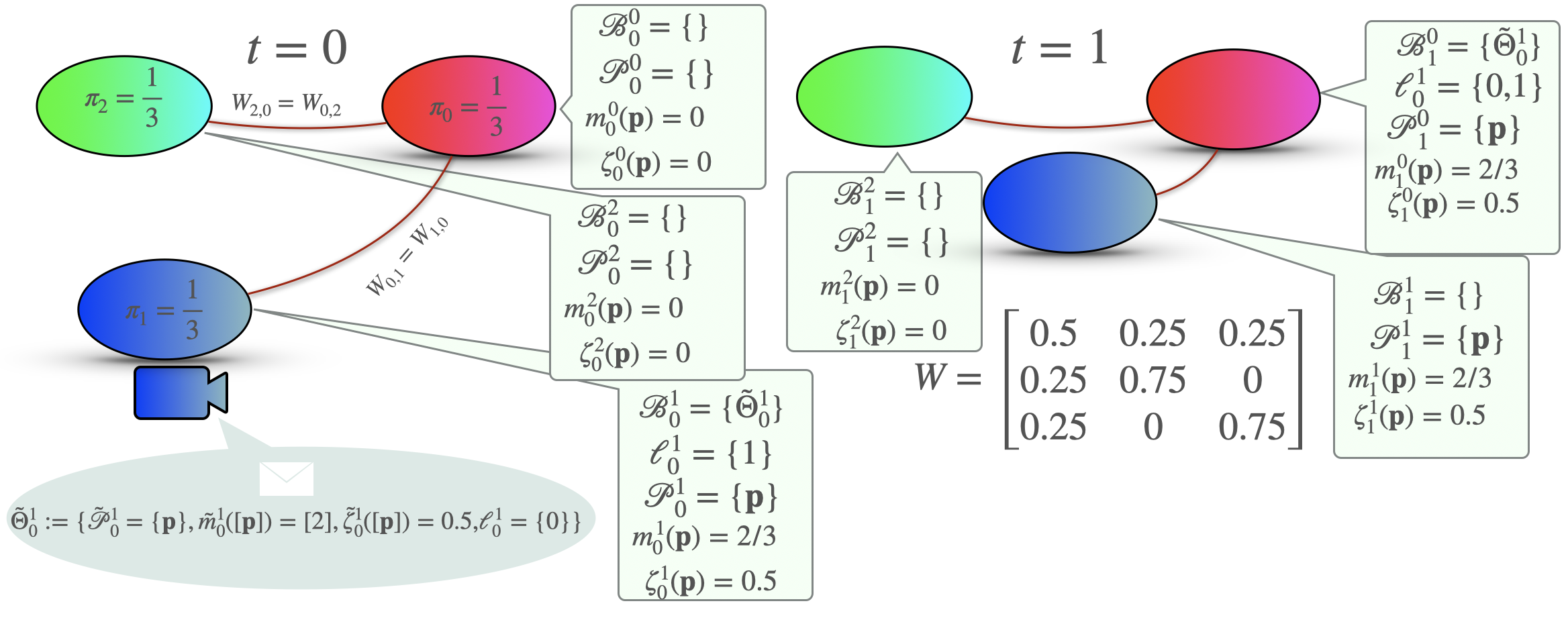}
\caption{Visualization of the distributed GP parameter update in \eqref{eq:final_update_rule} in a network with three nodes (red, green, blue). The node parameters are shown at time $t=0$ and $t=1$. Only a single observation with one pseudo point is received by node $1$ (blue) at time $t=0$ and is propagated to node $0$ (red) at time $t=1$.}
\label{fig:batchUpdateVis}
\end{figure}

Let $\tilde{\Theta}_t^\rb := \{ \tilde{\calP}_t^\rb, \tilde{m}_t^\rb(\tilde{\calP}_t^\rb), \tilde{\zeta}_t^\rb(\tilde{\calP}_t^\rb), \ell_t^\rb \}$ define a mini-batch of observations for robot $\rb$. At time $t$, $\tilde{\Theta}_t^\rb$ contains the new observations $\tilde{\calP}_t^\rb$, $\tilde{m}_t^\rb(\tilde{\calP}_t^\rb)$, $\tilde{\zeta}_t^\rb(\tilde{\calP}_t^\rb)$ of robot $\rb$ as well as a list of robots $\ell_t^\rb$ that have already received this mini-batch. The list $\ell_t^\rb$ is initialized by $\crl{\rb}$. Additionally, for each robot $\rb$, we define a set of mini-batches $\calB_{t+1}^\rb$ that the robot should use at time $t$ to update its GP parameters. The mini-batch set $\calB_{t}^\rb$ from the previous time step contains old mini-batches that robot $\rb$ should transmit to its neighbors. Inspired by the similarity of \eqref{eq:central-expansion} and \eqref{eq:echo-k-expansion}, we propose a distributed protocol which ensures:
\begin{itemize}[leftmargin=*]
  \item each mini-batch visits each robot once rather than echoing in the network, relying on $\ell_t^\rb$ to keep track of visited robots,
  \item convergence to the centralized GP distribution is achieved in finite and minimum time by picking the stationary distribution $\bfpi$ of $\weight$ as the coefficient in \eqref{eq:echo-k-expansion}.
\end{itemize}
The distributed update of the parameters of robot $\rb$ at time step $t$ is:
\begin{equation}
\label{eq:final_update_rule}
\begin{aligned}
\calB_{t+1}^\rb &= \bigcup_{\tilde{\Theta}_\tau^\rbsec \in\calB_t^r, r\in \mathcal{N}_\rb, \rb \notin \ell_\tau^\rbsec} \tilde{\Theta}_\tau^\rbsec \cup \tilde{\Theta}_{t+1}^\rb\\
\ell_\tau^\rbsec &= \ell_\tau^\rbsec \cup \{\rb\} \text{ for all } \tilde{\Theta}_\tau^\rbsec \in\calB_{t+1}^\rb\\
\calP_{t+1}^\rb &= \bigcup_{\tilde{\Theta}_\tau^\rbsec \in \calB_{t+1}^\rb} \calP_{\tau}^\rbsec \cup\calP_t^\rb\\
m_{t+1}^{\rb}(\bfp) &= m^\rb_{t}(\bfp) + \sum_{\tilde{\Theta}_\tau^\rbsec\in \calB_{t+1}^\rb} \pi_{\rbsec} \tilde{m}^\rbsec_\tau(\bfp)\\
\zeta^\rb_{t+1}(\bfp) &= \frac{m^\rb_{t}(\bfp) \zeta^\rb_{t}(\bfp) + \sum_{\tilde{\Theta}_\tau^\rbsec \in \calB_{t+1}^\rb} \pi_{\rbsec} \tilde{m}^\rbsec_\tau(\bfp) \tilde{\zeta}^\rbsec_{\tau}(\bfp)}{m^\rb_{t+1}(\bfp)}
\end{aligned}
\end{equation}
%

We prove below that this distributed update rule converges in finite time to the centralized GP distribution. Compared with \eqref{eq:updaterule}, the distributed update in \eqref{eq:final_update_rule} is able to achieve finite-time convergence because it uses the weights $\bfpi$ from the stationary distribution of $\weight$ right away, instead of processing the same information an infinite number of times to determine the final weights. Moreover, \eqref{eq:updaterule} stipulates that two robots should exchange all of their information at each time step, which is very inefficient in practice. The mini-batch messages in \eqref{eq:final_update_rule} allow the robots to exchange only the latest information and guarantee that each observation reaches each agent once. A visualization of this method is shown in Fig.~\ref{fig:batchUpdateVis}. 


\begin{proposition}
\label{prop:cnvrgn}
Let $\tilde{\calD}_{t}^\rb = (\tilde{\calX}_t^\rb, \tilde{\bfy}_t^\rb)$ be the data received by robot $\rb$ at time $t$, associated with pseudo points $\tilde{\calP}_t^\rb \subset \calP_{\#}$ and number of observations $\tilde{m}_t^\rb(\bfp)$ and average observation $\tilde{\zeta}_t^\rb(\bfp)$ for $\bfp \in \calP_{\#}$. If the data streaming stops at some time $T < \infty$, then at time $t = T + n - 1$, the distributions $\calG\calP(\mu_t^\rb(\bfx), k_t^\rb(\bfx,\bfx'))$ maintained by each robot $\rb$, specified according to \eqref{eq:compressed_gp_posterior} with parameters in \eqref{eq:final_update_rule} are exactly equal to the distribution $\calG\calP(\mu_t^{ctr}(\bfx), k_t^{ctr}(\bfx,\bfx'))$ of the centralized estimator with parameters in \eqref{eq:central}, i.e., $\mu_t^\rb(\bfx) = \mu^{ctr}_t(\bfx)$ and $k_t^\rb(\bfx,\bfx') = k^{ctr}_t(\bfx,\bfx')$ almost surely for all $i \in \calV$, $\bfx,\bfx'$.
\end{proposition}

\begin{proof}
As in the proof of Prop.~\ref{th:Net1}, it is sufficient to show that at $t = T + n - 1$, $m_t^\rb(\bfp) = m^{ctr}_t(\bfp)$ and $\zeta_t^\rb(\bfp) = \zeta^{ctr}_t(\bfp)$ for all $i \in \calV$, $\bfp \in \calP_{\#}$. As before, we express $m_t^\rb(\bfp)$ and $\zeta_t^\rb(\bfp)$ in terms of $\tilde{m}_\tau^\rbsec(\bfp)$ and $\tilde{\zeta}_\tau^\rbsec(\bfp)$ for arbitrary $\bfp \in \calP_{\#}$ and $\tau \leq t$. The key is to realize whether mini-batch $\tilde{\Theta}_\tau^\rbsec$ is received by robot $\rb$. Since the mini-batch exchanges are happening based on the communication graph structure, the elements of $\weight^{t-\tau}$ determine which robots have received a mini-batch released at time $\tau$ by time $t$. Precisely, if $\brl{\weight^{t-\tau}}_{\rb\rbsec} > 0$, then robot $\rb$ has received mini-batch $\tilde{\Theta}_\tau^\rbsec$ by time $t$ and otherwise, if $\brl{\weight^{t-\tau}}_{\rb\rbsec} = 0$, it has not received it. Let $\sign(x)$ denote the sign of a scalar $x$ with $\sign(0) = 0$. Expanding \eqref{eq:final_update_rule} recursively leads to:
\begin{align}
\label{eq:k-expansion}
m_t^\rb(\bfp) &= \sum_{\tau=0}^{t}\sum_{\rb=1}^n \sign(\brl{\weight^{t-\tau}}_{\rb\rbsec}) \pi_\rbsec \tilde{m}_\tau^\rbsec(\bfp)\\
\zeta^\rb_t(\bfp) &= \frac{1}{m^\rb_{t}(\bfp)}\sum_{\tau=0}^{t}\sum_{\rbsec=1}^n \sign(\brl{\weight^{t-\tau}}_{\rb\rbsec})  \pi_\rbsec \tilde{m}_{\tau}^\rbsec(\bfp)\tilde{\zeta}^\rbsec_{\tau}(\bfp)\notag
\end{align}
Since the data collection stops at some finite time $T$, $\tilde{m}_\tau^\rb(\bfp) = \tilde{\zeta}_\tau^\rb(\bfp) = 0$ for all $\tau > T$, $\rb \in \calV$. Comparing \eqref{eq:echo-k-expansion} and \eqref{eq:central-expansion}, equality of $\mu_t^\rb(\bfx)$ and $ \mu^{ctr}_t(\bfx)$ and $k_t^\rb(\bfx,\bfx')$ and $k^{ctr}_t(\bfx,\bfx')$ at $t = T + n - 1$ is concluded by the fact that $\brl{\weight^{n-1}}_{\rb\rbsec} > 0$ because the network is connected.
\end{proof}

\section{Distributed Metric-Semantic Mapping}
\label{sec:DistributedMSM}

We apply the distributed GP regression technique developed in Sec.~\ref{sec:DistributedGP} to the multi-robot metric-semantic TSDF mapping problem. Each robot $\rb$ receives local distance and class observations $\crl{\lambda^\rb_{t+1,\msr},c^\rb_{t+1,\msr}}$, which are transformed using the procedure in Sec.~\ref{sec:training_set} into training data sets $\D[\rb]{t+1,\cls} = \prl{\X[\rb]{t+1,\cls}, \Y[\rb]{t+1,\cls}}$ for estimating the TSDFs $\{f_\cls(\bfx)\}$ of the different object classes. Each dataset $\D[\rb]{t+1,\cls}$ is compressed into a set of pseudo points $\P[\rb]{t+1,\cls}$ with associated number of observations $\tilde{m}_{t+1,\cls}^\rb(\bfp)$ and average observation $\tilde{\zeta}_{t+1,\cls}^\rb(\bfp)$ for $\bfp \in \P[\rb]{t+1,\cls}$. Each robot maintains a separate GP $\mathcal{GP}(\mu^\rb_{t,\cls}(\bbx), k^\rb_{t,\cls}(\bbx, \bbx')$ for each class TSDF $f_\cls(\bfx)$. In the multi-robot case, the GP distributions of robot $\rb$ are updated simultaneously and independently for all classes using the new class-specific observation data $\P[\rb]{t+1,\cls}$, $\tilde{m}_{t+1,\cls}^\rb(\P[\rb]{t+1,\cls})$, $\tilde{\zeta}_{t+1,\cls}^\rb(\P[\rb]{t+1,\cls})$ as well as information from the neighboring robots in the form of class-specific mini-batches $\calB_{t+1,\cls}^\rb$ as described in \eqref{eq:final_update_rule}. To make the GP models scalable to large environments, we organize the pseudo points $\calP_{t,\cls}^\rb$ for each robot $\rb$ and class $\cls$ in an octree data structure, as in Sec.~\ref{sec:octree}, and predict the class of a query point via the method in Sec.~\ref{sec:class}. Prop.~\ref{prop:cnvrgn} guarantees that the local TSDF GPs at each robot converge to a common GP, which is equivalent to the one that would be obtained by centralized sparse GP regression. Moreover, when the streaming of new observations stops, the convergence happens in finite time as soon as each observation is received by each robot exactly once. In other words, there is no unnecessary communication in the form of information exchange echo in the network.

\section{Evaluation using 2-D Simulated Data}
\label{sec:evaluation2d}

In this section, we evaluate our semantic TSDF mapping approach in 2-D simulated environments. We first demonstrate the qualitative and quantitative performance of the single-robot approach of Sec.~\ref{sec:algorithm}. Then, we report results for the multi-robot approach of Sec.~\ref{sec:DistributedMSM} using three robots to map the same environment collaboratively. In all experiments, we employ a sparse Mat{\'e}rn kernel ($\nu = 3/2$) \cite{kim2014recursive}. We choose the workspace discretization $\allPs$ as a grid with resolution $voxel\ size$. Given a training point $\hat{\bfx}$ in \eqref{Gtc}, we choose a square region of pseudo points from $\allPs$ around $\hat{\bfx}$. These pseudo points are used to construct the training set in \eqref{defD} around the sensor hit points $\hat{\bfx}$, instead of a circle with radius $\epsilon$. We call the number of pseudo points on the edge of the square region $frame\ size$, and choose it so that $(frame\ size-1) \times voxel\ size \geq 2 \epsilon$. 



\subsection{Single-Robot 2-D Evaluation}
\label{sec:2d-single}
We generate random 2-D environments (see Fig.~\ref{sdff}) and robot trajectories by sampling poses sequentially and keeping the ones that are in free space. Observations are obtained along the robot trajectories using a simulated distance-class sensor. We apply our incremental sparse GP regression method to obtain a probabilistic TSDF map and compare it with the ground truth TSDF.

\subsubsection{TSDF Accuracy}
One sample environment from our 2-D simulation with the ground truth and reconstructed TSDF and boundaries is shown in Fig.~\ref{sdff}. Our method provides continuous probabilistic TSDF estimates. The choice of $frame\ size$ is very dependent on the desired truncation value for the SDF reconstruction. Larger $frame\ size$ allows estimating larger truncation values but incurs additional computation cost. The precision and resilience to measurement noise of our method are evaluated in Fig.~\ref{var_prob}. The test points are chosen from a grid with resolution $0.5 \times voxel\ size$ within the truncation distance from the ground-truth object boundaries.



\begin{figure}[t]
\centering
\includegraphics[width=\linewidth]{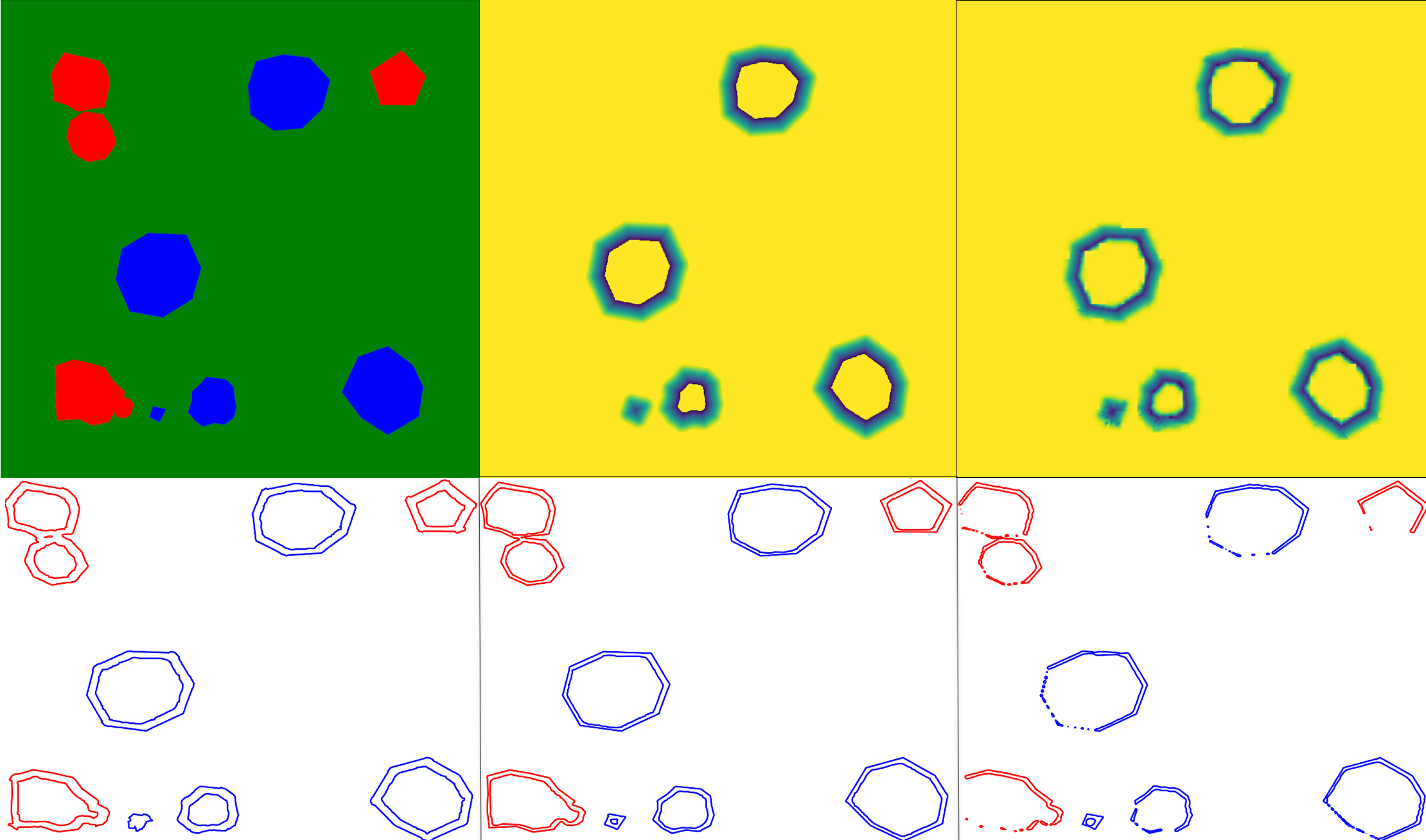}
\caption{Ground-truth 2-D simulated environment (top left) with two object classes (red, blue), ground-truth TSDF for the blue class (top middle), and reconstructed TSDF with $frame\ size=10$ (top right). The reconstructed TSDF boundaries are shown for three different $frame\ size$ parameters on the bottom row: $10$ (bottom left), $3$ (bottom middle), $2$ (bottom right). Sharp edges are captured better with $frame\ size$ $3$ vs. $10$ but using $frame\ size$ less that $3$ caused missing parts at the boundaries.}
\label{sdff}
\end{figure}

%
%
%

\subsubsection{Classification Accuracy}

We evaluate the average precision and recall of our posterior classification over 50 random 2-D maps. In each map, we pick uniformly distributed random points along the obstacle boundaries, and calculate the SDF error and the class-detection accuracy. Since the values are symmetric for binary classification, we present the average precision and recall over the two classes in Fig.~\ref{var_prob}. The figure shows that the misclassification rate, precision, recall, and SDF error are not very sensitive to class error probability. The misclassification rate is the ratio of all to the misclassified test points. The SDF error is the average absolute value difference between the estimated and ground-truth SDF values. We report normalized SDF error: $\frac{\text{SDF error}}{voxel\ size}$. Fig.~\ref{para} investigates the effect of the parameters of our algorithm on misclassification rate, normalized SDF error, False Discovery Rate ($\text{FDR} := 1-\text{Precision}$), and False Negative Rate ($\text{FNR} := 1-\text{Recall}$). We see that the misclassification rate, FNR, and FDR respond similarly to parameter variations.

Increasing the maximum number of pseudo points per octree support region, $max(N)$, improves the (normalized) SDF error. The improvement is significant at first but after a certain octree support region size, even exponential increases in $max(N)$ do not significantly affect the SDF error. The classification measures improve slightly with an initial increase in $max(N)$. Increasing $\delta$ has a similar effect on all the performance measures. Increasing the GP noise variance $\sigma^2$ at first improves all the measures but then it worsens them. An incorrect choice of $\sigma^2$ is critical to the method, but affects the misclassification rate smoothly so, it must be in the right region, but as long as the value of $\sigma^2$ is in the right ballpark, choosing the optimal $\sigma^2$ is not critical.

\begin{figure}[t]
\pgfplotsset{width=0.78\linewidth, compat=default}
\begin{tikzpicture}
%
%
%

\begin{axis}[
    legend style={overlay, at={(1.3, 0.05)},anchor = center, legend columns=1, font=\tiny, nodes={scale=0.8, transform shape}},
    title={},
    xlabel={Class error probability},
    xmin=0, xmax=0.10,
    ymin=0, ymax=1,
    xtick={0,0.02,0.04,0.06,0.08,0.10},
    ytick={},
    ylabel style={overlay,at={(0.91,0.95)}, anchor=center, font=\small},
     xticklabel style={
        /pgf/number format/fixed,
        /pgf/number format/precision=2
     },
    scaled x ticks=false
]
\addplot[
    color = gray,
    mark=square,
    ]
    coordinates {
    (0,100)
    };
\addplot[
    color=gray,
    mark=triangle,
    ]
    coordinates {
    (0,100)
    };
\addplot[
    color=gray,
    mark=asterisk,
    ]
    coordinates {
    (0,100)
    };
    \addplot[
    color=gray,
    mark=diamond,
    ]
    coordinates {
    (0,100)
    };
    \addplot[
    color=gray,
    mark=star,
    ]
    coordinates {
    (0,100)
    };
    \addplot[
    color=gray,
    mark=pentagon*,
    ]
    coordinates {
    (0,100)
    };
        \addplot[
    color=cyan,
    mark=*,
    ]
    coordinates {
    (0,100)
    };
        \addplot[
    color=blue,
    mark=*,
    ]
    coordinates {
    (0,100)
    };
        \addplot[
    color=red,
    mark=*,
    ]
    coordinates {
    (0,100)
    };
        \addplot[
    color=green,
    mark=*,
    ]
    coordinates {
    (0,100)
    };
 \legend{$var=0$, $var=0.2$, $var=0.4$, $var=0.6$, $var=0.8$, $var=1$, Misclassification Rate,Precision,Recall, Normalized SDF Error}
 %
%
%
\addplot[
    color=cyan,
    mark=square,
    ]
    coordinates {
    (0,0.00290444)(0.02,0.00296533)(0.04,0.00326689)(0.06,0.00523603)(0.08,0.00857851)(0.1,0.01675957)
    };
\addplot[
    color=cyan,
    mark=triangle,
    ]
    coordinates {
    (0,0.02109129)(0.02,0.02245829)(0.04,0.02771147)(0.06,0.03550442)(0.08,0.05106535)(0.1,0.08478497)
    };
\addplot[
    color=cyan,
    mark=asterisk,
    ]
    coordinates {
    (0,0.09092003)(0.02,0.09619021)(0.04,0.11180337)(0.06,0.13766458)(0.08,0.17278183)(0.1,0.22254975)
    };
    \addplot[
    color=cyan,
    mark=diamond,
    ]
    coordinates {
    (0,0.23057268)(0.02,0.23362044)(0.04,0.24601887)(0.06,0.27442385)(0.08,0.30439269)(0.1,0.34407737)
    };
    \addplot[
    color=cyan,
    mark=star,
    ]
    coordinates {
    (0,0.33956084)(0.02,0.35032185)(0.04,0.36125512)(0.06,0.36802702)(0.08,0.39191881)(0.1,0.41534371)
    };
    \addplot[
    color=cyan,
    mark=pentagon*,
    ]
    coordinates {
    (0,0.41058296)(0.02,0.41206498)(0.04,0.42558612)(0.06,0.43216579)(0.08,0.44372726)(0.1,0.45405913)
    };
%
%
%
\addplot[
    color=blue,
    mark=square,
    ]
    coordinates {
    (0,0.99681356)(0.02,0.99664179)(0.04,0.99624657)(0.06,0.99442366)(0.08,0.99122988)(0.1,0.98330307)
    };
\addplot[
    color=blue,
    mark=triangle,
    ]
    coordinates {
   (0,0.97803817)(0.02,0.97642469)(0.04,0.97233134)(0.06,0.96440436)(0.08,0.9494801)(0.1,0.91380338)
    };
\addplot[
    color=blue,
    mark=asterisk,
    ]
    coordinates {
    (0,0.90739882)(0.02,0.90260592)(0.04,0.88642562)(0.06,0.86081058)(0.08,0.82666085)(0.1,0.77829603)
    };
    \addplot[
    color=blue,
    mark=diamond,
    ]
    coordinates {
    (0,0.76871456)(0.02,0.76671226)(0.04,0.75298711)(0.06,0.72660024)(0.08,0.69533346)(0.1,0.65661518)
    };
    \addplot[
    color=blue,
    mark=star,
    ]
    coordinates {
    (0,0.66049661)(0.02,0.65092669)(0.04,0.6390206)(0.06,0.63648041)(0.08,0.60835411)(0.1,0.5856967)
    };
    \addplot[
    color=blue,
    mark=pentagon*,
    ]
    coordinates {
    (0,0.58966153)(0.02,0.58881207)(0.04,0.57666344)(0.06,0.56873154)(0.08,0.55626464)(0.1,0.54763366)
    };
\addplot[
    color=red,
    mark=square,
    ]
    coordinates {
    (0,0.99665986)(0.02,0.99668263)(0.04,0.99641182)(0.06,0.99283705)(0.08,0.98769463)(0.1,0.9756313)
    };
\addplot[
    color=red,
    mark=triangle,
    ]
    coordinates {
    (0,0.97071341)(0.02,0.96948961)(0.04,0.96150189)(0.06,0.95039902)(0.08,0.92967042)(0.1,0.88972053)
    };
\addplot[
    color=red,
    mark=asterisk,
    ]
    coordinates {
    (0,0.88350213)(0.02,0.87739239)(0.04,0.8591831)(0.06,0.83117263)(0.08,0.7948299)(0.1,0.74620088)
    };
    \addplot[
    color=red,
    mark=diamond,
    ]
    coordinates {
    (0,0.73817881)(0.02,0.73499649)(0.04,0.72277245)(0.06,0.69743774)(0.08,0.66987428)(0.1,0.63426027)
    };
    \addplot[
    color=red,
    mark=star,
    ]
    coordinates {
    (0,0.63790586)(0.02,0.62902445)(0.04,0.61914144)(0.06,0.61479225)(0.08,0.59200008)(0.1,0.57212934)
    };
    \addplot[
    color=red,
    mark=pentagon*,
    ]
    coordinates {
    (0,0.57557268)(0.02,0.57526571)(0.04,0.56360466)(0.06,0.55762759)(0.08,0.54910936)(0.1,0.53974245)
    };

\addplot[
    color = green,
    mark = square,
    ]
    coordinates {
    (0,0.1169109)(0.02,0.118513)(0.04,0.1216905)(0.06,0.1315331)(0.08,0.1435017)(0.1,0.1679731)
    };
\addplot[
    color=green,
    mark=triangle,
    ]
    coordinates {
    (0,0.2372624)(0.02,0.24156)(0.04,0.2462737)(0.06,0.2533404)(0.08,0.2584248)(0.1,0.2669184)
    };
\addplot[
    color=green,
    mark=asterisk,
    ]
    coordinates {
    (0,0.2859333)(0.02,0.2845277)(0.04,0.2850748)(0.06,0.2846976)(0.08,0.2922149)(0.1,0.2942701)
    };
    \addplot[
    color=green,
    mark=diamond,
    ]
    coordinates {
    (0,0.3019674)(0.02,0.300863)(0.04,0.3021788)(0.06,0.2992694)(0.08,0.2999674)(0.1,0.2992099)
    };
    \addplot[
    color = green,
    mark = star,
    ]
    coordinates {
    (0,0.311746)(0.02,0.3143833)(0.04,0.3103833)(0.06,0.311825)(0.08,0.3108286)(0.1,0.3092891)
    };
    \addplot[
    color=green,
    mark=pentagon*,
    ]
    coordinates {
    (0,0.3133884)(0.02,0.3158663)(0.04,0.3140884)(0.06,0.3178508)(0.08,0.3170073)(0.1,0.317577)
    };
    \end{axis}
\end{tikzpicture}
%
%
%
%
\begin{tikzpicture}[overlay, yshift=0.36\linewidth, xshift=-0.02\linewidth, scale=0.5]
\begin{axis}[
%
%
%
    title={SDF error},
    xlabel={Noise variance},
    xmin=0, xmax=1,
    ymin=0, ymax=0.07,
    xtick={0,0.2,0.4,0.6,0.8,1},
    ytick={0,0.01,0.02,0.03,0.04,0.05,0.06},
    legend pos=north west,
    ymajorgrids=true,
    grid style=dashed,
]
\addplot[
    color=red,
    mark=dot,
    ]
    coordinates {
    (0,0.00284038)(0.2,0.01117536)(0.4,0.0206323)(0.6,0.03144756)(0.8,0.04307106)(1,0.05407404)
    };
    %
%
%
\end{axis}
\end{tikzpicture}
%
%
%
%
%
%
%
%
%
%
%
%
%
%
%
%
%
\caption{Misclassification Rate, Precision, Recall, and Normalized SDF Error for different class error probability and distance noise variance. The top right plot shows the average SDF error over $10$ random maps with a $100$ random observations each, with $voxel\ size = 0.1$, $max(N)=100$, $\delta=1.2$.}
\label{var_prob}
\end{figure}
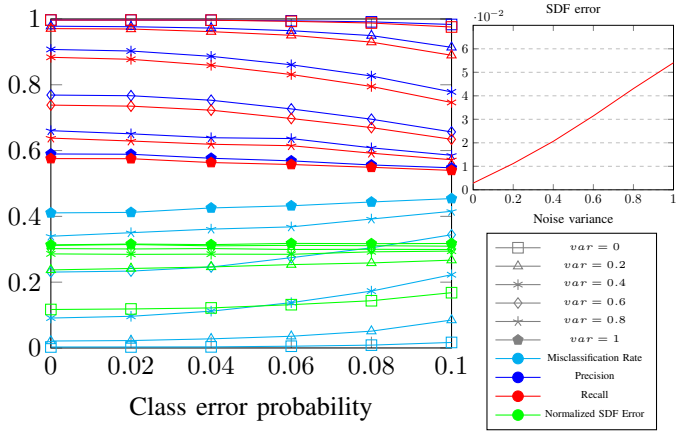

\begin{figure}[t]
\begin{subfigure}[htbp]{0.48\columnwidth}
\begin{tikzpicture}[scale=0.55, line width=1pt]
    \begin{axis}[
	title={},
    	xlabel={Feature points' number in each leaf ($max(N)$)},
	ymode=log,
	xmode=log,
    	xmin=4, xmax=4096,
    	ymin=0, ymax=1,
    	xtick={4,16,64,256,1024},
    	ytick={},
    	ylabel style={overlay, at={(0.91, 0.95)}, anchor=center, font=\tiny},
    	ylabel = {},
    ]
        \addplot[
    color = cyan,
    mark = dot,
    ]
    coordinates {
    (4, 0.01218772) (16, 0.0125142) (64, 0.01230902) (256, 0.01431954) (1024, 0.01828492) (4096, 0.02434235) 
    };
        \addplot[
    color = blue,
    mark = dot,
    ]
    coordinates {
    (4, 0.9824044) (16, 0.4956945) (64, 0.2815095) (256, 0.248844) (1024, 0.2584055) (4096, 0.266601) 
    };
        \addplot[
    color = red,
    mark = dot,
    ]
    coordinates {
    (4, 1-0.98606546) (16, 1-0.98578224) (64, 1-0.98603365) (256, 1-0.9843247) (1024, 1-0.98046386) (4096, 1-0.97435744) 
    };
        \addplot[
    color = green,
    mark = dot,
    ]
    coordinates {
    (4, 1-0.98631709) (16, 1-0.98601404) (64, 1-0.9861766) (256, 1-0.983555) (1024, 1-0.97824619) (4096, 1-0.97057522) 
    };
    \end{axis}
\end{tikzpicture}
\end{subfigure}\hfill\hspace{1pt}
\begin{subfigure}[htbp]{0.48\columnwidth}
\begin{tikzpicture}[scale=0.55, line width=1pt]
    \begin{axis}[
	title={},
    	xlabel={Over lap ratio of tree's leaves ($\delta$)},
	ymode=log,
    	xmin=1, xmax=2,
    	ymin=0, ymax=1,
    	xtick={1,1.2,1.4,1.6,1.8,2},
    	ytick={},
    	ylabel style={overlay, at={(0.91, 0.95)}, anchor=center, font=\small},
    	ylabel = {},
	yticklabels={,,},
    ]
        \addplot[
    color = cyan,
    mark = dot,
    ]
    coordinates {
    (1, 0.00981909) (1.2, 0.00957082) (1.4, 0.01121102) (1.6, 0.01337076) (1.8, 0.01736967) (2, 0.02158993)
    };
        \addplot[
    color = blue,
    mark = dot,
    ]
    coordinates {
    (1, 0.3348038) (1.2, 0.2620699) (1.4, 0.2418551) (1.6, 0.2504536) (1.8, 0.2601078) (2, 0.2629537)
    };
        \addplot[
    color = red,
    mark = dot,
    ]
    coordinates {
    (1, 1-0.98858812) (1.2, 1-0.98905908) (1.4, 1-0.98756481) (1.6, 1-0.9856564) (1.8, 1-0.98112406) (2, 1-0.97635192)
    };
        \addplot[
    color = green,
    mark = dot,
    ]
    coordinates {
    (1, 1-0.9880466) (1.2, 1-0.98813084) (1.4, 1-0.98605807) (1.6, 1-0.98265943) (1.8, 1-0.97723401) (2, 1-0.97314582) 
    };
    \end{axis}
\end{tikzpicture}
\end{subfigure}\\
\begin{subfigure}[htbp]{0.48\columnwidth}
\begin{tikzpicture}[scale=0.55,line width=1pt]
    \begin{axis}[
        legend style={anchor=center, legend columns=-1, font=\tiny},
        legend entries={Misclassification Rate, Normalized SDF Error, FDR, FNR},
        legend to name=named,
	title={},
    	xlabel={Gaussian Process noise variance},
	ymode=log,
    	xmin=0, xmax=2.5,
    	ymin=0, ymax=1,
    	xtick={0,0.5,1,1.5,2},
    	ytick={},
    	ylabel style={overlay, at={(0.91, 0.95)}, anchor=center, font=\small},
    	ylabel = {},
    ]
        \addplot[
    color = cyan,
    mark = dot,
    ]
    coordinates {
    (0, 0.02189483) (0.5, 0.01687455) (1, 0.01654711) (1.5, 0.01629494) (2, 0.01886629) (2.5, 0.02227487)
    };
        \addplot[
    color = blue,
    mark = dot,
    ]
    coordinates {
    (0, 0.2703628) (0.5, 0.2421209) (1, 0.2433853) (1.5, 0.2593666) (2, 0.2847622) (2.5, 0.3337421)
     };
        \addplot[
    color = red,
    mark = dot,
    ]
    coordinates {
    (0, 1-0.97578877) (0.5, 1-0.98137738) (1, 1-0.98173447) (1.5, 1-0.98154523) (2, 1-0.97954766) (2.5, 1-0.97680789)
    };
        \addplot[
    color = green,
    mark = dot,
    ]
    coordinates {
    (0, 1-0.97210083) (0.5, 1-0.97965053) (1, 1-0.9801948) (1.5, 1-0.97996812) (2, 1-0.97678082) (2.5, 1-0.97163774) 
    };
    \end{axis}
\end{tikzpicture}
\end{subfigure}
\begin{subfigure}[htbp]{0.48\columnwidth}
\begin{tikzpicture}[scale=0.55,line width=1pt]
    \begin{axis}[
	title={},
    	xlabel={Voxel size},
    	xmin=0.05, xmax=0.3,
    	ymin=0, ymax=1,
    	xtick={0.05,0.1,0.15,0.2,0.25,0.3},
	ymode=log,
    	ytick={},
    	ylabel style={overlay, at={(0.91, 0.95)}, anchor=center, font=\small},
    	ylabel = {},
	yticklabels={,,},
    ]
        \addplot[
    color = cyan,
    mark = dot,
    ]
    coordinates {
    (0.05, 0.00651495) (0.1, 0.01007735) (0.15, 0.01299705) (0.2, 0.01575142) (0.25, 0.01804737) (0.3, 0.02473469)
    };
        \addplot[
    color = blue,
    mark = dot,
    ]
    coordinates {
    (0.05, 0.01408539/0.05) (0.1, 0.02408478/0.1) (0.15, 0.03842124/0.15) (0.2, 0.06327664/0.2) (0.25, 0.10695529/0.25) (0.3, 0.18488528/0.3)
    };
        \addplot[
    color = red,
    mark = dot,
    ]
    coordinates {
    (0.05, 1-0.99261828) (0.1, 1-0.98850113) (0.15, 1-0.98459437) (0.2, 1-0.98218475) (0.25, 1-0.9788855) (0.3, 1-0.97191114)
    };
        \addplot[
    color = green,
    mark = dot,
    ]
    coordinates {
    (0.05, 1-0.99155083) (0.1, 1-0.98730615) (0.15, 1-0.98425005) (0.2, 1-0.97951085) (0.25, 1-0.97839137) (0.3, 1-0.97016124)
    };
    \end{axis}
\end{tikzpicture}
\end{subfigure}
\ref{named}
%
%
%
%
%
%
\caption{Misclassification rate, normalized SDF error, False Discovery Rate (FDR), and False Negative Rate (FNR) as a function of the number of pseudo points per octree support region ($max(N)$), support region overlap ratio ($\delta$), GP noise variance $\sigma^2$, and workspace discretization ($voxel\ size$). The default parameter values are $\delta=1.5$, $max(N)=100$, $\sigma^2=1$, $voxel\ size=0.1$. Class and distance measurements with class error probability of $0.05$ and distance noise variance $0.5$ are obtained from $100$ random observations in each of $50$ random 2-D maps. Test points are selected within a threshold of $0.05$ from the ground truth class boundaries.}
\label{para}
\end{figure}
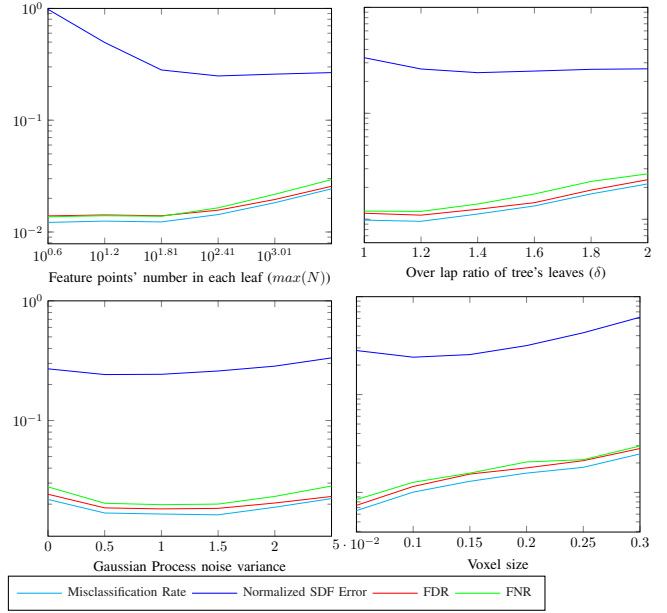


\subsection{Multi-Robot 2-D Evaluation}
\label{sec:2d-multi}

Next, we evaluate the distributed GP regression in a three-robot simulation and investigate the convergence of the local GP estimates of each robot to a centralized GP estimate. We use the same random polygonal 2-D environments with two object classes but this time generate trajectories for three different robots (see Fig.~\ref{multi2D}). The robots communicate with each other over a graph with a fixed weight matrix:
\begin{equation}\label{eq:wightmatrix} 
\weight=\begin{bmatrix}
0.5 & 0.25 & 0.25\\
0.25 & 0.75 & 0\\
0.25 & 0 & 0.75\\
\end{bmatrix}. 
\end{equation}
The GP regression parameters at each robot are the same as the defaults in Sec.~\ref{sec:2d-single}. 

To verify Prop.~\ref{prop:cnvrgn} empirically, we compare the mean absolute error (MAE) between the GP prediction of an individual robot $\rb$ and the centralized estimator $ctr$ using all robot observations as described in Sec.~\ref{sec:centralized-estimator}. Specifically, at each time step $t$, we consider all classes $\cls$ and associated pseudo points $\calP_{t,\cls}^{ctr}$ that have been observed by the centralized estimator and calculate the mean MAE as:
\begin{equation}\label{eq:mae}
MAE_t = \frac{1}{L_t |\calP_{t,\cls}^{ctr}|}\sum_\ell \sum_{\bfp \in \calP_{t,\cls}^{ctr}} \left| \mu_{t,\cls}^\rb(\bfp) - \mu_{t,\cls}^{ctr}(\bfp) \right|,
\end{equation}
where $L_t$ is the number of observed object classes by time $t$. The variance MAE is computed equivalently to \eqref{eq:mae} with $\mu_{t,\cls}^\rb(\bfp)$ and $\mu_{t,\cls}^{ctr}(\bfp)$ replaced by $k_{t,\cls}^\rb(\bfp,\bfp)$ and $k_{t,\cls}^{ctr}(\bfp,\bfp)$.

Fig.~\ref{multi2D} shows the final reconstructions of one robot and the centralized estimator. As expected, the final reconstructions are identical and convergence happens in finite time. The behavior of the mean and variance MAE curves is similar. This is expected because the distance between the local and centralized GP parameters is due to unobserved information rather than stochastic noise. We see that the MAE curves approach $0$ quickly. Several peaks are observed in the curves when new sections of the environment that are not visible to robot $\rb$ are observed by another robot in the network. The new information disseminates in the network and the MAE curves approach zero again.

\begin{figure}[t]
\includegraphics[width=0.5\linewidth]{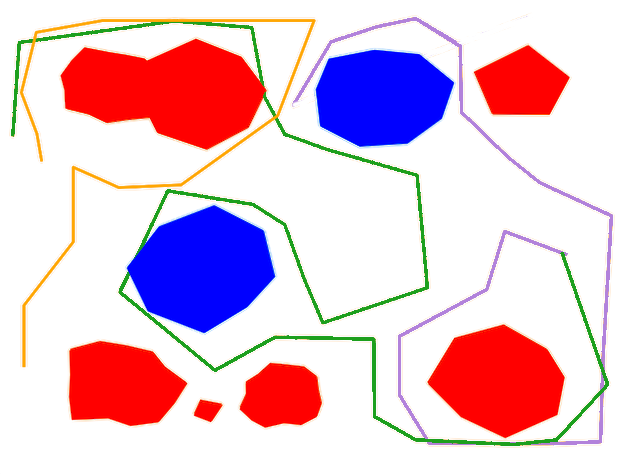}%
\includegraphics[width=0.5\linewidth]{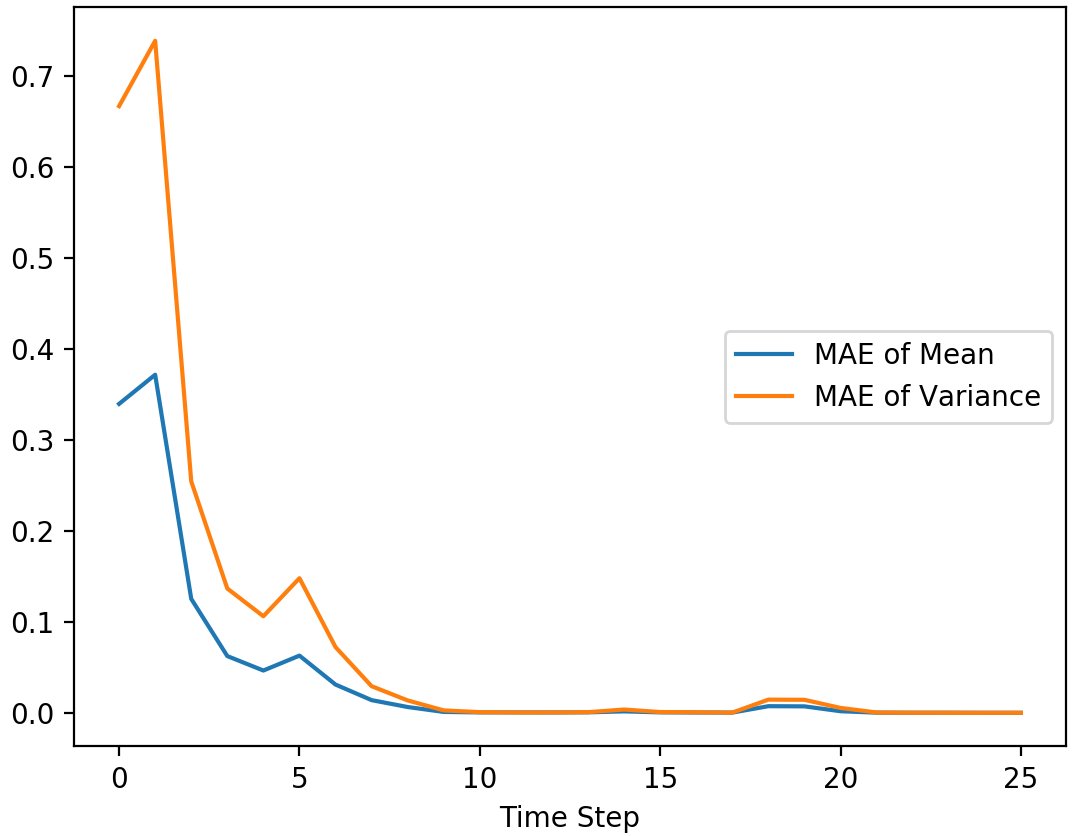}\\
\includegraphics[width=0.5\linewidth]{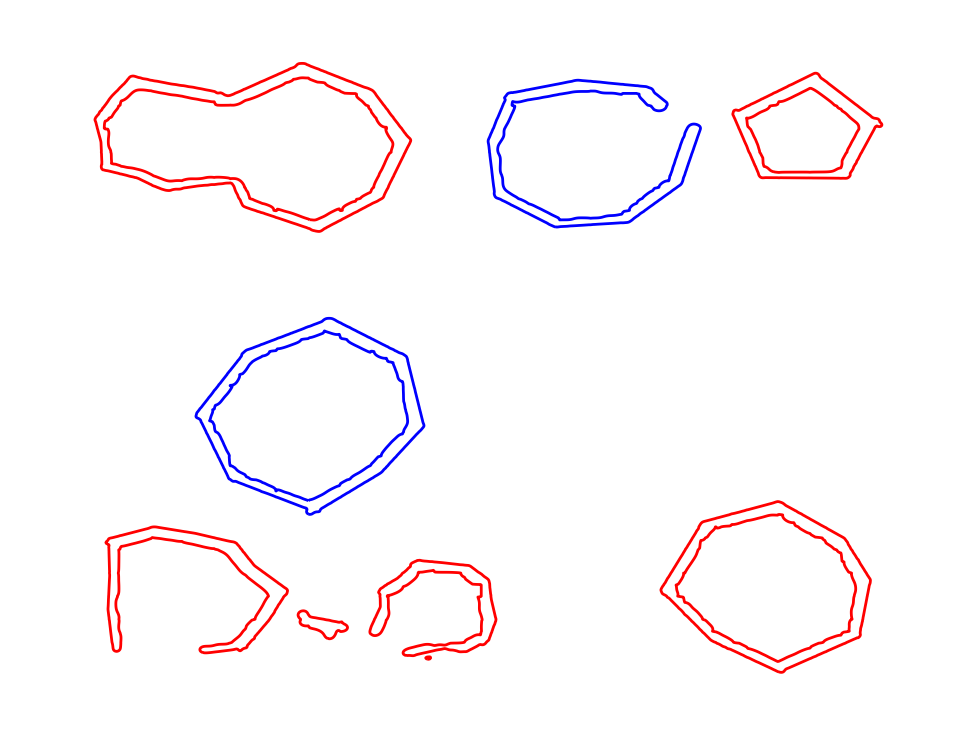}%
\includegraphics[width=0.5\linewidth]{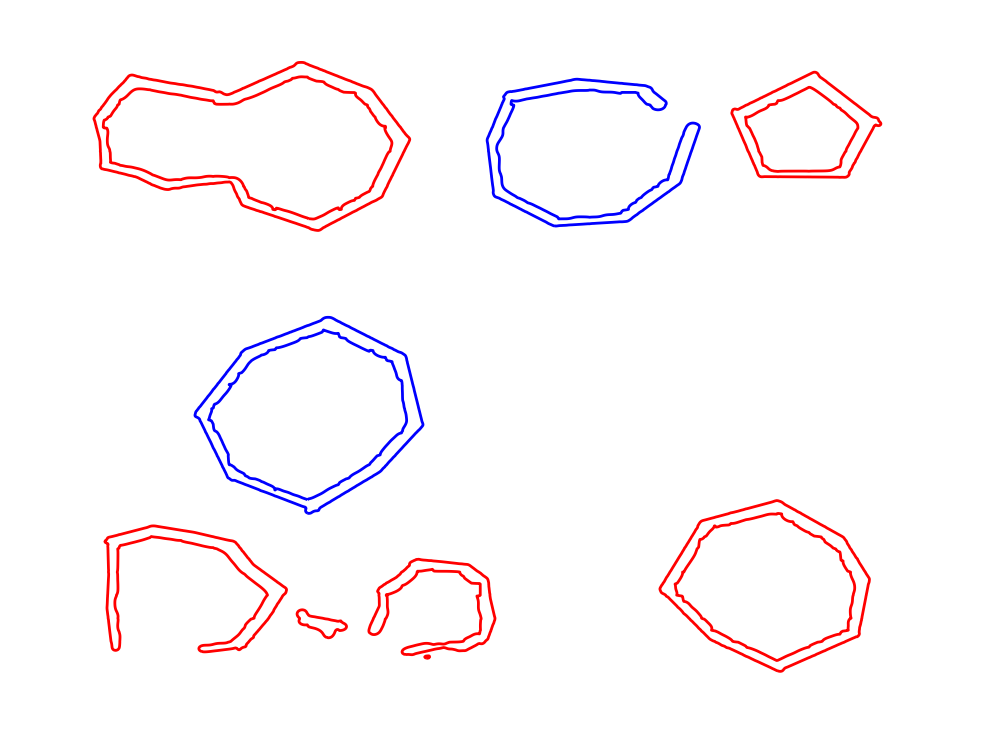}
\caption{
Three robot trajectories (green, orange, purple) in a 2-D simulated environment (top left) with two object classes (red, blue). The zero level-sets of the TSDF reconstructions for the two classes by centralized GP regression (bottom left) and distributed GP regression from the perspective of the orange robot (bottom right) are shown. As expected, due to Prop.~\ref{prop:class_prediction}, the centralized and individual robot reconstructions are identical. This is verified quantitatively in the GP mean and variance mean absolute error (MAE) plot (top right). The initial GP parameters for each robot and object class were $\mu^\rb_{0,\cls}(\bfx)=0.5, k^\rb_{0,\cls}(\bfx,\bfx)=1$.}
\label{multi2D}
\end{figure}

\section{Evaluation using 3-D Real Data}
\label{sec:evaluation3d}

In this section, we evaluate our semantic TSDF mapping approach using real RGB-D data from physical 3-D environments. We demonstrate the qualitative and quantitative performance of the single-robot approach of Sec.~\ref{sec:algorithm} and the multi-robot approach of Sec.~\ref{sec:DistributedMSM}, using three robots to map the same environment collaboratively. As in the 2-D experiments in Sec.~\ref{sec:evaluation2d}, we use a sparse Mat{\'e}rn kernel ($v = 3/2$) and a grid of potential pseudo points $\allPs$ with resolution $voxel\ size$. Given a query point $\hat{\bfx}$, we choose a cubic region around it such that $(frame\ size-1) \times voxel\ size \geq 2 \times \epsilon$ to construct the training data in \eqref{Gtc}. All points from $\allPs$ that lie in the cubic region are chosen as pseudo points associated with $\hat{\bfx}$.



\subsection{Single-Robot 3-D Evaluation}
\label{sec:3d-single}
We compare our method to the incremental Euclidean signed distance mapping method Fiesta \cite{han2019fiesta} on the Cow and Lady dataset \cite{oleynikova2017voxblox}. We also demonstrate the 3-D semantic reconstruction performance of our method on the SceneNN dataset~\cite{scenenn-3dv16}.

\subsubsection{Cow and Lady Dataset}

The reconstruction of the Cow and Lady dataset with $829$ depth images and known camera trajectory by the single-robot TSDF GP regression of Sec.~\ref{sec:algorithm} is shown in Fig.~\ref{fig:cow-and-lady}. A triangular mesh is extracted from the mean TSDF prediction using the marching cubes algorithm \cite{marchingcubes}. The reconstruction time and error with respect to the ground-truth scene point cloud provided by the dataset are reported in Fig.~\ref{cowLady1}. The error of Fiesta with default parameters is shown as well. Similar to the 2-D simulations, increasing the maximum number of pseudo points $max(N)$ per octree support region improves the SDF error of our approach. The improvement is significant at first and less pronounced afterwards. Conversely, the computation time decreases at first because the number of leaves in the octree decreases and then increases afterwards as the GP covariance matrices get larger. Increasing $\delta$ leads to an insignificant improvement in the SDF error at the expense of a significant reconstruction time increase. Increasing the GP noise variance improves the SDF error at first (especially when the error is close to zero) but worsens is afterwards without significant impact on time. As $voxel\ size$ varies, our method outperforms Fiesta noticeably.

%
%

\begin{figure}[t]
\includegraphics[width=8.5cm]{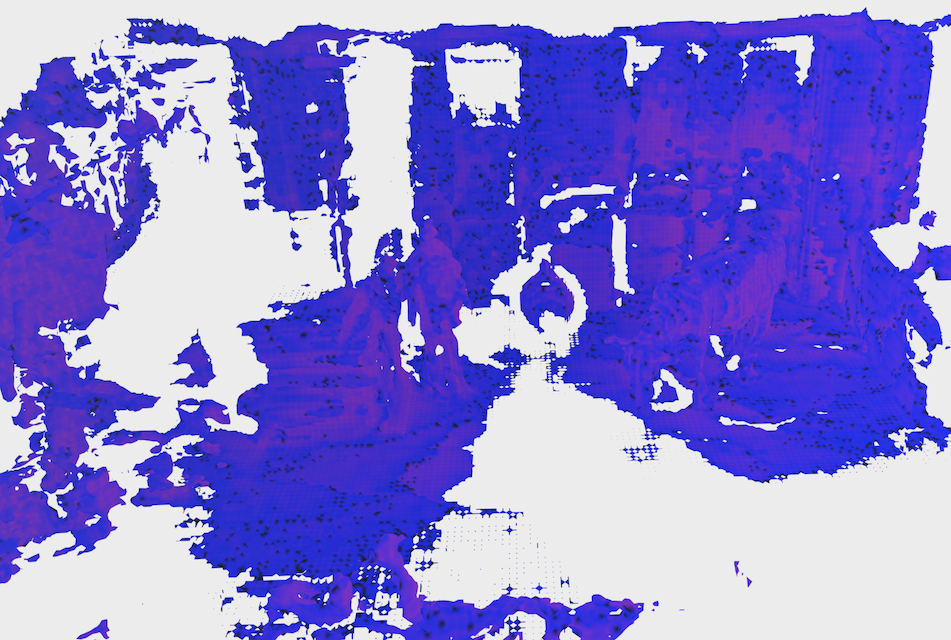}
\caption{Single-robot reconstruction of the Cow and Lady dataset \cite{oleynikova2017voxblox}. Red hues indicate lower TSDF variance.}
\label{fig:cow-and-lady}
\end{figure}

\begin{figure}[t]
\begin{subfigure}[htbp]{0.48\columnwidth}
\begin{tikzpicture}[scale=0.52, line width=1pt]
    \begin{axis}[
	title={},
    	xlabel={Feature points' number in each leaf ($Max(N)$)},
    	xmin=10, xmax=1000,
    	ymin=0, ymax=0.1,
    	xtick={10, 100, 200, 400, 600, 800},
    	ytick={},
    	ylabel style={ anchor=center, font=\small},
    	ylabel = {(m)},
    ]
        \addplot[
    color = red,
    mark = dot,
    ]
    coordinates {
    (10, 0.06880932346) (50, 0.04711500897) (100, 0.0445206839) (200, 0.04294358839) (400, 0.0414028134) (600, 0.04044528084) (800, 0.03918963097) (1000, 0.03844095222)
    };
\end{axis}
\begin{axis}[
    xmin=10, xmax=1000,
    ymin=330, ymax=520,
    axis y line=right,
    ylabel style={anchor=center, font=\small},
        	xtick={10, 100, 200, 400, 600, 800},
    ytick={},
    yticklabels={,,},
     ]
        \addplot[
    color = blue,
    mark = dot,
    ]
    coordinates {
    (10, 513.71) (50, 370.23) (100, 360.82) (200, 369.75) (400, 373.49) (600, 385.46) (800, 419.67) (1000, 463.60) 
    };
    \end{axis}
\end{tikzpicture}
\end{subfigure}\hfill\hspace{1pt}
\begin{subfigure}[htbp]{0.48\columnwidth}
\begin{tikzpicture}[scale=0.52,line width=1pt]
    \begin{axis}[
	title={},
    	xlabel={Over lap ratio of each leaf ($\delta$)},
    	xmin=1, xmax=2,
    	ymin=0, ymax=0.1,
    	xtick={1,1.2, 1.4, 1.6, 1.8, 2},
    	ytick={},
    	ylabel style={ anchor=center, font=\small},
	yticklabels={,,},
    ]
        \addplot[
    color = red,
    mark = dot,
    ]
    coordinates {
    (1, 0.04320124374) (1.2, 0.04311415269) (1.4, 0.0430272118) (1.6, 0.04172228562) (1.8, 0.04039260812) (2, 0.03918730467) 
    };
\end{axis}
\begin{axis}[
    xmin=1, xmax=2,
    ymin=330, ymax=520,
    axis y line=right,
    ylabel style={anchor=center, font=\small},
    ylabel={(sec)},
    ytick={},
     ]
        \addplot[
    color = blue,
    mark = dot,
    ]
    coordinates {
    (1, 334.13) (1.2, 339.95) (1.4, 358.13) (1.6, 379.50) (1.8, 395.12) (2, 405.98) 
    };
    \end{axis}
\end{tikzpicture}
\end{subfigure}\\
\begin{subfigure}[htbp]{0.48\columnwidth}
\begin{tikzpicture}[scale=0.52,line width=1pt]
    \begin{axis}[
	title={},
    	xlabel={Gaussian Process noise variance},
    	xmin=0, xmax=100,
    	ymin=0, ymax=0.25,
    	xtick={0, 5, 10, 25, 40, 50, 80},
    	ytick={},
    	ylabel style={ anchor=center, font=\small},
    	ylabel = {(m)},
    ]
        \addplot[
    color = red,
    mark = dot,
    ]
    coordinates {
    (0, 0.0447149589) (5, 0.04349856944) (10, 0.0432992905) (25, 0.04294358839) (40, 0.04279555971) (50, 0.04276849214) (80, 0.04294045161) (100, 0.04319170428)
    };
\end{axis}
\begin{axis}[
    xmin=0, xmax=100,
    ymin=250, ymax=750, 
    axis y line=right,
    ylabel style={anchor=center, font=\small},
    xtick={0, 5, 10, 25, 40, 50, 80},
    ytick={},
     yticklabels={,,},
     ]
        \addplot[
    color = blue,
    mark = dot,
    ]
    coordinates {
    (0, 368.54) (5, 369.60) (10, 369.59) (25, 372.06) (40, 368.23) (50, 368.37) (80, 367.85) (100, 368.68) 
    };

    \end{axis}
\end{tikzpicture}
\end{subfigure}
\begin{subfigure}[htbp]{0.48\columnwidth}
\hfill\hspace*{-3mm}\begin{tikzpicture}[scale=0.52,line width=1pt]
    \begin{axis}[
	title={},
    	xlabel={Voxel size},
    	xmin=0.05, xmax=0.2,
    	ymin=0, ymax=0.25,
    	xtick={0.05, 0.1, 0.2},
    	ytick={},
    	ylabel style={ anchor=center, font=\small},
	yticklabels={,,},
    ]
        \addplot[
    color = red,
    mark = square,
    ]
    coordinates {
    (0.05, 0.02851870903) (0.1, 0.04294358839) (0.2, 0.06127624831) 
    };
            \addplot[
    color = green,
    mark = square,
    ]
    coordinates {
    (0.05, 0.05354238896528112) (0.1, 0.10154338521867906) (0.2, 0.21609132220118651) 
    };
\end{axis}
\begin{axis}[
    legend style={anchor=center, legend columns=-1, font=\small},
    legend entries={Time (sec),Error (m), Fiesta Error (m)},
    legend to name=named,
    xtick={0.05, 0.1, 0.2},
    xmin=0.05, xmax=0.2,
    ymin=250, ymax=750,
    axis y line=right,
    ylabel style={anchor=center, font=\small},
    ylabel={(sec)},
    ytick={},
     ]
        \addplot[
    color = blue,
    mark = dot,
    ]
    coordinates {
    (0.05, 748.70) (0.1, 372.06) (0.2, 293.78)
    };
        \addplot[
    color = red,
    mark = dot,
    ]
    coordinates {
    (0, 10000)
    };
         \addplot[
    color = green,
    mark = dot,
    ]
    coordinates {
    (0, 10000)
    };
    \end{axis}
\end{tikzpicture}
\end{subfigure}
\ref{named}
%
%
%
%
%
\caption{Evaluation of the SDF reconstruction time (sec) and error (m) of our incremental sparse GP regression algorithm on the Cow and Lady dataset \cite{oleynikova2017voxblox} and in comparison with Fiesta \cite{han2019fiesta}. The errors are evaluated with respect to the ground-truth scene point cloud provided by the dataset. Training is done with 829 depth images and known camera trajectory. The default parameters for our algorithm are $max(N)=200$, $\delta=1.5$, $\sigma^2 = 25$, $voxel\ size = 0.1$, $frame\ size=5$, and SDF truncation value $3\times voxel\ size$.}
\label{cowLady1}
\end{figure}
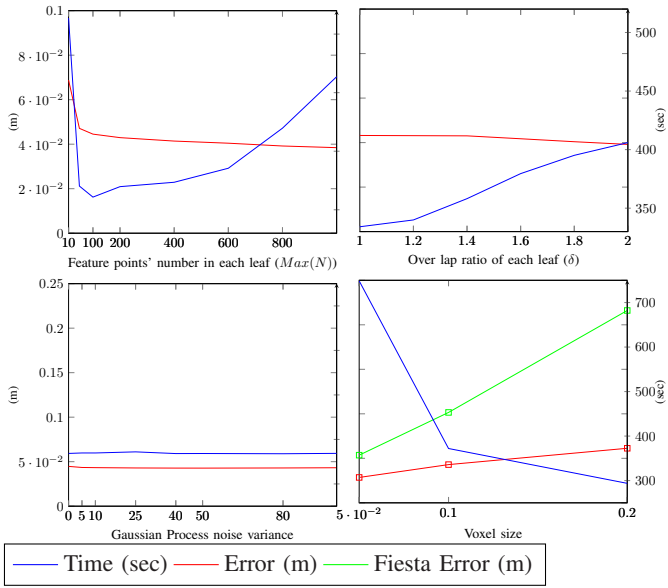

\begin{figure*}[t]
\input{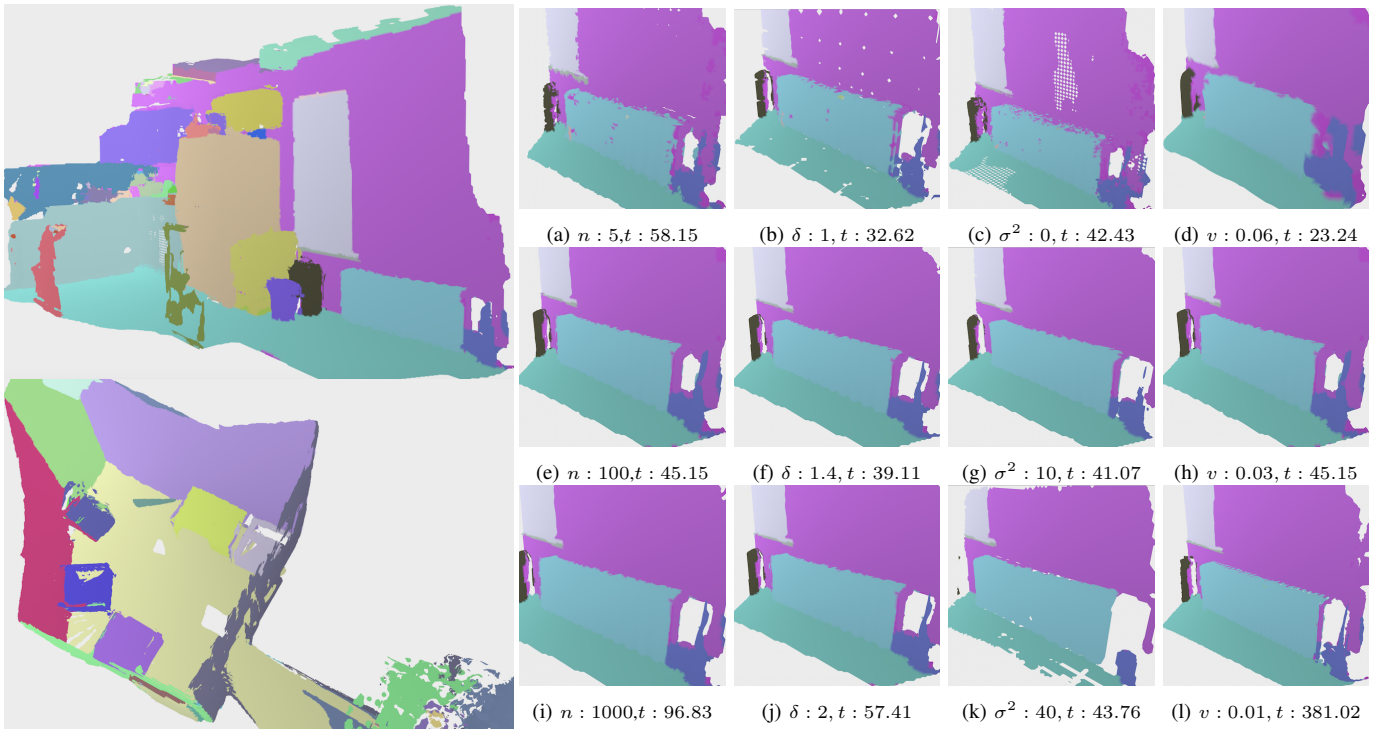}
\caption{Single-robot reconstructions of sequence $255$ (top left), containing $2450$ RGB-D images and $85$ semantic categories (in random colors), and sequence $011$ (bottom left), containing $3700$ RGB-D images and $61$ semantic categories (in random colors), of the SceneNN dataset \cite{scenenn-3dv16}. The incremental sparse GP TSDF mapping process took $1040.41$ sec. for sequence $255$ and $1885.72$ sec. for sequence $011$. The following default parameters were used for the octree: $\delta=1.5$, $n = max(N)=100$ and the GP training: $\sigma^2=3$, $v = voxel\ size = 0.03$, $frame\ size=1$. On the right we see the effect of these parameters ($t$ is time in seconds) on the metric-semantic reconstruction over $140$ RGB-D images.}
\label{classPara}
\end{figure*}

\subsubsection{SceneNN Dataset}

We evaluate the classification accuracy of our method on the SceneNN dataset in Fig.~\ref{classPara}. The GP posterior is evaluated on a test grid with resolution $0.5 \times voxel\ size$. The test points with posterior variance less than a threshold are used to reconstruct a triangular mesh via the marching cubes algorithm \cite{marchingcubes}. We use Prop.~\ref{prop:class_prediction} for classification. The effect of the different parameters on the performance is illustrated in Fig.~\ref{classPara}. Increasing $max(N)$ improves both classification and TSDF reconstruction results. The improvement after $max(N) = 100$ is negligible but time increases significantly. Increasing $\delta$ improves the TSDF reconstruction significantly at first. After $\delta = 1.4$, the improvement is negligible. As seen in the 2-D simulations, choosing a correct magnitude for the GP noise variance $\sigma^2$ is very important for both the classification and TSDF reconstruction but choosing the optimal value for $\sigma^2$ is not critical.




\subsection{Multi-Robot 3-D Evaluation}
Finally, we evaluate our distributed GP regression on the Cow and Lady and SceneNN datasets. To imitate data collection by multiple robots, we split the RGB-D image sequences into equal parts and consider each as data obtained by a different robot. As in the 2-D simulation, we use three robots with communication structure specified by the weight matrix $\weight$ in \eqref{eq:wightmatrix}. Each robot uses the distributed update rule in \eqref{eq:final_update_rule} and communication continues for 2 rounds after the last RGB-D image from the individual robot sequences is received. The parameters of the individual robots are the same as in the single-robot experiments in Sec.~\ref{sec:3d-single}. The choice of additional rounds is due to Prop.~\ref{prop:cnvrgn}, where we showed theoretically that $T + n - 1$ rounds are needed, where $T$ is the observation sequence length and $n$ is the number of robots, for the local GP distributions to agree with that of a centralized GP estimator. As in the 2-D simulations, to verify Prop.~\ref{prop:cnvrgn} empirically, we compare the mean absolute error (MAE) in \eqref{eq:mae} between the GP mean and variance of an individual robot and the centralized estimator.

The results from the Cow and Lady dataset are reported in Fig.~\ref{oneClassMulti} and Fig.~\ref{oneClassMultiMeanVar}, while those from the SceneNN dataset are reported in Fig.~\ref{multiClassMulti} and Fig.~\ref{multiClassMultiMeanVar}. The local and centralized reconstruction results are identical in both data sets, which confirms the expected theoretical consistency. The mean and variance MAE curves also behave similarly in both data sets because the errors in the local GP regression are due to unobserved information, that has not yet been received by the robot, rather than measurement noise. As in the 2-D simulation, the peaks in the MAE curves are due to another robot in the network observing a new region that has not yet been observed by this robot. These peaks quickly decrease, which indicates the fast empirical convergence of the distributed sparse GP algorithm.

\begin{figure*}[t]
\begin{minipage}[b]{.7\textwidth}
\includegraphics[width=0.5\linewidth]{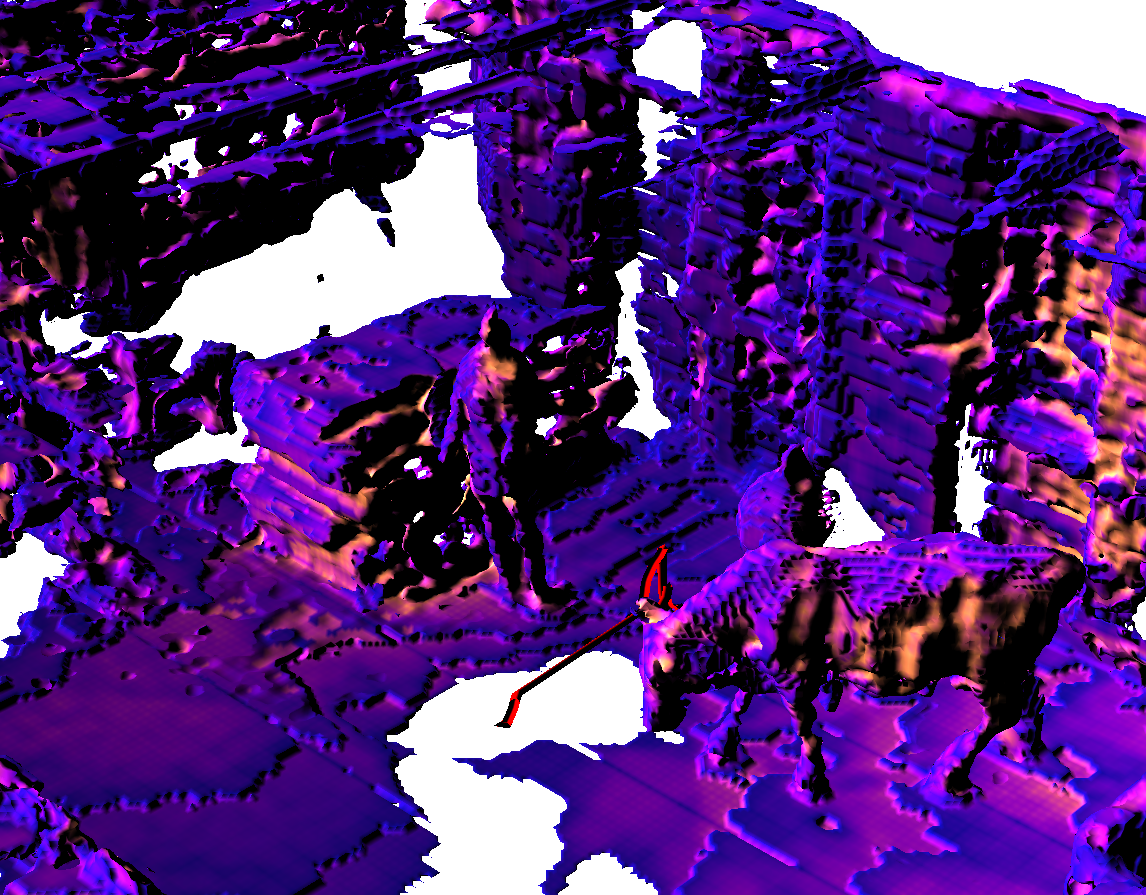}
\hfill%
\includegraphics[width=0.5\linewidth]{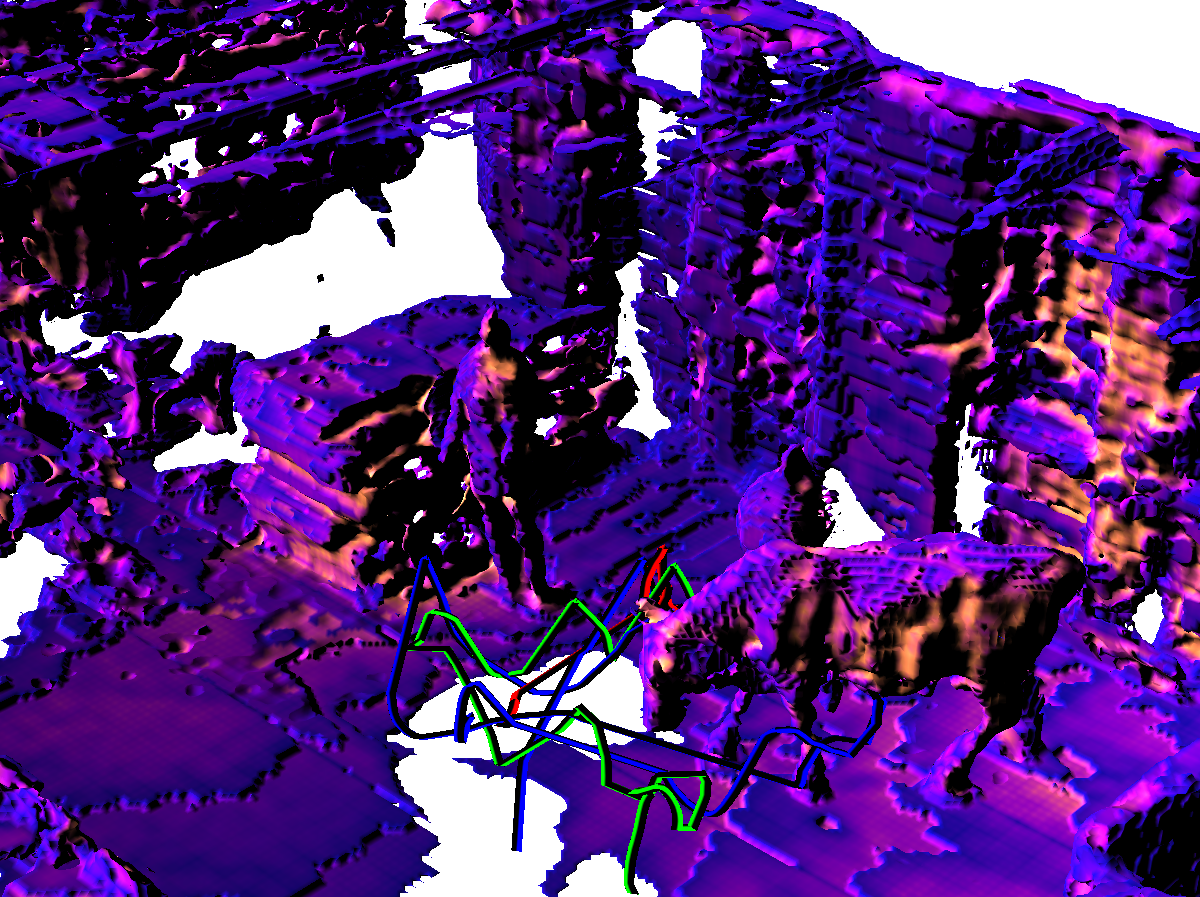}
\caption{The Cow and Lady dataset \cite{oleynikova2017voxblox} is divided into three equal sequences of about $275$ depth images, and each is considered data obtained by one robot. The three camera trajectories are shown in red, green, and blue on the right. The left plot shows the final reconstruction obtained by the first robot. The right plot shows the final reconstruction obtained by centralized GP regression using the observations of all three robots. The orange hues indicate larger variance. As expected, due to Prop.~\ref{prop:cnvrgn}, the reconstruction of robot one is identical with that of the centralized estimator. The initial GP parameters for each robot and object class were $\mu^\rb_{0,\cls}(\bbx)=0.15$ and $k^\rb_{0,\cls}(\bbx,\bbx)=5$.}
\label{oneClassMulti}
\end{minipage}%
\hfill%
\begin{minipage}[b]{.29\textwidth}
  \centering
  \includegraphics[width=\linewidth]{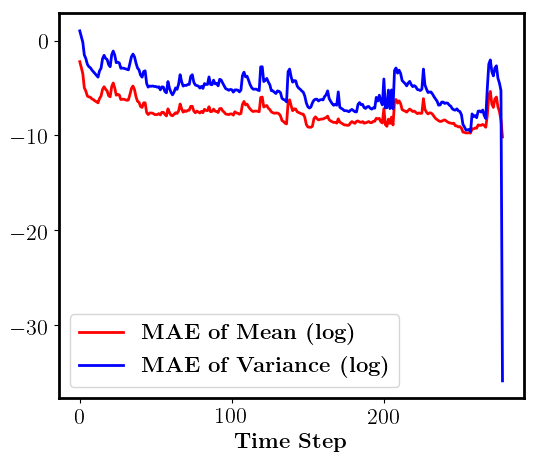}
  \caption{Log-space plot of the mean absolute error (MAE) between the mean (red) and variance (blue) predictions of robot $1$ and centralized GP regression for the sequence in Fig.~\ref{oneClassMulti}. When the data streaming stops at the end, the MAE approaches zero ($-\infty$ in log space).}
  \label{oneClassMultiMeanVar}
\end{minipage}
\end{figure*}

\begin{figure*}[t]
\begin{minipage}[b]{.7\linewidth}
\includegraphics[width=0.5\linewidth]{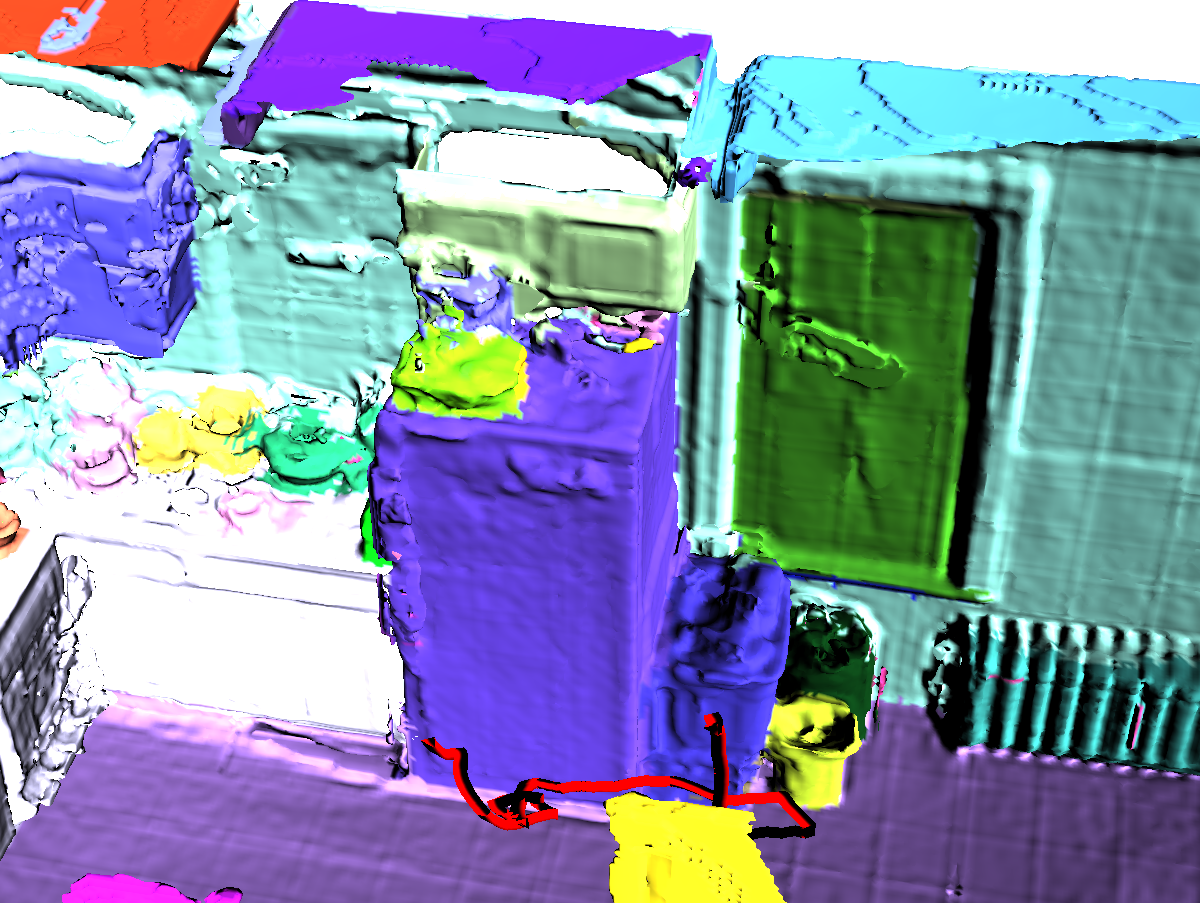}
\hfill%
\includegraphics[width=0.5\linewidth]{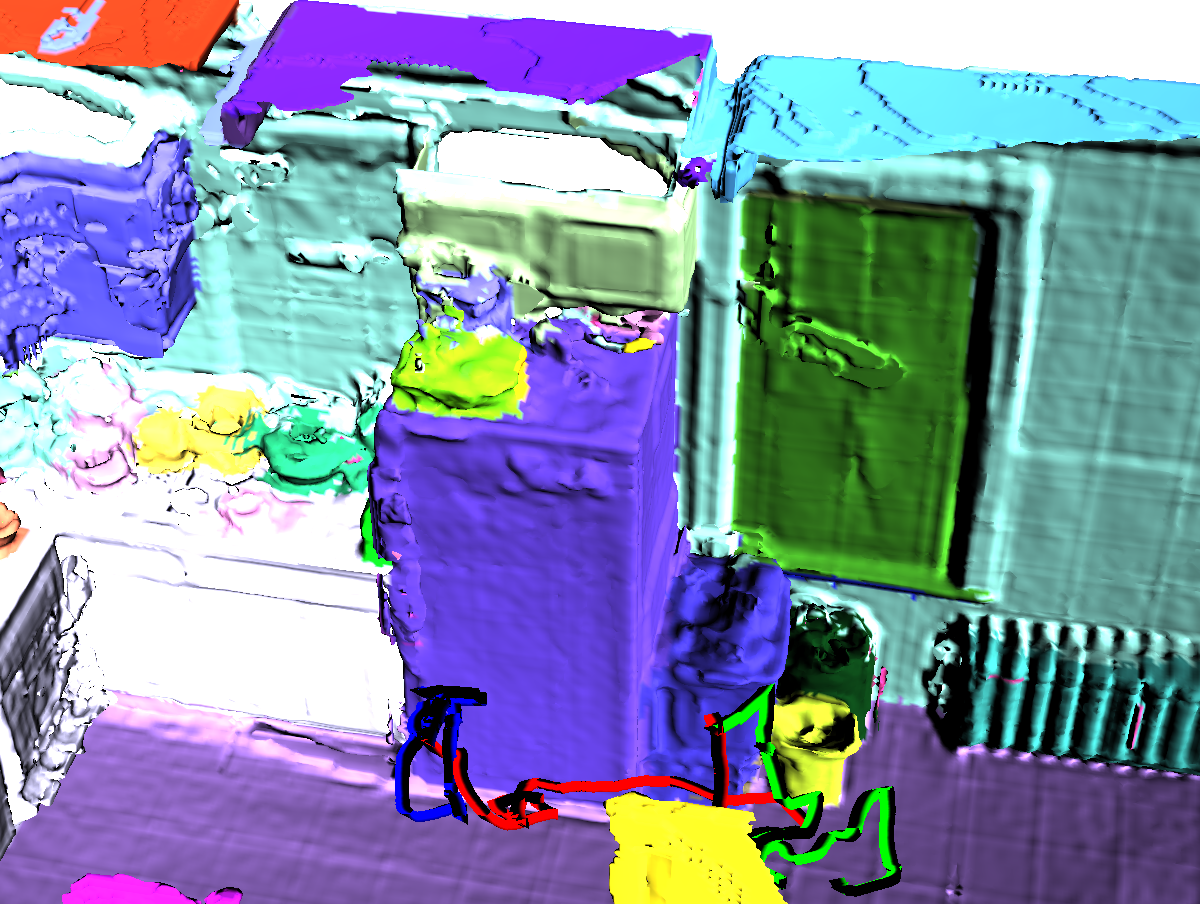}
\caption{Sequence $255$ of the SceneNN dataset \cite{scenenn-3dv16} is divided into three equal sequences of about $800$ RGB-D images, and each is considered data obtained by a different robot. The three camera trajectories are shown in red, green, and blue on the right. The left plot shows the final metric-semantic reconstruction obtained by the first robot. The right plot shows the final reconstruction obtained by centralized GP regression using the observations of all three robots. As expected, due to Prop.~\ref{prop:cnvrgn}, the reconstruction of robot one is identical with that of the centralized estimator. The initial GP parameters for each robot and object class were $\mu^\rb_{0,\cls}(\bbx)=0.09$ and $k^\rb_{0,\cls}(\bbx,\bbx)=5$.}
\label{multiClassMulti}
\end{minipage}%
\hfill%
\begin{minipage}[b]{.29\linewidth}
  \centering
  \includegraphics[width=\linewidth]{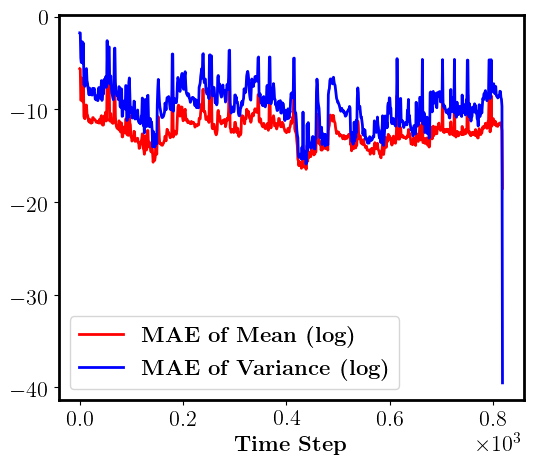}
\caption{Log-space plot of the mean absolute error (MAE) between the mean (red) and variance (blue) predictions of robot $1$ and centralized GP regression for the sequence in Fig.~\ref{multiClassMulti}. When the data streaming stops at the end, the MAE approaches zero ($-\infty$ in log space).}
\label{multiClassMultiMeanVar}
\end{minipage}
\end{figure*}

\section{Conclusion}
\label{sec:conclusion}

This paper developed a Bayesian inference method for online probabilistic metric-semantic mapping via scalable Gaussian Processes regression of semantic class signed distance functions. Our algorithm was enabled by several key ideas. First, repeated observations of the same environment locations can be compressed before training a GP regression method without any effect on the posterior distribution. This, combined with an overlapping-leaf octree data structure of pseudo points, allowed the development of an incremental sparse GP regression technique, which scales to large domains. Second, instead of explicit modeling of class likelihoods and reliance on computationally challenging GP classification techniques, the presence of distance measurements allows independent GP regression for each class. A class probability mass function can still be recovered at test time based on the distance distributions, and its accuracy was shown empirically to be resilient to increasing classification error rates. Third, distributed parameter estimation techniques based on consensus can be extended to distributed function estimation by relying on incrementally growing pseudo points. This enables distributed incremental sparse GP regression, guaranteed to converge in finite-time to the same distribution as that of a centralized estimator without relying on multi-hop communication. Our method enables robot teams to collaboratively build dense metric-semantic maps of unknown environments using streaming RGB-D measurements. This offers a promising direction for future research in semantic task specifications and uncertainty-aware task planning. 

\bibliographystyle{cls/IEEEtran.bst}
\bibliography{bib/ref.bib}

\begin{thebibliography}{10}
\providecommand{\url}[1]{#1}
\csname url@samestyle\endcsname
\providecommand{\newblock}{\relax}
\providecommand{\bibinfo}[2]{#2}
\providecommand{\BIBentrySTDinterwordspacing}{\spaceskip=0pt\relax}
\providecommand{\BIBentryALTinterwordstretchfactor}{4}
\providecommand{\BIBentryALTinterwordspacing}{\spaceskip=\fontdimen2\font plus
\BIBentryALTinterwordstretchfactor\fontdimen3\font minus
  \fontdimen4\font\relax}
\providecommand{\BIBforeignlanguage}[2]{{%
\expandafter\ifx\csname l@#1\endcsname\relax
\typeout{** WARNING: IEEEtran.bst: No hyphenation pattern has been}%
\typeout{** loaded for the language `#1'. Using the pattern for}%
\typeout{** the default language instead.}%
\else
\language=\csname l@#1\endcsname
\fi
#2}}
\providecommand{\BIBdecl}{\relax}
\BIBdecl

\bibitem{curless1996volumetric}
B.~Curless and M.~Levoy, ``A volumetric method for building complex models from
  range images,'' in \emph{Conference on Computer Graphics and Interactive
  Techniques}, 1996, pp. 303--312.

\bibitem{kazhdan2006poisson}
M.~Kazhdan, M.~Bolitho, and H.~Hoppe, ``Poisson surface reconstruction,'' in
  \emph{Eurographics Symposium on Geometry Processing}, 2006.

\bibitem{hornung2013octomap}
A.~Hornung, K.~M. Wurm, M.~Bennewitz, C.~Stachniss, and W.~Burgard, ``{OctoMap:
  An efficient probabilistic 3D mapping framework based on octrees},''
  \emph{Autonomous Robots}, vol.~34, no.~3, 2013.

\bibitem{niessner2013real}
M.~Nie{\ss}ner, M.~Zollh{\"o}fer, S.~Izadi, and M.~Stamminger, ``Real-time 3d
  reconstruction at scale using voxel hashing,'' \emph{ACM Transactions on
  Graphics (ToG)}, vol.~32, no.~6, pp. 1--11, 2013.

\bibitem{kimera}
A.~Rosinol, M.~Abate, Y.~Chang, and L.~Carlone, ``{Kimera: an Open-Source
  Library for Real-Time Metric-Semantic Localization and Mapping},'' in
  \emph{IEEE Intl. Conf. on Robotics and Automation (ICRA)}, 2020.

\bibitem{mccormac2017semanticfusion}
J.~McCormac, A.~Handa, A.~Davison, and S.~Leutenegger, ``Semanticfusion: Dense
  3d semantic mapping with convolutional neural networks,'' in \emph{2017 IEEE
  International Conference on Robotics and automation (ICRA)}.\hskip 1em plus
  0.5em minus 0.4em\relax IEEE, 2017, pp. 4628--4635.

\bibitem{hermans2014dense}
A.~Hermans, G.~Floros, and B.~Leibe, ``{Dense 3D Semantic Mapping of Indoor
  Scenes from RGB-D Images},'' in \emph{IEEE International Conference on
  Robotics and Automation (ICRA)}, 2014, pp. 2631--2638.

\bibitem{kundu2014joint}
A.~Kundu, Y.~Li, F.~Dellaert, F.~Li, and J.~M. Rehg, ``Joint semantic
  segmentation and 3d reconstruction from monocular video,'' in \emph{European
  Conference on Computer Vision}.\hskip 1em plus 0.5em minus 0.4em\relax
  Springer, 2014, pp. 703--718.

\bibitem{GP-classification-1}
J.~Hensman, A.~Matthews, and Z.~Ghahramani, ``{Scalable Variational Gaussian
  Process Classification},'' in \emph{International Conference on Artificial
  Intelligence and Statistics}, 2015, pp. 351--360.

\bibitem{GP-classification-2}
T.~Galy-Fajou, F.~Wenzel, C.~Donner, and M.~Opper, ``Multi-class gaussian
  process classification made conjugate: Efficient inference via data
  augmentation,'' in \emph{Uncertainty in Artificial Intelligence Conference},
  2020, pp. 755--765.

\bibitem{jadidi2017gaussian}
M.~{Ghaffari Jadidi}, L.~Gan, S.~Parkison, J.~Li, and R.~Eustice, ``{Gaussian
  Processes Semantic Map Representation},'' \emph{arXiv:1707.01532}, 2017.

\bibitem{rasmussen2003gaussian}
C.~E. Rasmussen, ``Gaussian processes in machine learning,'' in \emph{Summer
  School on Machine Learning}.\hskip 1em plus 0.5em minus 0.4em\relax Springer,
  2003, pp. 63--71.

\bibitem{o2012gaussian}
S.~O'Callaghan and F.~Ramos, ``Gaussian process occupancy maps,'' \emph{The
  International Journal of Robotics Research (IJRR)}, vol.~31, no.~1, pp.
  42--62, 2012.

\bibitem{kim2013occupancy}
S.~Kim and J.~Kim, ``{Occupancy Mapping and Surface Reconstruction Using Local
  Gaussian Processes With Kinect Sensors},'' \emph{IEEE Trans. on Cybernetics},
  vol.~43, no.~5, pp. 1335--1346, 2013.

\bibitem{jadidi2014exploration}
M.~G. Jadidi, J.~V. Mir{\'o}, R.~Valencia, and J.~Andrade-Cetto, ``{Exploration
  on Continuous Gaussian Process Frontier Maps},'' in \emph{IEEE Int. Conf. on
  Robotics and Automation (ICRA)}, 2014, pp. 6077--6082.

\bibitem{snelson2006sparse}
E.~Snelson and Z.~Ghahramani, ``Sparse gaussian processes using
  pseudo-inputs,'' in \emph{Advances in neural information processing systems},
  2006, pp. 1257--1264.

\bibitem{vff}
J.~Hensman, N.~Durrande, and A.~Solin, ``{Variational Fourier features for
  Gaussian processes},'' \emph{The Journal of Machine Learning Research},
  vol.~18, no.~1, pp. 5537--5588, 2017.

\bibitem{koppel2019consistent}
A.~Koppel, ``{Consistent online Gaussian Process regression without the sample
  complexity bottleneck},'' in \emph{American Control Conference (ACC)}, 2019,
  pp. 3512--3518.

\bibitem{koppel2019optimally}
A.~Koppel, A.~S. Bedi, K.~Rajawat, and B.~M. Sadler, ``Optimally compressed
  nonparametric online learning,'' \emph{IEEE Signal Processing Magazine},
  2020.

\bibitem{tresp2000bayesian}
V.~Tresp, ``A bayesian committee machine,'' \emph{Neural computation}, vol.~12,
  no.~11, pp. 2719--2741, 2000.

\bibitem{kim2014recursive}
S.~Kim and J.~Kim, ``{Recursive Bayesian Updates for Occupancy Mapping and
  Surface Reconstruction},'' in \emph{Australasian Conference on Robotics and
  Automation (ACRA)}, 2014.

\bibitem{bauer2016understanding}
M.~Bauer, M.~van~der Wilk, and C.~E. Rasmussen, ``Understanding probabilistic
  sparse gaussian process approximations,'' in \emph{Advances in neural
  information processing systems}, 2016, pp. 1533--1541.

\bibitem{rasmussen2002infinite}
C.~E. Rasmussen and Z.~Ghahramani, ``Infinite mixtures of gaussian process
  experts,'' in \emph{Advances in neural information processing systems}, 2002,
  pp. 881--888.

\bibitem{nedic2017distributed}
A.~Nedi{\'c}, A.~Olshevsky, and C.~A. Uribe, ``Distributed learning for
  cooperative inference,'' \emph{arXiv preprint:1704.02718}, 2017.

\bibitem{jadbabaie2012non}
A.~Jadbabaie, P.~Molavi, A.~Sandroni, and A.~Tahbaz-Salehi, ``Non-bayesian
  social learning,'' \emph{Games and Economic Behavior}, vol.~76, no.~1, pp.
  210--225, 2012.

\bibitem{Zobeidi_GPMapping_IROS20}
E.~Zobeidi, A.~Koppel, and N.~Atanasov, ``Dense incremental metric-semantic
  mapping via sparse gaussian process regression,'' in \emph{IEEE/RSJ
  International Conference on Intelligent Robots and Systems (IROS)}, 2020.

\bibitem{elfes1989using}
A.~Elfes, ``Using occupancy grids for mobile robot perception and navigation,''
  \emph{Computer}, vol.~22, no.~6, pp. 46--57, 1989.

\bibitem{supereight}
E.~Vespa, N.~Nikolov, M.~Grimm, L.~Nardi, P.~H.~J. Kelly, and S.~Leutenegger,
  ``{Efficient Octree-Based Volumetric SLAM Supporting Signed-Distance and
  Occupancy Mapping},'' \emph{IEEE Robotics and Automation Letters}, vol.~3,
  no.~2, pp. 1144--1151, 2018.

\bibitem{engel2014lsd}
J.~Engel, T.~Sch{\"o}ps, and D.~Cremers, ``{LSD-SLAM: Large-scale direct
  monocular SLAM},'' in \emph{European Conf. on Computer Vision}, 2014.

\bibitem{segmap}
R.~Dub{\'e}, A.~Cramariuc, D.~Dugas, H.~Sommer, M.~Dymczyk, J.~Nieto,
  R.~Siegwart, and C.~Cadena, ``{SegMap: Segment-based mapping and localization
  using data-driven descriptors},'' \emph{The International Journal of Robotics
  Research}, vol.~39, no. 2-3, pp. 339--355, 2020.

\bibitem{henry2012rgb}
P.~Henry, M.~Krainin, E.~Herbst, X.~Ren, and D.~Fox, ``Rgb-d mapping: Using
  kinect-style depth cameras for dense 3d modeling of indoor environments,''
  \emph{The International Journal of Robotics Research}, vol.~31, no.~5, pp.
  647--663, 2012.

\bibitem{surfel-mapping}
J.~Behley and C.~Stachniss, ``{Efficient Surfel-Based SLAM using 3D Laser Range
  Data in Urban Environments},'' in \emph{Robotics: Science and Systems}, 2018.

\bibitem{teixeira2016real}
L.~Teixeira and M.~Chli, ``Real-time mesh-based scene estimation for aerial
  inspection,'' in \emph{IEEE/RSJ International Conference on Intelligent
  Robots and Systems (IROS)}, 2016, pp. 4863--4869.

\bibitem{piazza2018real}
E.~Piazza, A.~Romanoni, and M.~Matteucci, ``Real-time cpu-based large-scale
  three-dimensional mesh reconstruction,'' \emph{IEEE Robotics and Automation
  Letters}, vol.~3, no.~3, pp. 1584--1591, 2018.

\bibitem{newcombe2011kinectfusion}
R.~{Newcombe}, S.~{Izadi}, O.~{Hilliges}, D.~{Molyneaux}, D.~{Kim},
  A.~{Davison}, P.~{Kohi}, J.~{Shotton}, S.~{Hodges}, and A.~{Fitzgibbon},
  ``{KinectFusion: Real-Time Dense Surface Mapping and Tracking},'' in
  \emph{IEEE Int. Symposium on Mixed and Augmented Reality}, 2011, pp.
  127--136.

\bibitem{whelan2016elasticfusion}
T.~Whelan, R.~F. Salas-Moreno, B.~Glocker, A.~J. Davison, and S.~Leutenegger,
  ``Elasticfusion: Real-time dense slam and light source estimation,''
  \emph{The International Journal of Robotics Research}, vol.~35, no.~14, pp.
  1697--1716, 2016.

\bibitem{oleynikova2017voxblox}
H.~Oleynikova, Z.~Taylor, M.~Fehr, R.~Siegwart, and J.~Nieto, ``Voxblox:
  Incremental 3d euclidean signed distance fields for on-board mav planning,''
  in \emph{IEEE/RSJ Int. Conf. on Intelligent Robots and Systems}, 2017.

\bibitem{han2019fiesta}
L.~Han, F.~Gao, B.~Zhou, and S.~Shen, ``Fiesta: Fast incremental euclidean
  distance fields for online motion planning of aerial robots,'' in
  \emph{IEEE/RSJ Int. Conf. on Intelligent Robots and Systems}, 2019.

\bibitem{kintinuous}
T.~Whelan, M.~Kaess, H.~Johannsson, M.~Fallon, J.~J. Leonard, and J.~McDonald,
  ``{Real-time large-scale dense RGB-D SLAM with volumetric fusion},''
  \emph{The International Journal of Robotics Research (IJRR)}, vol.~34, no.
  4-5, pp. 598--626, 2015.

\bibitem{klingensmith2015chisel}
M.~Klingensmith, I.~Dryanovski, S.~S. Srinivasa, and J.~Xiao, ``{Chisel: Real
  Time Large Scale 3D Reconstruction Onboard a Mobile Device using Spatially
  Hashed Signed Distance Fields},'' in \emph{Robotics: science and systems},
  vol.~4.\hskip 1em plus 0.5em minus 0.4em\relax Citeseer, 2015, p.~1.

\bibitem{flashfusion}
L.~Han and L.~Fang, ``{FlashFusion: Real-time Globally Consistent Dense 3D
  Reconstruction using CPU Computing},'' in \emph{Robotics: Science and Systems
  (RSS)}, 2018.

\bibitem{InfiniTAM}
O.~K{\"{a}}hler, V.~A. Prisacariu, and D.~W. Murray, ``Real-time large-scale
  dense 3d reconstruction with loop closure,'' in \emph{European Conference on
  Computer Vision (ECCV)}, 2016, pp. 500--516.

\bibitem{voxgraph}
V.~{Reijgwart}, A.~{Millane}, H.~{Oleynikova}, R.~{Siegwart}, C.~{Cadena}, and
  J.~{Nieto}, ``{Voxgraph: Globally Consistent, Volumetric Mapping Using Signed
  Distance Function Submaps},'' \emph{IEEE Robotics and Automation Letters},
  2020.

\bibitem{grinvald2019volumetric}
M.~Grinvald, F.~Furrer, T.~Novkovic, J.~J. Chung, C.~Cadena, R.~Siegwart, and
  J.~Nieto, ``{Volumetric Instance-Aware Semantic Mapping and 3D Object
  Discovery},'' \emph{IEEE Robotics and Automation Letters}, vol.~4, no.~3, pp.
  3037--3044, 2019.

\bibitem{bylow2013real}
E.~Bylow, J.~Sturm, C.~Kerl, F.~Kahl, and D.~Cremers, ``Real-time camera
  tracking and 3d reconstruction using signed distance functions.'' in
  \emph{Robotics: Science and Systems}, 2013.

\bibitem{Oleynikova2016ContinuoustimeTO}
H.~Oleynikova, M.~Burri, Z.~Taylor, J.~I. Nieto, R.~Siegwart, and E.~Galceran,
  ``Continuous-time trajectory optimization for online uav replanning,'' in
  \emph{IEEE/RSJ Int. Conf. on Intelligent Robots and Systems (IROS)}, 2016.

\bibitem{Saulnier_ActiveMapping_ICRA20}
K.~Saulnier, N.~Atanasov, G.~Pappas, and V.~Kumar, ``Information theoretic
  active exploration in signed distance fields,'' in \emph{IEEE Int. Conf. on
  Robotics and Automation (ICRA)}, 2020.

\bibitem{Kim2015}
S.~Kim and J.~Kim, ``{GPmap: A Unified Framework for Robotic Mapping Based on
  Sparse Gaussian Processes},'' in \emph{International Conference on Field and
  Service Robotics}, 2015.

\bibitem{Wang_GPRegressionMapping}
J.~Wang and B.~Englot, ``Fast, accurate gaussian process occupancy maps via
  test-data octrees and nested bayesian fusion,'' in \emph{IEEE Int. Conf. on
  Robotics and Automation (ICRA)}, 2016, pp. 1003--1010.

\bibitem{ramos2016hilbert}
F.~Ramos and L.~Ott, ``{Hilbert maps: Scalable continuous occupancy mapping
  with stochastic gradient descent},'' \emph{The International Journal of
  Robotics Research}, vol.~35, no.~14, pp. 1717--1730, 2016.

\bibitem{bayesian_hilbert}
R.~Senanayake and F.~Ramos, ``{Bayesian Hilbert Maps for Continuous Occupancy
  Mapping in Dynamic Environments},'' in \emph{Conference on Robot Learning
  (CoRL)}, ser. Proceedings of Machine Learning Research, vol.~78, 2017, pp.
  458--471.

\bibitem{Guizilini-RSS-17}
V.~Guizilini and F.~Ramos, ``{Learning to Reconstruct 3D Structures for
  Occupancy Mapping},'' in \emph{Robotics: Science and Systems}, 2017.

\bibitem{guo2019information}
S.~Guo and N.~A. Atanasov, ``Information filter occupancy mapping using
  decomposable radial kernels,'' in \emph{IEEE/RSJ International Conference on
  Intelligent Robots and Systems (IROS)}, 2019, pp. 7887--7894.

\bibitem{vineet2015incremental}
V.~{Vineet}, O.~{Miksik}, M.~{Lidegaard}, M.~{Nie{\ss}ner}, S.~{Golodetz},
  V.~A. {Prisacariu}, O.~{K{\"a}hler}, D.~W. {Murray}, S.~{Izadi},
  P.~{P{\'e}rez}, and P.~H.~S. {Torr}, ``Incremental dense semantic stereo
  fusion for large-scale semantic scene reconstruction,'' in \emph{IEEE
  International Conference on Robotics and Automation (ICRA)}, 2015, pp.
  75--82.

\bibitem{sengupta2015semantic}
S.~Sengupta and P.~Sturgess, ``Semantic octree: Unifying recognition,
  reconstruction and representation via an octree constrained higher order
  mrf,'' in \emph{2015 IEEE International Conference on Robotics and Automation
  (ICRA)}.\hskip 1em plus 0.5em minus 0.4em\relax IEEE, 2015, pp. 1874--1879.

\bibitem{yang2017semantic}
S.~Yang, Y.~Huang, and S.~Scherer, ``{Semantic 3D occupancy mapping through
  efficient high-order CRFs},'' in \emph{IEEE/RSJ International Conference on
  Intelligent Robots and Systems (IROS)}, 2017, pp. 590--597.

\bibitem{zhao2016building}
Z.~Zhao and X.~Chen, ``{Building 3D semantic maps for mobile robots using RGB-D
  camera},'' \emph{Intelligent Service Robotics}, vol.~9, no.~4, pp. 297--309,
  2016.

\bibitem{zheng2018pixels}
K.~Zheng and A.~Pronobis, ``From pixels to buildings: End-to-end probabilistic
  deep networks for large-scale semantic mapping,'' \emph{arXiv preprint
  arXiv:1812.11866}, 2018.

\bibitem{gan2020bayesian}
L.~Gan, R.~Zhang, J.~W. Grizzle, R.~M. Eustice, and M.~{Ghaffari Jadidi},
  ``Bayesian spatial kernel smoothing for scalable dense semantic mapping,''
  \emph{IEEE Robotics and Automation Letters}, vol.~5, no.~2, pp. 790--797,
  2020.

\bibitem{consensus}
R.~{Olfati-Saber}, J.~A. {Fax}, and R.~M. {Murray}, ``Consensus and cooperation
  in networked multi-agent systems,'' \emph{Proceedings of the IEEE}, vol.~95,
  no.~1, pp. 215--233, 2007.

\bibitem{rad2010distributed}
K.~{Rahnama Rad} and A.~Tahbaz-Salehi, ``Distributed parameter estimation in
  networks,'' in \emph{IEEE Conference on Decision and Control (CDC)}, 2010,
  pp. 5050--5055.

\bibitem{atanasov2014joint}
N.~Atanasov, R.~Tron, V.~M. Preciado, and G.~J. Pappas, ``Joint estimation and
  localization in sensor networks,'' in \emph{IEEE Conference on Decision and
  Control (CDC)}, 2014, pp. 6875--6882.

\bibitem{nedic2016distributed}
A.~Nedi{\'c}, A.~Olshevsky, and C.~A. Uribe, ``{Distributed Gaussian learning
  over time-varying directed graphs},'' in \emph{Asilomar Conference on
  Signals, Systems and Computers}, 2016, pp. 1710--1714.

\bibitem{choudhary2017distributed}
S.~Choudhary, L.~Carlone, C.~Nieto, J.~Rogers, H.~I. Christensen, and
  F.~Dellaert, ``Distributed mapping with privacy and communication
  constraints: Lightweight algorithms and object-based models,'' \emph{The
  International Journal of Robotics Research}, vol.~36, no.~12, pp. 1286--1311,
  2017.

\bibitem{koch2016multi}
P.~Koch, S.~May, M.~Schmidpeter, M.~K{\"u}hn, C.~Pfitzner, C.~Merkl, R.~Koch,
  M.~Fees, J.~Martin, D.~Ammon, and A.~N{\"u}chter, ``Multi-robot localization
  and mapping based on signed distance functions,'' \emph{Journal of
  Intelligent \& Robotic Systems}, vol.~83, no. 3-4, pp. 409--428, 2016.

\bibitem{door-slam}
P.~{Lajoie}, B.~{Ramtoula}, Y.~{Chang}, L.~{Carlone}, and G.~{Beltrame},
  ``{DOOR-SLAM: Distributed, Online, and Outlier Resilient SLAM for Robotic
  Teams},'' \emph{IEEE Robotics and Automation Letters}, vol.~5, no.~2, pp.
  1656--1663, 2020.

\bibitem{milioto2019icra}
A.~Milioto and C.~Stachniss, ``{Bonnet: An Open-Source Training and Deployment
  Framework for Semantic Segmentation in Robotics using CNNs},'' in \emph{IEEE
  Intl. Conf. on Robotics \& Automation (ICRA)}, 2019.

\bibitem{random-network-consensus}
A.~{Tahbaz-Salehi} and A.~{Jadbabaie}, ``{A Necessary and Sufficient Condition
  for Consensus Over Random Networks},'' \emph{IEEE Transactions on Automatic
  Control}, vol.~53, no.~3, pp. 791--795, 2008.

\bibitem{distributed-subgradient}
A.~{Nedic} and A.~{Ozdaglar}, ``{Distributed Subgradient Methods for
  Multi-Agent Optimization},'' \emph{IEEE Transactions on Automatic Control},
  vol.~54, no.~1, pp. 48--61, 2009.

\bibitem{moreau-condition}
L.~{Moreau}, ``Stability of multiagent systems with time-dependent
  communication links,'' \emph{IEEE Transactions on Automatic Control},
  vol.~50, no.~2, pp. 169--182, 2005.

\bibitem{time-varying-consensus}
F.~{Saadatniaki}, R.~{Xin}, and U.~A. {Khan}, ``Decentralized optimization over
  time-varying directed graphs with row and column-stochastic matrices,''
  \emph{IEEE Transactions on Automatic Control}, vol.~65, no.~11, pp.
  4769--4780, 2020.

\bibitem{scenenn-3dv16}
B.-S. Hua, Q.-H. Pham, D.~T. Nguyen, M.-K. Tran, L.-F. Yu, and S.-K. Yeung,
  ``Scenenn: A scene meshes dataset with annotations,'' in \emph{International
  Conference on 3D Vision (3DV)}, 2016.

\bibitem{marchingcubes}
W.~E. Lorensen and H.~E. Cline, ``{Marching cubes: A high resolution 3D surface
  construction algorithm},'' \emph{Computer Graphics}, vol.~21, no.~4, pp.
  163--169, 1987.

\end{thebibliography}

\end{document}